\documentclass[accepted]{uai2026} 
                        
\usepackage[american]{babel}

\usepackage{natbib} 
\usepackage{mathtools} 
\usepackage{booktabs} 
\usepackage{tikz} 

\usepackage{algorithm}
\usepackage{algorithmic}

\usepackage{braket}
\usepackage{comment}
\usepackage[utf8]{inputenc} 
\usepackage{url}            
\usepackage{booktabs}       
\usepackage{amsfonts}       
\usepackage{nicefrac}       
\usepackage{microtype}      
\usepackage{xcolor}         
\usepackage{amsmath, amssymb}
\usepackage{amsthm, enumitem}
\usepackage{graphicx}
\usepackage{subcaption}

\def \xbf {\mathbf{x}}
\def \Xcal {\mathcal{X}}
\def \ybf {\mathbf{y}}
\newcommand{\vbf}{\mathbf{v}}
\newcommand{\phibf}{\boldsymbol{\phi}}
\newcommand{\xibf}{\boldsymbol{\xi}}
\newcommand{\rhobf}{\boldsymbol{\rho}}
\newcommand{\omegabf}{\boldsymbol{\omega}}
\newcommand{\thetabf}{\boldsymbol{\theta}}
\newcommand{\Lambdabf}{\boldsymbol{\Lambda}}
\newcommand{\Sigmabf}{\boldsymbol{\Sigma}}
\newcommand{\kbf}{\mathbf{k}}
\newcommand{\ebf}{\mathbf{e}}
\newcommand{\abf}{\mathbf{a}}
\newcommand{\bbf}{\mathbf{b}}
\newcommand{\wbf}{\mathbf{w}}

\newcommand{\Abf}{\mathbf{A}}
\newcommand{\Kbf}{\mathbf{K}}
\newcommand{\Ibf}{\mathbf{I}}
\newcommand{\Pbf}{\mathbf{P}}

\newcommand{\Obf}{\mathbf{O}}
\newcommand{\Hbf}{\mathbf{H}}
\newcommand{\Wbf}{\mathbf{W}}
\newcommand{\Gbf}{\mathbf{G}}
\newcommand{\Xbf}{\mathbf{X}}
\newcommand{\Ybf}{\mathbf{Y}}
\newcommand{\Zbf}{\mathbf{Z}}
\newcommand{\Ubf}{\mathbf{U}}
\newcommand{\Vbf}{\mathbf{V}}
\newcommand{\Bbf}{\mathbf{B}}
\newcommand{\Sbf}{\mathbf{S}}

\newcommand{\Hcal}{\mathcal{H}}

\newcommand{\tr}{\mathrm{Tr}}
\newcommand{\C}{\mathbb{C}}

\newcommand{\E}{\mathbb{E}}
\newcommand{\Prob}{\mathbb{P}}  \newcommand{\hsnorm}[1]{\|#1\|_{2}} 
\newtheorem{theorem}{Theorem}
\newtheorem{lemma}{Lemma}
\newtheorem{proposition}{Proposition}
\newtheorem{corollary}{Corollary} \theoremstyle{definition}  

\DeclareMathOperator*{\argmin}{arg\,min}
\DeclareMathOperator*{\argmax}{arg\,max}

\title{Balancing Expressivity and Learnability in \\ Quantum Kernel Bandit Optimization}

\author[1]{Yuqi~Huang}
\author[1,2]{Vincent~Y.~F.~Tan}
\author[3]{Sharu~Theresa~Jose}
\affil[1]{%
    Department of Mathematics\\
    National University of Singapore\\
    Singapore
}
\affil[2]{%
    Department of Electrical and Computer Engineering\\
    National University of Singapore\\
    Singapore
}
\affil[3]{%
    School of Computer Science\\
    University of Birmingham\\
    Birmingham, United Kingdom
  }
  
\begin{document}
\maketitle

\begin{abstract}
    We investigate Gaussian process (GP) bandit optimization with quantum kernels, assuming the mean reward function lies in the reproducing kernel Hilbert space (RKHS) induced by the quantum kernel. This setting is motivated by NISQ-era tasks such as quantum control, state preparation and variational quantum algorithms. While quantum kernels can  offer a  `quantum advantage' via domain-specific inductive biases, na\"{i}vely using full, high-dimensional kernels increases model complexity and information gain, leading to higher cumulative regret and poor learnability. To address this, we propose projected quantum kernels and classical kernel approximation techniques that reduce feature dimensionality while preserving key quantum properties. Using these approximate kernels, we develop misspecified GP bandit algorithms and derive regret bounds that characterize the trade-off between approximation error and information gain. The regret bounds provide principled guidance for selecting the optimal model complexity.  Empirically, our methods  outperform  full quantum kernels in  sample efficiency, while  substantially reducing computational overhead,
    enabling scalable GP optimization for quantum-native applications.
\end{abstract}

\section{Introduction}
Recent advances in quantum hardware have ushered in the era of \textit{noisy intermediate-scale quantum} (NISQ) computing, characterized by quantum devices with limited qubits and high noise levels. Within this regime, quantum machine learning has emerged as a promising application, aiming to leverage these NISQ-era devices to accelerate learning from classical or quantum data. This includes discriminative and generative learning approaches based on quantum neural networks \cite{schuld2021quantumkernel} and quantum kernel methods \cite{blank2020quantum,rebentrost2014quantum}, as well as sequential decision making problems.


Recently, there has been growing interest in developing quantum algorithms for classical bandit learning problems, aiming to achieve speedups in sequential decision-making. These include quantum algorithms for best arm identification \citep{casale2020quantum,Wang2025baiqo,Buchholz2025mab,Wang2021qeamab}, and for addressing exploration-exploitation tradeoff under both linear \citep{lumbreras2022multi,wan2023quantum,wu2023quantum} and non-linear reward models \citep{hikima2024quantum,dai2023quantum}. However, these approaches assume access to a quantum reward oracle, which simplifies quantum algorithm design, but limits applicability in realistic settings where the classical reward distribution must be encoded into a quantum state.

In this work, we depart from this oracle-based paradigm, and study the \textit{quantum kernel bandit} problem. Here, the learner interacts sequentially with an unknown reward function through noisy evaluations. We assume that the unknown mean reward function $f^*: \mathcal{X} \rightarrow \mathbb{R}$ lies in the RKHS $ \mathcal{H}_{\kappa_Q}$ induced by a  quantum kernel $\kappa_{Q}(\cdot,\cdot)$. At each round $t$, the learner selects an action  $\xbf_t \in \mathcal{X}$ and observes a noisy reward $y_t=f^*(\xbf_t)+\eta_t$, where $\eta_t$ is zero mean noise. The objective of the learner is to minimize the \textit{cumulative regret} over $T$ rounds, $R_T=\sum_{t=1}^T (f^*(\xbf^*) - f^*(\xbf_t))$, where $\xbf^* =\arg \max_{\xbf \in \mathcal{X}} f^*(\xbf)$ denotes the optimal action in hindsight. 

 This formulation naturally captures a range of NISQ-era optimization tasks, including state preparation in quantum sensing \citep{schuff2020improving}, experimental quantum control \citep{bukov2018reinforcement}, and the optimization of variational quantum algorithms \citep{wanner2025variational}. Here, actions $\mathbf{x}$ correspond to continuous parameters (e.g., gate angles or pulse amplitudes), and a fixed quantum embedding map $\rhobf(\xbf)$ induces a corresponding quantum kernel. 
The resulting quantum RKHS provides a natural hypothesis space for modelling the mean reward function $f^*(\mathbf{x})$, while reward feedback is obtained from  noisy finite-shot measurements. 
\begin{figure}[t!] 
  \centering 
    \includegraphics[width=.75\linewidth]{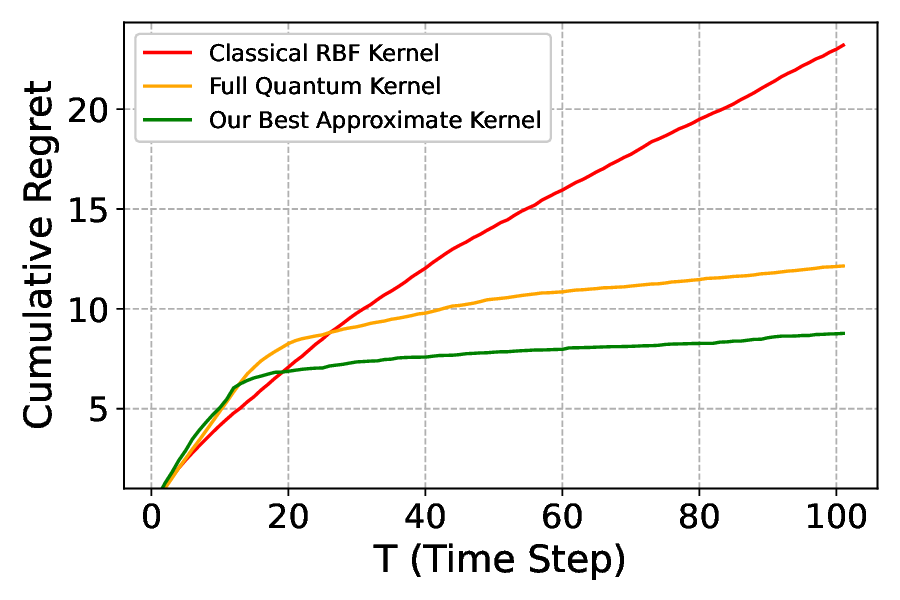}
    \caption{Cumulative regrets over $T=100$ rounds (using the EC-GP-UCB algorithm) with three different modelling kernels on a synthetic task where the unknown reward function lies in a quantum RKHS. While the full quantum kernel beats the classical RBF kernel, our proposed {\em projected quantum kernel-based approximation} achieves an even smaller regret. 
    See subsection~\ref{exp1} for experimental details.}
  \label{fig:quantum_advantage}
  
\end{figure}

We approach this problem using quantum Gaussian process (GP) models. Leveraging the duality  between RKHS and GPs \citep{kanagawa2018gaussian}, we encode the inductive bias (i.e., $f^* \in \mathcal{H}_{\kappa_Q}$) by placing a GP prior with the kernel $\kappa_Q$. While such models can theoretically outperform classical GPs with standard kernels such as RBF or Mat\'ern (see Figure~\ref{fig:quantum_advantage}) 
\citep{smith2023faster}, they face two unique challenges:
\begin{itemize}[leftmargin=*]
    \item The expressivity of quantum kernels on $n$-qubit systems scales exponentially with $n$, leading to high \textit{information gain} and, consequently, excessive cumulative regret. This hinders the learnability of the quantum GP models.
    \item As the quantum system size $n$ grows, kernel values concentrate exponentially, requiring exponentially many quantum measurements for accurate estimation \citep{thanasilp2024exponential}.
\end{itemize}  


To address these challenges, we propose new GP algorithms that \textit{approximate the full quantum RKHS} using either lower-dimensional quantum subspaces or quantum-inspired classical approximations. These methods reduce the model complexity, and thus the information gain, at the cost of \textit{kernel misspecification}, which can impact the regret \citep{bogunovic2021misspecified}. 
Our main contributions are:
\begin{itemize}[leftmargin=*]
    \item {\bf Quantum-Specific Learnability Barrier}: 
    We identify the exponential scaling of the maximum information gain with the number of qubits $n$ as the fundamental learnability barrier for GP bandits with fidelity quantum kernels. This statistical barrier, together with kernel concentration as $n$ grows, motivates the use of approximate quantum-kernel bandit algorithms.

    \item {\bf Algorithmic Framework}: We develop a new framework for approximate GP optimization in quantum kernel bandits by combining quantum kernel approximation techniques with GP (and linear) bandit algorithms. Our approximation approaches include: linear projected quantum kernels (LPQKs), operating on reduced quantum subspaces, and quantum-inspired classical approximations -- Random Fourier Features (RFF) and Newton basis expansions.
    \item  {\bf Quantum-Structured Trade-off Analysis with Regret Bounds}: 
    We then derive regret bounds that quantify the trade-off between information gain and kernel misspecification. 
    In regimes where the full quantum kernel is excessively high-dimensional, we show that a suitably chosen approximate kernel achieves \emph{strictly lower regret} than the full model (Figure~\ref{fig:quantum_advantage}). The bounds also guide the choice of approximation parameters, such as the number of qubits traced out in projected quantum kernels, or the number of random features used in RFF.
    Although some approximation and misspecified-bandit tools used in this work are kernel-general, our analysis specializes them to quantum kernels by exploiting the structure of quantum feature maps. 
    In particular, LPQKs correspond to Pauli-weight projections of the $n$-qubit density-matrix feature space. For quantum encodings that admit a Fourier form, the frequencies are determined by eigenvalue gaps of the encoding generators. Thus, our results connect circuit and observable structure directly to the information-gain/misspecification trade-off in regret.

    \item  {\bf Computational Efficiency}:
While naive GP optimization with a full quantum kernel incurs $\mathcal{O}(T^3)$ time complexity and costly circuit evaluations, our approximation-based methods leverage a $D$-dimensional feature map ($D \ll T$) to reduce the time complexity to $\mathcal{O}(TD^2+D^3)$; or rely on smaller qubit subsystems, thereby enhancing the scalability. 
    
    \item  {\bf Experimental Validation}: We evaluate our methods on synthetic and real quantum tasks (e.g., phase classification, variational quantum eigensolver optimization), showing that our approximations often \emph{outperform the full quantum kernel} both in sample efficiency and computational cost, and that theoretical insights into approximation error can guide dimension selection, yielding near-optimal performance in practice.
\end{itemize}
\textbf{Related Works:} 
A common approach to optimizing the NISQ-era objectives is to use variational quantum algorithms (VQAs), which directly update the parameters of a parameterized quantum circuit \citep{cerezo2021variational}.
These methods often rely on gradients estimated by finite differences or the parameter-shift rule, such gradient estimates can be measurement-intensive, and may suffer from barren plateaus \citep{mcclean2018barren,wang2021noise}. We take a different, gradient-free surrogate-optimization viewpoint: the energy or reward is treated as a noisy black-box function, and a quantum-kernel GP surrogate is used to select queries to balance exploration and exploitation.

Kernelized bandits (including GP bandits) traditionally assume \textit{realizability}: the unknown reward function $f^*$ either lies in the RKHS associated with the kernel $\kappa(\xbf,\xbf')$, or is drawn from a mean-zero GP prior ${\rm GP}(0,\kappa(\xbf,\xbf'))$ \citep{valko2013finite, srinivas2009gaussian}. However, in practice, the true RKHS or the exact kernel function is often unknown, making the realizability assumption too restrictive.
Our work belongs to the  \textit{misspecified kernelized bandit} setting, where the learner optimizes using an approximate kernel rather than the ground truth. This paradigm was first formalized for linear bandits by \citet{ghosh2017mislin}, who established a regret lower bound of $\Omega(\varepsilon\sqrt{d}T)$ for $\varepsilon$-perturbed additive reward models  with $d$-dimensional vectors $\xbf$. Subsequent research has deepened the understanding of misspecification in  linear bandits \cite{zanette2020learning,neu2020efficient}.  More recently, \citet{bogunovic2021misspecified} and \citet{camilleri2021high} extended these insights to kernelized settings, demonstrating that misspecification introduces an additive regret term of   $\mathcal{O}(\varepsilon \sqrt{\gamma_T} T)$, where $\varepsilon$ is the misspecification error and $\gamma_T$ is the  maximum information gain of the kernel. 

While quantum kernels have been widely used in supervised and generative learning settings,  their integration into probabilistic models like GPs remains relatively underexplored. \cite{smith2023faster} pioneered the use of GP surrogates with quantum kernels for variational quantum eigensolver (VQE) optimization, showing that encoding physical knowledge into the kernel can outperform GPs using classical kernels (e.g., RBF, Mat\'ern).
\citet{nicoli2023physics} proposed a quantum-inspired classical kernel, termed the VQE-kernel, for Bayesian optimization of VQE problems. 
Other recent works, such as \citet{rapp2024quantum} and \citet{guo2024benchmarking}, explore quantum GP regression with NISQ-compatible quantum kernels,  highlighting the promise of quantum probabilistic models in practical settings. 

In contrast to prior work, we study GP bandit optimization with quantum kernels under kernel approximation, bridging the misspecified kernel bandit literature with quantum kernel methods. Our analysis explicitly characterizes how approximating high-dimensional quantum kernels impacts information gain, regret, and computational complexity.

\section{Problem Setting}
\subsection{Quantum Kernel Bandit Learning}
We consider a sequential online learning problem where at each round $t$, the learner chooses an action $\xbf_t \in \mathcal{X}$ from the finite action set $\mathcal{X}$. Corresponding to the chosen action, the learner observes a noisy reward
\begin{align}
    y_t= f^*(\xbf_t) + \eta_t, \label{eq:reward}
\end{align} where $f^*:\mathcal{X} \rightarrow \mathbb{R}$ is the true unknown reward function and $\eta_t$ is the i.i.d. zero-mean noise. The goal of the learner is to minimize the \textit{cumulative regret} over $T$ rounds, \begin{align}
    R_T =\sum_{t=1}^T f^*(\xbf^*) -\sum_{t=1}^T f^*(\xbf_t), \label{eq:cumulative_regret}
\end{align} with respect to the optimal $\xbf^*=\arg \max_{\xbf \in \mathcal{X}} f^*(\xbf)$. We assume that the $f^*$ belongs to the RKHS $ \mathcal{H}_{\kappa_Q}(\mathcal{X})$ defined by a quantum kernel $\kappa_Q(\xbf,\xbf')$, with bounded RKHS norm, i.e., $\Vert f^* \Vert_{\mathcal{H}_{\kappa_Q}} \leq B$ for some $B<\infty$.


Specifically, we consider the \textit{global fidelity quantum kernel}, defined via a \textit{quantum feature map}  $\rhobf : \mathcal{X} \rightarrow \mathcal{H}$. Here,  $\mathcal{H}$ denotes the Hilbert space of $2^n \times 2^n$ Hermitian matrices with the inner product $\langle \Abf,\Bbf \rangle_{\mathcal{H}}={\rm Tr}(\Abf \Bbf^{\dag})$.  The  feature map $\rhobf$  is realized through a quantum circuit $\Sbf(\xbf)$, which maps $\xbf \in \mathcal{X}$ to an $n$-qubit quantum state $\rhobf(\xbf)$ by operating on the initial state $\vert \mathbf{0} \rangle$ as  \begin{align}
\rhobf(\xbf)=\Sbf(\xbf) \vert \mathbf{0}\rangle \langle \mathbf{0} \vert \Sbf(\xbf)^{\dag}. \label{eq:quantumfeaturemap}
\end{align}
The resulting quantum state $\rhobf(\xbf)$ is represented as a $2^n \times 2^n$-dimensional \textit{density matrix} (i.e., a positive semi-definite, Hermitian matrix with unit trace).

For $\xbf=(x_1,\hdots,x_d) \in \mathbb{R}^d$, the quantum circuit $\Sbf(\xbf)$ can be constructed by interleaving fixed unitary evolutions $\Wbf$ with data dependent encoding gates $e^{-ix_j\Gbf_j}, 1\le j \le d$ as \citep{schuld2021quantum}: \begin{align}
   \hspace{-0.2cm} \Sbf(\xbf)= \Wbf^{(d+1)}e^{-i x_d \Gbf_d}\Wbf^{(d)}\cdots e ^{-i x_1\Gbf_1} \Wbf^{(1)}. \label{eq:quantumcircuit}
\end{align} Here, each $\Gbf_j$ is a $d_{\Gbf_j} \leq 2^n$-dimensional Hermitian operator serving as the generator for the encoding gate, while $\Wbf^{(1)}, \hdots, \Wbf^{(d+1)}$ denote arbitrary unitary evolutions before and after the encoding gates. This construction naturally generalizes to multiple repeated layers of $\Sbf(\xbf)$ \citep{schuld2008effect}.

The quantum feature map induces the \textit{global fidelity quantum kernel} \citep{huang2021power},
\begin{align} \label{eqn:defquantumkernel}
    \kappa_Q(\xbf,\xbf')={\rm Tr}(\rhobf(\xbf)\rhobf(\xbf')^{\dag}).
\end{align} The kernel evaluates the similarity between  $\xbf, \xbf'$ using global evaluation on the $n$-qubit quantum state. The associated quantum RKHS $\mathcal{H}_{\kappa_Q}(\mathcal{X})$ is the Hilbert-space completion of
$\operatorname{span}\{\kappa_Q(\cdot,\xbf):\xbf\in\mathcal{X}\}$ with inner product satisfying
$\langle \kappa_Q(\cdot,\xbf),\kappa_Q(\cdot,\xbf')\rangle_{\mathcal{H}_{\kappa_Q}}
=
\kappa_Q(\xbf,\xbf')$.
Consequently, every $f\in\mathcal{H}_{\kappa_Q}(\mathcal{X})$ satisfies the reproducing property
$f(\xbf)=\langle f,\kappa_Q(\cdot,\xbf)\rangle_{\mathcal{H}_{\kappa_Q}}$.


For the global fidelity kernel, this RKHS corresponds to reward models that are linear in the quantum feature space, i.e., $$f^*(\xbf)={\rm Tr}(\Hbf\rhobf(\xbf))$$  for some Hermitian operator $\Hbf \in \mathcal{H}$ \citep{schuld2021quantum}. Such models arise in many quantum machine learning applications. For instance, the   variational quantum eigensolver (VQE) aims to find the parameterized quantum state $\rhobf(\xbf^*)$ that approximates the ground state of a quantum Hamiltonian $\Hbf$, based on noisy observations $y_t=f^*(\xbf_t)+\eta_t$ of the energy function $f^*(\xbf)={\rm Tr}(\Hbf\rhobf(\xbf))$, where the noise stems from measurement uncertainty \citep{nicoli2023physics}.

\subsection{Quantum GP Bandit Optimization}
We approach the quantum kernelized bandit problem using quantum GP models. Leveraging the equivalence between RKHS and GPs, we encode the inductive bias $f^* \in \mathcal{H}_{\kappa_Q}$ by placing a GP prior over $f^*$  with the quantum kernel, i.e., $f^* \sim {\rm GP}(0,\kappa_Q(\cdot,\cdot))$. Assuming Gaussian noise $\eta_t \sim \mathcal{N}(0,\sigma^2)$, the posterior over $f^*(\xbf)$ is also a GP, with mean $\mu_T(\xbf)$ and variance $\sigma_T(\xbf,\xbf)$ given by standard formulas (see Appendix~\ref{app:GPdetails}).



 In this setting, a natural bandit algorithm is GP-UCB \citep{srinivas2010} applied with the quantum GP model. At each round $t$, GP-UCB selects an action $\xbf_t$ that maximizes an upper confidence bound (UCB) that contains $f^*(\xbf)$ with high probability over the time horizon:
\[
\xbf_t \;=\; \argmax_{\xbf \in \mathcal{X}} 
\bigl[ \mu_{t-1}(\xbf) + \sqrt{\beta_t}\,\sigma_{t-1}(\xbf)\bigr],
\]
where $\beta_t$ is an exploration parameter and $\sigma_{t-1}(\xbf)$ is an estimate of the standard deviation of the process at point $\xbf$. 


\begin{theorem}[(Simplified) \citep{srinivas2010}]
\label{thm:GP-UCB}
Under mild regularity conditions on the kernel $\kappa_Q(
\cdot,\cdot)$, there exists a suitable sequence $\{\beta_t\}$ such that quantum GP-UCB satisfies, with high probability,
$$
R_T 
\;=\;
\mathcal{O}\bigl(\sqrt{T\,\beta_T\,\gamma_T}\bigr),
$$ where $$\gamma_T 
\;=\;
\max_{A \subset \Xcal, |A|=T}
I\bigl(\mathbf{f}_A;\, \mathbf{y}_A\bigr)$$ is the \textit{maximum information gain}, with $\mathbf{f}_A = [f^*(\xbf)]_{\xbf \in A}$ and $\mathbf{y}_A$ the corresponding noisy observations.
\end{theorem}

\textbf{Challenges of Quantum GP Bandit Optimization}: Theorem~\ref{thm:GP-UCB} shows that the cumulative regret scales  with the maximum information gain $\gamma_T$, a kernel-dependent quantity that measures how informative the noisy observations are about the unknown $f^*$. 
In particular, GP models with large $\gamma_T$ are harder to learn and typically incur higher regret.  

For GPs equipped with an $n$-qubit fidelity quantum kernel as in \eqref{eqn:defquantumkernel}, the information gain scales as $ \gamma_T  = \mathcal{O}(4^n \log T)$ since the model is equivalent to a linear model with $4^n$ dimensional features (see Appendix~\ref{app:Informationgain}). This scaling is exponential in the number of qubits $n$, and  is {\em  tight} for many practical data-encoding circuits \citep{Schuld2019qmlfeature, Havl_ek_2019}.

As an illustrative example, consider the quantum feature map in \eqref{eq:quantumfeaturemap}-\eqref{eq:quantumcircuit} with $\Wbf=\Ibf$ and generators $\Gbf_i\in \{\Xbf,\Ybf,\Zbf\}$ given by single-qubit Pauli operators. This feature map, called the {\em rotation encoding}, generates a feature space whose dimension grows exponentially in $n$. For example, even without entangling gates, a product of single-qubit fixed-axis rotations, e.g., $R_y(x_i)$, has a three-dimensional local span generated by $\{\Ibf,\Xbf,\Zbf\}$, and therefore gives feature dimension $3^n$ when applied independently to $n$ qubits. Richer encodings, such as two noncommuting Pauli rotations or data re-uploading, can yield a feature dimension of {\em exactly} $4^n$ \citep{kubler2021inductive}. We empirically validate this exponential dependence in Figure~\ref{fig:info_gain}, where we plot the information gain $\gamma_T$ as a function of the number of data samples $T$ for different values of $n$ (left); and as a function of $n$ for fixed $T=1000$ (right). The latter clearly illustrates the exponential dependence of $\gamma_T$ on $n$. We use Pauli-Y  and Pauli-Z encoding, similar to our experiments in Section~\ref{sec:experiments}.

Consequently, applying na\"ive GP-UCB algorithm with high-qubit quantum kernels can lead to prohibitively large regret. Moreover, as $n$ increases, kernel values for different data pairs can concentrate exponentially around a fixed value, requiring an exponential number of quantum measurements to accurately estimate the kernel \citep{thanasilp2024exponential}. 


 To overcome these challenges, we propose surrogate GP models that use a lower-dimensional RKHS $\mathcal{H}_{\kappa_{\rm app}}$, induced by a suitably chosen kernel $\kappa_{\rm app}$. In particular, we assume that there exists an approximation $$\tilde{f} = \argmin_{f \in \mathcal{H}_{\kappa_{\rm app}}(\mathcal{X})} \Vert f -f^* \Vert_{\infty} $$  that has small uniform error, i.e.,   $\Vert \tilde{f} -f^* \Vert_{\infty} \leq \varepsilon$ for small $\varepsilon>0$.  The surrogate GP then models $\tilde{f}  \sim {\rm GP}(0,\kappa_{{\rm app}})$. 

 While such approximations can substantially reduce information gain, they typically introduce \textit{kernel misspecification} relative to the full quantum model. This misspecification can adversely affect cumulative regret \citep{bogunovic2021misspecified}. This motivates designing approximate RKHS for surrogate GP models  that balance information gain with misspecification error to achieve more favorable regret.

\begin{figure}[t!]               
  \centering
  \begin{subfigure}{.24\textwidth}
    \includegraphics[width=\linewidth]{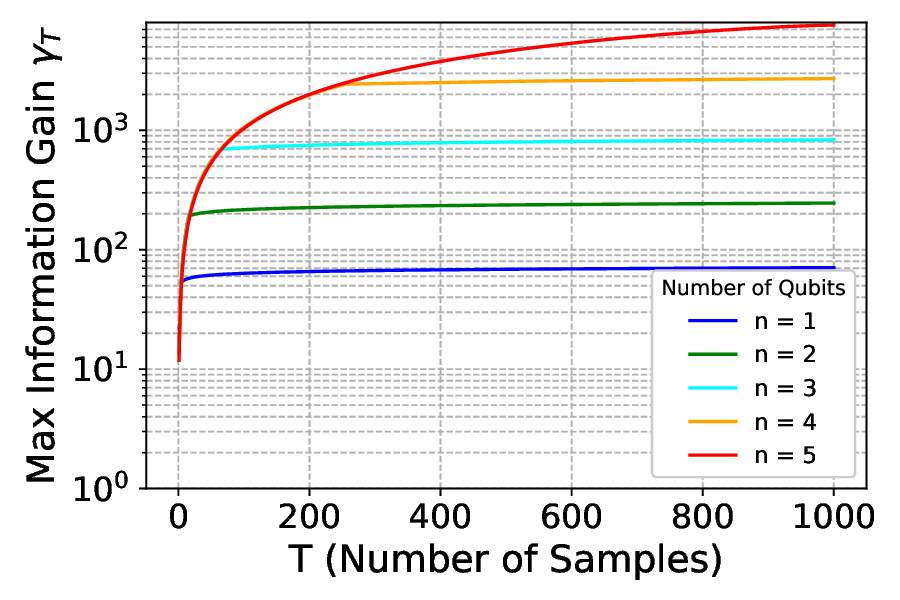}
    \caption{$\gamma_T$ against $T$ for diff.\  $n$}
  \end{subfigure}
  \begin{subfigure}{.24\textwidth}
    \includegraphics[width=\linewidth]{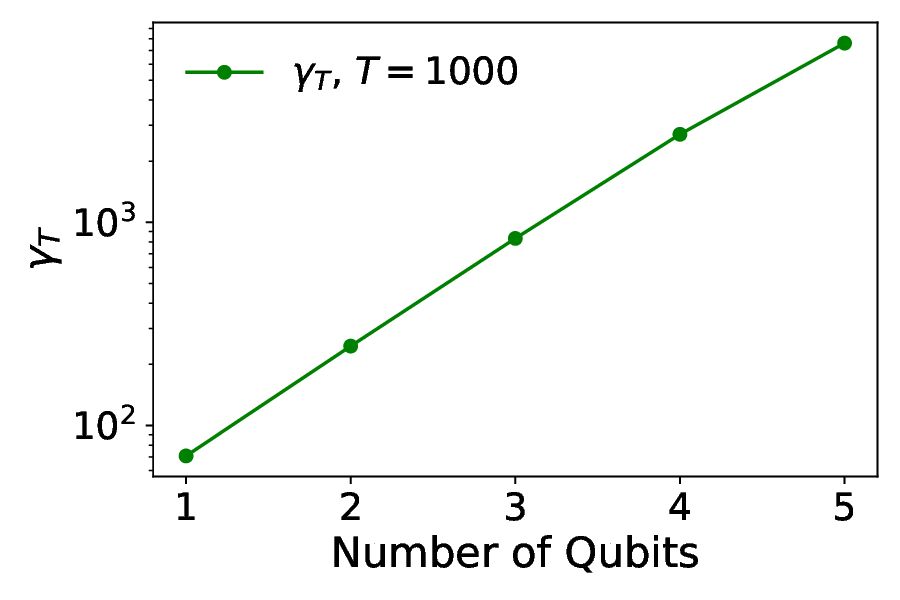}
    \caption{$\gamma_T$ at $T=1000$ against $n$}
  \end{subfigure}

  \caption{ Plot of  $\gamma_T$ as a function of (a) the iterations $T$ for varying system qubit size $n$, with Pauli-Y and Z rotation encoding feature maps, and (b) the number $n$ of qubits with $T=1000$, verifying the exponential increase of $\gamma_T$ with $n$.}
  \label{fig:info_gain}

\end{figure}
\section{Balancing Information Gain and Model Misspecification}

This section proposes quantum GP surrogate models based on approximate kernels $\kappa_{\rm app}$, designed to balance information gain and model misspecification error. 
A key challenge in designing effective approximations is to preserve the quantum properties of the full kernel while avoiding exponential model complexity. We consider three low-rank approximation approaches: $(a)$ projected quantum kernels via reduced quantum circuits; $(b)$ classical Random Fourier Features (RFFs) for efficient kernel approximation; and $(c)$ P-greedy kernel approximations via data-dependent subset selection. While projected quantum kernels still rely on quantum hardware, RFF and P-greedy are fully classical, offering scalability and implementation flexibility. 

\subsection{Linear Projected Quantum Kernels-Based EC-GP-UCB Algorithm} 
Linear projected quantum kernels (LPQK) \citep{huang2021power,kubler2021inductive} compare quantum states at the level of individual sub-systems, by projecting the global quantum state onto a reduced subspace. Specifically, for a sub-system  $s \subseteq\{1,\hdots,n\}$ (indexed by a subset of qubit labels), the LPQK is defined as
\begin{align} \label{eqn:defprojectedkernel}
    \kappa_Q^s(\xbf,\xbf')={\rm Tr}\big(\rhobf^s(\xbf)\rhobf^s(\xbf')\big),
\end{align} where $\rhobf^s(\xbf)={\rm Tr}_{s^c}\big(\rhobf(\xbf)\big)$ is the reduced density matrix obtained by tracing out the complement sub-system $s^c$. The resulting state $\rhobf^{s}(\xbf)$ lies in a Hilbert space of dimension $2^{|s|} \times 2^{|s|}$, which is exponentially smaller than the full state dimension $2^n \times 2^n$. 

While projection onto one sub-system may discard significant information about the global state $\rhobf(\xbf)$, this loss can be mitigated by aggregating information across multiple sub-systems. To this end, we define a summed LPQK over a collection $\mathbb{S}_b=\{s_1,\hdots,s_w\}$ of all sub-systems satisfying $|s_i|\leq b$, for $i=1,\hdots, w$, as 
$
    \kappa_{\mathbb{S}_b}(\xbf,\xbf')= \sum_{s \in \mathbb{S}_b} \frac{1}{|\mathbb{S}_b|}\kappa_Q^s(\xbf,\xbf')
$ \citep{gan2023unified}.
 The resulting kernel has effective dimensionality upper bounded by the sum of the dimensions of the individual projected subspaces. 

 \textbf{LPQK with Misspecified Bandit Algorithm.} We use the summed LPQK $\kappa_{\mathbb{S}_b}$ as the approximate kernel in a surrogate GP model, with $\tilde{f} \sim {\rm GP}(0,\kappa_{\mathbb{S}_b}(\cdot,\cdot))$. The next theorem shows that the RKHS $\mathcal{H}_{\kappa_{\mathbb{S}_b}}(\mathcal{X})$, defined by the summed LPQKs for $1 \leq b<n$, is a natural surrogate for approximating the full quantum RKHS $\mathcal{H}_{\kappa_Q}(\mathcal{X})$, and provides an explicit bound on the misspecification error. 
\begin{theorem} \label{thm:lpqk}
Let $\mathcal{H}_{\kappa_{\mathbb{S}_b}}(\mathcal{X})$ be the RKHS induced by the summed LPQKs $\kappa_{\mathbb{S}_b}$, and let $\mathcal{H}_{\kappa_Q}(\mathcal{X})$ be the RKHS induced by the  full quantum kernel in~\eqref{eqn:defquantumkernel}.  Then,
\begin{enumerate}
\item[(i)] 
$\mathcal{H}_{\kappa_{\mathbb{S}_b}}(\mathcal{X}) \subseteq\mathcal{H}_{\kappa_Q}(\mathcal{X})$,
\item[(ii)] If the true function $f^*(\xbf)=\mathrm{Tr}[\Hbf\rhobf(\xbf)]$ is generated by an observable $\Hbf$ drawn uniformly at random from the sphere of operators with Hilbert--Schmidt norm $B$ (equivalent to normalized Gaussian, see Appendix~\ref{app:pfthm2}), then with probability at least $1-\delta$, the misspecification error is bounded as 
\begin{align*}
    \varepsilon_b&:=
    \inf_{f \in \mathcal{H}_{\kappa_{\mathbb{S}_b}}(\mathcal{X})}
    \Vert f-f^*\Vert_{\infty}\\
    &\leq
    B\sqrt{
    \left(1-\frac{\sum_{w=0}^{b}\binom{n}{w}3^w}{4^n}\right)
    +
    \sqrt{\frac{\ln(1/\delta)}{4^n}}
    } .
\end{align*}
Moreover, if $(b+1)/n\ge 3/4$, then
\begin{align}
    \varepsilon_b \leq  B\sqrt{ e^{ -  \frac{8n}{3} \Big(\frac{b+1}{n}- \frac{3}{4}\Big)^2}+ \sqrt{\frac{\ln(1/\delta)}{4^n}} }. \label{eq:LPQK_misspecificationerror}
\end{align}
 \end{enumerate}
\end{theorem}
See Appendix~\ref{app:pfthm2} for the proof. Part (i) shows that the RKHS induced by LPQKs is a \emph{subspace} of the full quantum RKHS. Part (ii) should be interpreted as a conservative baseline: for a dense uniformly random observable, the error scales with the fraction of Pauli-weight components discarded by the projection. Thus aggressive dimension reduction is not guaranteed to have small error in this worst-case model. Instead, the practical benefit of LPQK  arises when the target observable or circuit has additional structure, such as locality or a rapidly decaying high-weight Pauli tail; see Appendix~\ref{app:lpqk_structured_observables} for more details.

For this misspecified setting, where the true reward function $f^* \not \in \mathcal{H}_{\kappa_{\mathbb{S}_b}}(\mathcal{X})$, we can leverage the EC-GP-UCB algorithm  of \citet{bogunovic2021misspecified}, which yields the following cumulative regret with probability at least $1-\delta$, 
\begin{align}
\hspace{-0.2cm}\!R_T 
\!=\!
\mathcal{O}\big(
  B\sqrt{\gamma_T T}
  \!+\!
  \sqrt{ (\ln(1/\delta) \!+\! \gamma_T)\gamma_TT}
  \!+\!
  \varepsilon_b T\sqrt{\gamma_T}
\big),\! \label{eq:regret_ECGPUCB}
\end{align}
where $B<\infty$ is the RKHS-norm of the approximating function $f$, i.e., $\Vert f \Vert_{\mathcal{H}_{\kappa_Q}} \leq B$.
Compared to the regret achieved in the realizable setting (Theorem~\ref{thm:GP-UCB}), there is an additional $\mathcal{O}(\varepsilon_b T \sqrt{\gamma_T})$ term in  the regret that depends on the misspecification error $\varepsilon_b$. For the summed LPQK kernel, this misspecification error depends on the sub-system size $b$ as shown in \eqref{eq:LPQK_misspecificationerror}. Additionally,
as shown in Appendix~\ref{app:pfthm2}, the information gain  of the approximate kernel $\kappa_{\mathbb{S}_b}$ scales as $\gamma_T=\mathcal{O}\big( (\sum_{w=0}^{b} \binom{n}{w} 3^w ) \log T \big)$ with $\sum_{w=0}^{b} \binom{n}{w} 3^w \ll 4^n$, significantly smaller than using the $n$-qubit full fidelity kernel. Consequently, the cumulative regret of EC-GP-UCB algorithm with summed LPQK kernel can be optimized by effectively balancing the information gain $\gamma_T$ and the misspecification error $\varepsilon_b$ by choosing appropriate size $b$ of the sub-system. For structured observables with small or rapidly decaying high-weight Pauli components (cf. Appendix~\ref{app:lpqk_structured_observables}), much smaller values of $b$ can substantially reduce the effective dimension while keeping $\varepsilon_b$ small.

\subsection{RFF-Based SquareCB Algorithm} \label{sec:rff}

We now investigate RFF-based approximation of the quantum fidelity kernel $\kappa_Q(\xbf,\xbf')$ \citep{landman2022classicapprox}, leveraging the Fourier representation of quantum kernels \citep{schuld2021quantum}. For exposition, we focus on the simple circuit family in \eqref{eq:quantumcircuit}, where $\Gbf_i=\Gbf$ for all $i$ with an eigenspectrum $\{\lambda_1,\hdots,\lambda_{d_\Gbf}\}$, and $d_\Gbf \leq 2^n$. Extensions to repeated and multi-layer data encodings admit analogous Fourier expansions and can be derived similarly \citep{schuld2008effect}. 
Under this encoding, $\kappa_Q$ admits the Fourier expansion  
$
\kappa_Q(\xbf,\xbf')
=
\sum_{\mathbf{s},\mathbf{t}\in\Omega}
c_{\mathbf{s},\mathbf{t}} \exp(i\,\mathbf{s}\cdot \xbf)\exp(i\,\mathbf{t}\cdot \xbf'),
$
where the frequency set $\Omega$ contains vectors of the form $\{\boldsymbol{\Lambda}_j-\boldsymbol{\Lambda}_k\}$, with $\boldsymbol{\Lambda}_j=(\lambda_{j_1}, \hdots \lambda_{j_d})$ for indices $j_1,\hdots,j_d \in \{1,\ldots,d_\Gbf\}$, and the Fourier coefficients $c_{\mathbf{s},\mathbf{t}} \in \mathbb{C}$ satisfy $c_{\mathbf{s},\mathbf{t}}=c^*_{-\mathbf{s},-\mathbf{t}}$. The frequency spectrum depends solely on the eigenspectrum of the generator $\Gbf$. 

When the quantum kernel $\kappa_Q(\xbf,\xbf')$ is \textit{shift-invariant}, i.e., $\kappa_Q(\xbf,\xbf')=g(\xbf-\xbf')$~\citep{schuld2021quantum} for some function $g: \mathbb{R}^d \rightarrow \mathbb{R}$,  Bochner's theorem \citep{rudin2017fourier} ensures that $\kappa_Q$ admits a Fourier integral representation:  $\kappa_Q(\xbf,\xbf')=\int_{  \Omega}p(\boldsymbol{\omega}) e^{i \boldsymbol{\omega}^{\top} (\xbf-\xbf')} d\boldsymbol{\omega}$, where $p(\boldsymbol{\omega})$ is a probability distribution over $\boldsymbol{\omega} \in \Omega$ (assuming $\kappa_Q$ is normalized).  For the circuit family above, shift-invariance arises when  each generator $\Gbf_i=\Gbf$ encodes input $x_i$ into distinct, non-overlapping qubits.

RFF approximates this kernel as \citep{rahmi2007rff} \begin{align}
  \kappa_Q(\mathbf{x},\mathbf{x}')
  \;\approx\;
  \boldsymbol{\phi}_{\rm RFF}(\mathbf{x})^\top \,\boldsymbol{\phi}_{\rm RFF}(\mathbf{x}') =:\kappa_{{\rm RFF}}(\xbf,\xbf'), \label{eq:RFF_kernel}
\end{align} where $\boldsymbol{\phi}_{\rm RFF}:\Xcal \rightarrow \mathbb{R}^D$ is the \textit{feature map} obtained by  drawing $D$ frequency samples $\{\boldsymbol{\omega}_j\}$ from
$p(\boldsymbol{\omega})$ and random phases $\{b_j\}$
from the uniform distribution over $[0,2\pi]$:
\[
    \boldsymbol{\phi}_{\rm RFF}(\mathbf{x}) =
    \sqrt{\tfrac{2}{D}}\,
    \begin{bmatrix}
       \cos\bigl(\boldsymbol{\omega}_1 \!\cdot \!\mathbf{x} \!+\! b_1\bigr)  & \!\!\!\ldots\!\!\! &
       \cos\bigl(\boldsymbol{\omega}_D \!\cdot \!\mathbf{x} \!+\! b_D\bigr)
    \end{bmatrix}^\top.
\]
For many standard quantum circuits, the frequency set $\Omega$ and coefficients $\{c_{\textbf{s},\textbf{t}}\}$ can be derived explicitly, allowing one to directly sample frequencies from $\Omega$ for RFF~\citep{schuld2021quantum}. See Appendix~\ref{app:rff} for more details.

\textbf{Misspecified Linear Bandit via RFF.}  The RFF kernel \eqref{eq:RFF_kernel} induces the following RKHS,
\begin{align}
    \mathcal{H}_{\kappa_{\rm RFF}}(\mathcal{X}) =\{ f(\cdot)=\mathbf{w}^{\top}\boldsymbol{\phi}_{\rm RFF}(\cdot): \mathbf{w} \in \mathbb{R}^D\}, \label{eq:RFF_RKHS}
\end{align} consisting of functions linear in $\phibf_{\rm RFF}(\cdot)$. While $\kappa_{\rm RFF}$ can define a GP model, its finite dimensionality also allows us to treat the problem as a \textit{misspecified linear bandit} with the RFF feature. Concretely, we model the unknown mean reward function as $\tilde{f}(\cdot)=\mathbf{w}^{\top}\boldsymbol{\phi}_{\rm RFF}(\cdot)  \in \mathcal{H}_{\kappa_{\rm RFF}}(\mathcal{X})$ such that the true reward function $f^*(\xbf)=\tilde{f}(\xbf) + \xi(\xbf)$ with $\sup_{\xbf}|\xi(\xbf)|\leq \varepsilon_D$. 
 In this setting,  the SquareCB algorithm \cite{foster2020mislin} achieves an expected  regret of $\mathbb E[R_T] =\tilde{\mathcal{O}}(\sqrt{DKT}+\varepsilon_D \sqrt{K}T),$ where the expectation is over the randomness of action policy and $K= |\mathcal{X}|$ is the number of actions. 

\begin{theorem} \label{thm:regret_RFF}
Let $f^* \in \mathcal{H}_{\kappa_Q}(\mathcal{X})$ with $\Vert f^* \Vert_{\mathcal{H}_{\kappa_Q}} \leq B < \infty$, where $\kappa_Q (\cdot, \cdot)$ is an $n$-qubit shift-invariant quantum kernel. 
Consider $\mathcal{H}_{\kappa_{\rm RFF}}(\mathcal{X})$ as in \eqref{eq:RFF_RKHS}  obtained by sampling $D$ Fourier features of $\kappa_Q$.
Then, with probability at least $1-\delta$ over the sampled Fourier features,
there exists $\tilde f\in \mathcal H_{\kappa_{\mathrm{RFF}}}(\mathcal X)$ such that the
misspecification error satisfies
\begin{align}
&\sup_{\xbf\in\mathcal{X}}|f^*(\xbf)-\tilde f(\xbf)|
 <\varepsilon_D \quad\mbox{where} \label{eq:RFF_misspecification}\\
&\varepsilon_D :=
\frac{4B|\Omega|}{\sqrt{D}}
\Bigl(
\sqrt{\log\frac{1}{\delta}}
+
4B\!_\Xcal\,\omega_{\max}
\Bigr),\notag 
\end{align}
and where $B_{\mathcal X} := \sup_{\xbf\in \mathcal X}\|\xbf\|_2$ and 
$\omega_{\max} := \max_{\omegabf\in\Omega}\|\boldsymbol{\omega}\|_2$ are the domain and frequency radii respectively. 
Using  SquareCB with $\mathcal{H}_{\kappa_{\text{RFF}}}(\mathcal{X})$ yields an expected cumulative  regret of
\begin{align}
\!&\mathbb E[R_T] \!\leq\! \min \Bigl\{\!
2\sqrt{4^n KT   \log\big(1\!+\!\tfrac{T}{4^n}\big)}, 2\sqrt{K}T^{3/4}\sqrt{\log T} \nonumber
\\
\!&\qquad+
20\sqrt{K}T^{3/4}B|\Omega|\Big(\sqrt{\log T}+4B\!_\Xcal\,\omega_{\max}\Big) 
\Bigr\},\label{exp_reg_rff}
\end{align} where the expectation is over SquareCB's random action sampling, the stochastic rewards, as well as the sampled Fourier features.
\end{theorem}

See Appendix~\ref{app:pfthm3} for proof.  In the regret bound in 
\eqref{exp_reg_rff}, the first term corresponds to applying SquareCB with the full quantum model (no misspecification), scaling as $\tilde{\mathcal{O}}(\sqrt{4^n KT})$. The second term arises from using SquareCB in the misspecified setting, with the misspecification error in~\eqref{eq:RFF_misspecification} optimized with the choice of $D=\sqrt{T}$, giving a regret of  $\tilde{\mathcal{O}}(\sqrt{K} T^{3/4})$.
This term is smaller than the $\tilde{\mathcal{O}}(\sqrt{4^n KT})$ regret  when $4^n \gtrsim \sqrt{T} $, i.e., for $n =\Omega( \log  T)$.  This analysis demonstrates that RFF-based surrogate models outperform full quantum fidelity kernel methods in regimes where the kernel complexity grows faster than the required approximation capacity. Furthermore, the analysis provides a theoretical prescription for selecting the optimal dimension $D$ to balance approximation error and model complexity.

We remark that, if the horizon $T$ is not known in advance, it can be chosen in an anytime manner via a standard doubling trick (cf.~\citet{besson2018doublingtrick}), which preserves the  regret order up to logarithmic factors in~$T$.

\textbf{Alternative Sampling and Possible Faster Error Decay.} In practice, instead of sampling $\boldsymbol{\omega}\sim p(\boldsymbol{\omega})$, one can use structured sampling (e.g., prioritizing frequencies with the largest Fourier coefficients). When the Fourier expansion of the kernel is explicitly known, this can tighten misspecification bounds \citep{schuld2021quantumkernel}, particularly if the coefficients decay rapidly---such as polynomially: $\varepsilon_D = \mathcal{O}(1/D)$   or exponentially: $\varepsilon_D = \exp(-\Omega(D))$. In these regimes, both SquareCB and EC-GP-UCB benefit from faster convergence and improved regret guarantees.

\textbf{Computational Complexity.}
GP or kernelized regression typically require matrix inversion operations with $\mathcal{O}(T^3)$ cost. In contrast, approximating $ \kappa_Q $ with a $ D $-dimensional feature map yields a linear model, solvable by ridge regression in $\mathcal{O}(TD^2 + D^3) $. As long as $ D < T $, the RFF approach provides substantial computational savings. Additionally, for quantum kernels involving a large number of qubits $ n $, kernel evaluations can become expensive; in such cases, RFF-based methods are more practical and scalable (more details are provided in Appendix~\ref{app:runningtime}).


\subsection{P-Greedy  Kernel Approximation}

While RFF is effective for \emph{shift-invariant} kernels with accessible feature maps, many quantum kernels lack these properties or are challenging to analyze. As an alternative,  P-greedy  \citep{DeMarchi2005pgreedy, takemori2021approximation} requires only access to kernel evaluations $\kappa_Q(\xbf,\xbf')$.

The P-greedy algorithm constructs a low-dimensional subspace $V(X_D) = \text{span}\left\{ \kappa_Q(\cdot,\xibf_1),\ldots,\kappa_Q(\cdot,\xibf_D)   \right\} \subset \mathcal{H}_{\kappa_Q}(\mathcal{X})$, spanned by kernel evaluated at a chosen subset of points $X_D=\{\xibf_1,\ldots,\xibf_D\} \subset \mathcal{X}$. This subspace is selected iteratively: At each iteration $k$, the algorithm greedily selects the point $\xibf_k \in \mathcal{X}$ that maximizes the power function: $$P_{V(X_{k-1})}(\xbf):=\sup_{f\in\mathcal{H}_{\kappa_Q}(\mathcal{X})\setminus\{0\}} \frac{|f(\xbf)-(\Pi_{V(X_{k-1})}f)(\xbf)|}{\|f\|_{\mathcal{H}_{\kappa_Q}}},$$ which identifies where the current subspace approximates the quantum RKHS worst. Here, $\Pi_{V(X)}: \mathcal{H}_{\kappa_Q}(\mathcal{X}) \rightarrow V(X)$  denotes the orthogonal projection onto $V(X)$. After $D$ iterations, the algorithm returns $X_D$, defining a subspace $V(X_D)$ that approximates the full quantum RKHS. The final feature map corresponds to the Newton basis formed via Gram--Schmidt orthonormalization of basis $\left\{ 
\kappa_Q(\cdot,\xibf_1),\ldots,\kappa_Q(\cdot,\xibf_D)   \right\}$, yielding an efficient low-rank approximation of the quantum kernel.

\paragraph{Performance and Spectral Decay.} The effectiveness of the P-greedy algorithm is governed by the spectral decay of the target quantum kernel \citep{jord2018kolmogorov}. For kernels with rapid eigenvalue decay, the algorithm may admit polynomial or exponential improvement with each greedily chosen point \citep{greedyalgo2012dev}.
Nevertheless, P-greedy typically achieves near-optimal performance: The misspecification error using $2D$ points satisfies \citep{greedyalgo2012dev} $\varepsilon_{2D} \leq B/\gamma\cdot \sqrt{2d_{D}}$ for some $0<\gamma<1$, where $d_D$ is the $D$-dimensional Kolmogorov width, representing the best possible $D$-dimensional approximation of the quantum RKHS $\mathcal{H}_{\kappa_Q}(\mathcal{X})$. 

As an example, consider quantum kernels $\kappa_Q$ with a Fourier expansion, $$\kappa_Q(\xbf,\xbf') = \sum_{\boldsymbol{\omega} \in \Omega} c_{\boldsymbol{\omega}} \, e^{i\langle \boldsymbol{\omega}, \xbf-\xbf'\rangle} .$$ 
The Kolmogorov width $d_D$ is controlled by the spectral tail with $$ d_D^2 \;=\; \sum_{j>D} c_j \;=\; \sum_{\boldsymbol{\omega} \notin \Omega_D} c_{\boldsymbol{\omega}},$$
where $\{c_j\}_{j\ge 1}$ are the eigenvalues in non-increasing order and $\Omega_D$ indexes the top-$D$ eigenvalues (see  Appendix~\ref{app:pgreedy_misspec} for a proof). Without additional circuit structure, fast spectral decay is generally not guaranteed. 
However, there exist many circuit families that produce non-uniform spectra, often leading to exponentially decaying $d_D$, and consequently,  $\varepsilon_D$ (see Appendix~\ref{app:pgreedy} for  examples).

Once an instance-specific (or circuit-specific) bound $\varepsilon_D$ is available for the P-greedy surrogate class, it can be directly integrated into the regret bounds of misspecified bandit algorithms  
to obtain a theoretically principled choice of the approximation dimension $D$.

\begin{figure*}[tbp]                
  \centering

  \begin{subfigure}{.3\textwidth}
    \includegraphics[width=\linewidth]{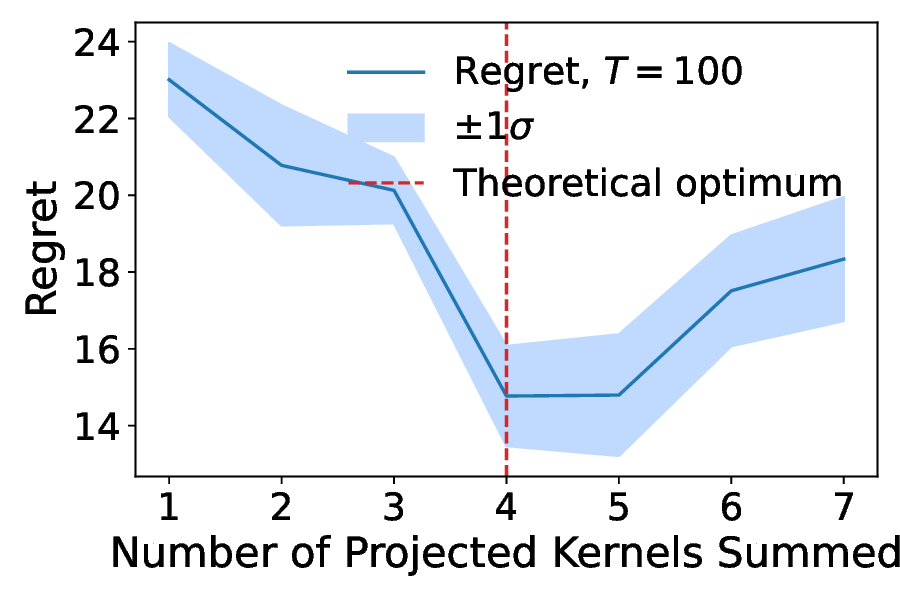}
    \caption{Projected Quantum Kernels}
  \end{subfigure}\hfill
  \begin{subfigure}{.3\textwidth}
    \includegraphics[width=\linewidth]{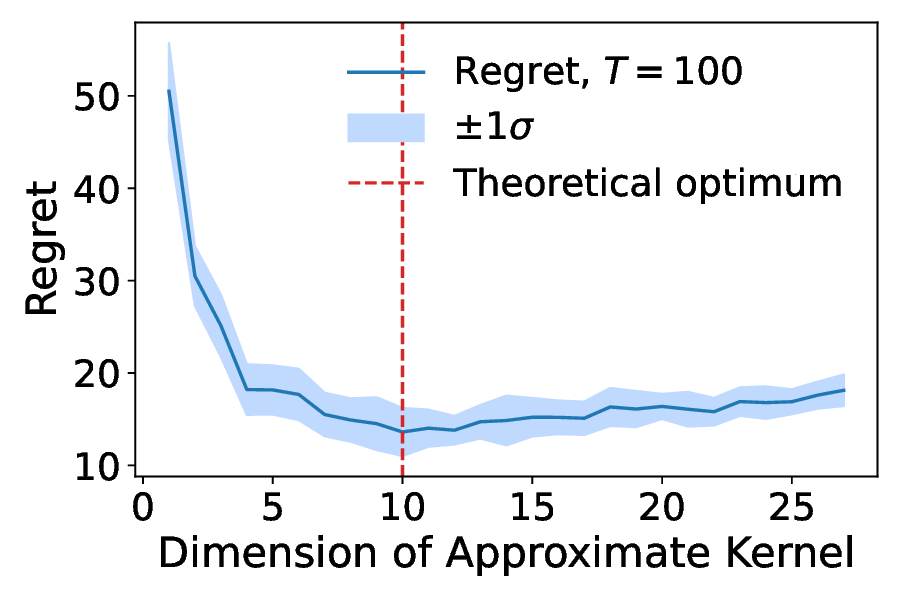}
    \caption{Random Fourier Features}
  \end{subfigure}\hfill%
  \begin{subfigure}{.3\textwidth}
    \includegraphics[width=\linewidth]{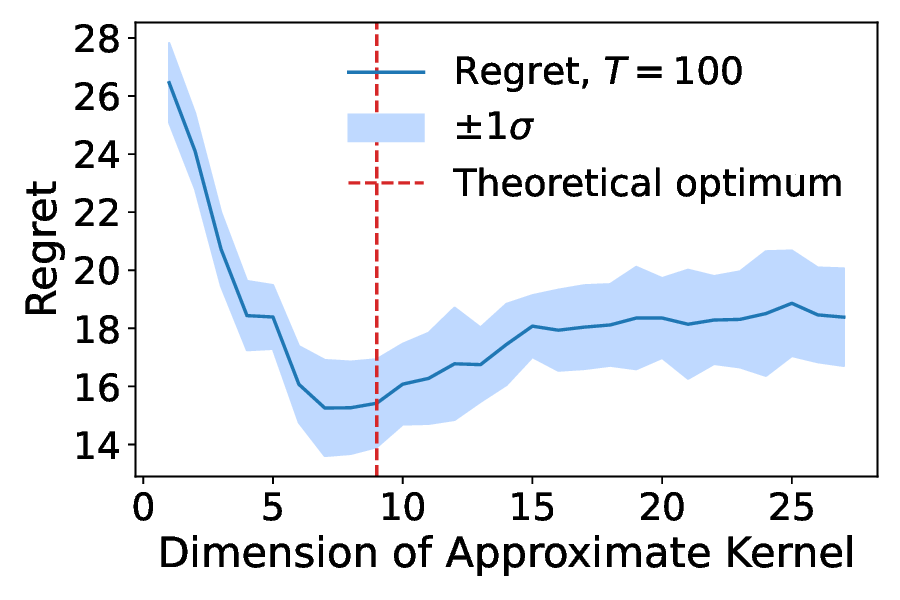}
    \caption{P-greedy}
  \end{subfigure}

  \caption{Cumulative regret of SquareCB against approximation model complexity for $(a)$ projected quantum kernels, $(b)$ RFF and $(c)$ P-greedy approximation.}
  \label{fig:exp1_SquareCB}

\end{figure*}

\begin{figure*}[tbp]                
  \centering

  \begin{subfigure}{.3\textwidth}
    \includegraphics[width=\linewidth]{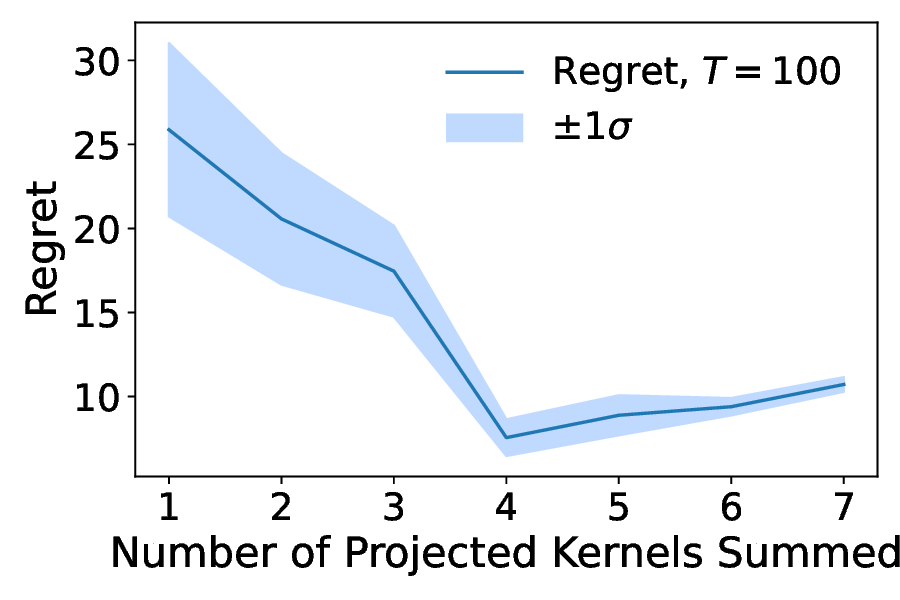}
    \caption{Projected Quantum Kernels}
  \end{subfigure}\hfill
  \begin{subfigure}{.3\textwidth}
    \includegraphics[width=\linewidth]{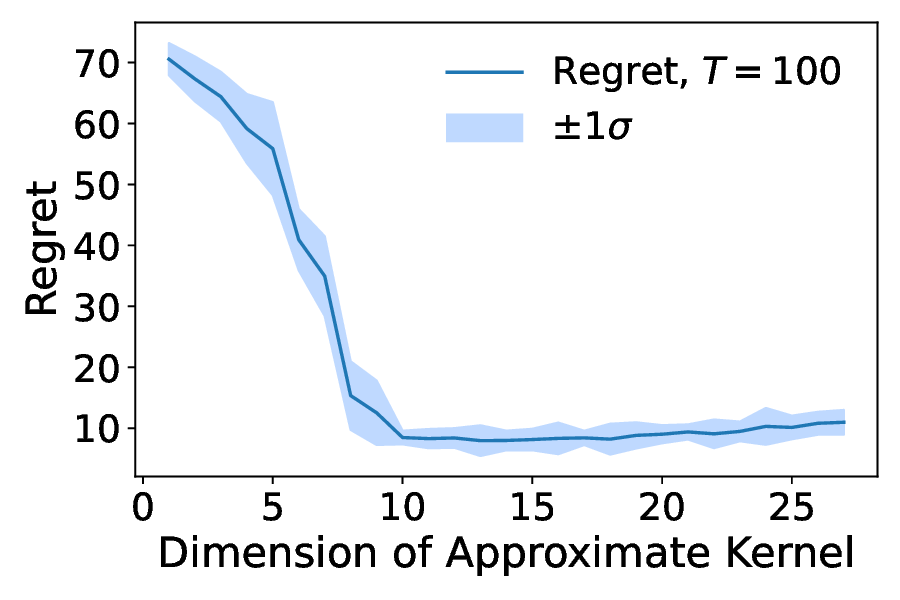}
    \caption{Random Fourier Features}
  \end{subfigure}\hfill%
  \begin{subfigure}{.3\textwidth}
    \includegraphics[width=\linewidth]{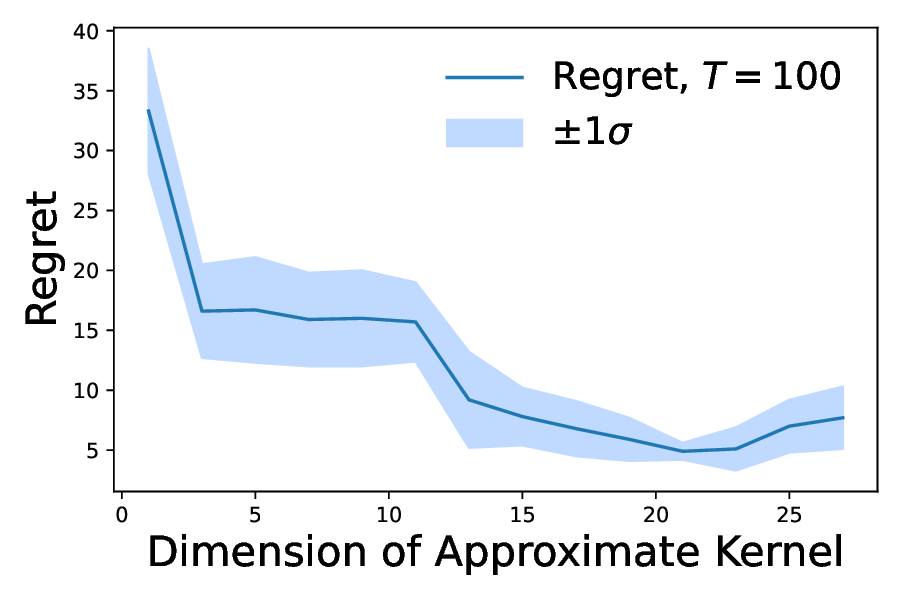}
    \caption{P-greedy}
  \end{subfigure}

  \caption{Cumulative regret of EC-GP-UCB against approximation model complexity for $(a)$ projected quantum kernels, $(b)$ RFF and $(c)$ P-greedy approximation.}
  \label{fig:exp1_ECGPUCB}

\end{figure*}

\section{Numerical Experiments}
\label{sec:experiments}
We demonstrate the effectiveness of our kernel approximation approach on a variety of tasks using classical bandit algorithms such as EC-GP-UCB \citep{bogunovic2021misspecified} or SquareCB \citep{foster2020mislin}. Further implementation details are in Appendix~\ref{app:expdetails}.

\subsection{GP Optimization with Synthetic Quantum Functions} \label{exp1}

Our first synthetic task uses a true reward function $f^*$ sampled from a three-qubit quantum kernel (extension to larger $n=6$ qubit systems can be found in Appendix~\ref{app:more_exp1}). We construct a parameterized quantum circuit in PennyLane with Pauli-X data-encoding gates followed by random rotation and entangling gates. This circuit defines a global fidelity quantum kernel $\kappa_Q$, which we use to generate random functions $f^*$ over the domain $[0,2\pi]^3$. Observed rewards are corrupted with i.i.d. Gaussian noise (variance $\sigma^2=0.01$) and normalized to $[0,1]$. 



We evaluate three kernel-approximation strategies: projected quantum kernels, RFF, and P-greedy, each paired with \emph{SquareCB} and \emph{EC-GP-UCB}. Model complexity is controlled by the number of summed LPQKs, the RFF dimension $D$, or the number of basis elements in P-greedy. 
Figures~\ref{fig:exp1_SquareCB} and \ref{fig:exp1_ECGPUCB} plot the cumulative regret as a function of model complexity for the three approximation methods, using SquareCB and EC-GP-UCB respectively. The rightmost point in each plot corresponds to the full quantum kernel (i.e., maximum model complexity). Each algorithm was run for $T=100$ rounds, with results averaged over 30 trials.

Across all settings, we observe the characteristic ``U-shaped'' trend in regret. At low model complexity, the surrogate kernel is too coarse to capture $f^*$ accurately, leading to underfitting and high regret. As model complexity increases, the approximation error decreases, improving performance. However, beyond an optimal point, further increase in complexity leads to higher regret due to increased information gain, mirroring overfitting in classical learning.

\textbf{Choosing the approximate kernel dimension.} Theorem~\ref{thm:regret_RFF} suggests setting the RFF dimension $D=\sqrt{T}$ to minimize the regret bound. For $T=100$, this yields $D=10$, which is highlighted by the red vertical line in Figure~\ref{fig:exp1_SquareCB}$(b)$. Note that our theoretical choice of $D$ aligns perfectly with the best empirical performance.

For P-greedy, the relationship between the approximation dimension $D$ and the misspecification error $\varepsilon$ is circuit-dependent. One can analyze the circuit structure---empirically or analytically---to estimate how $\varepsilon$ varies with $D$. Plugging this estimate into misspecified bandit regret bounds (e.g., SquareCB) then allows us to select the optimal $D$. 
Figure~\ref{fig:exp1_SquareCB_norm} illustrates the empirical relationship between $\varepsilon$ and $D$ for both the projected kernel and P-greedy methods. Substituting these estimates into SquareCB's regret bounds yields a theoretically optimal $D$, marked in red in Figures~\ref{fig:exp1_SquareCB}$(a)$ and $(c)$. These values closely match the observed performance peaks, validating the approach. The same principle applies to EC-GP-UCB.

\begin{figure}[t!]                
  \centering

  \begin{subfigure}{.23\textwidth}
    \includegraphics[width=\linewidth]{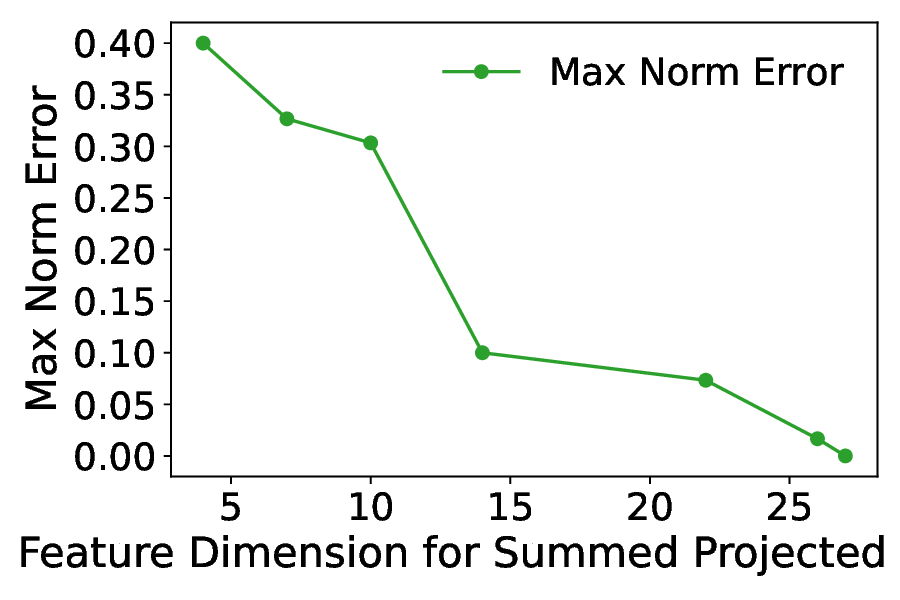}
    \caption{Projected Quantum Kernels}
  \end{subfigure}
  \begin{subfigure}{.23\textwidth}
    \includegraphics[width=\linewidth]{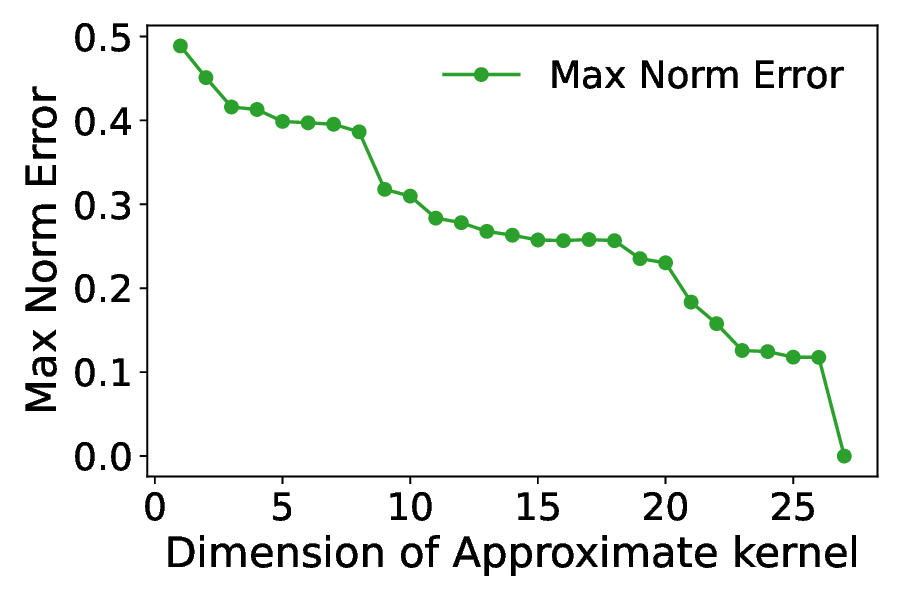}
    \caption{P-greedy}
  \end{subfigure}

  \caption{Uniform (max-norm) approximation error against the kernel dimension 
  (or effective dimension) for $(a)$ projected quantum kernels and  $(b)$ P-greedy algorithm.}
  \label{fig:exp1_SquareCB_norm}

\end{figure}

\begin{figure}[t!]               
  \centering
  \begin{subfigure}{.23\textwidth}
    \includegraphics[width=\linewidth]{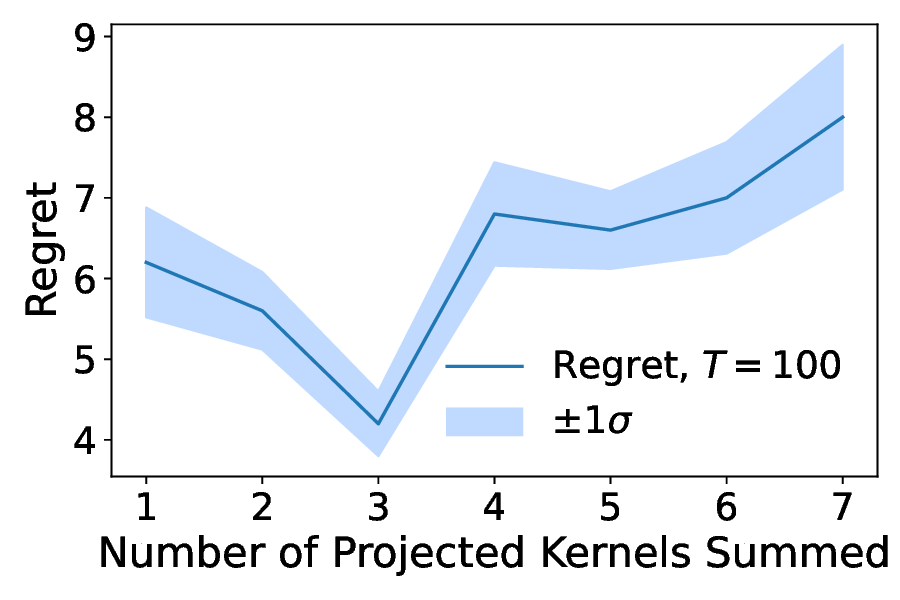}
    \caption{Projected Quantum Kernels}
  \end{subfigure} 
  \begin{subfigure}{.23\textwidth}
    \includegraphics[width=\linewidth]{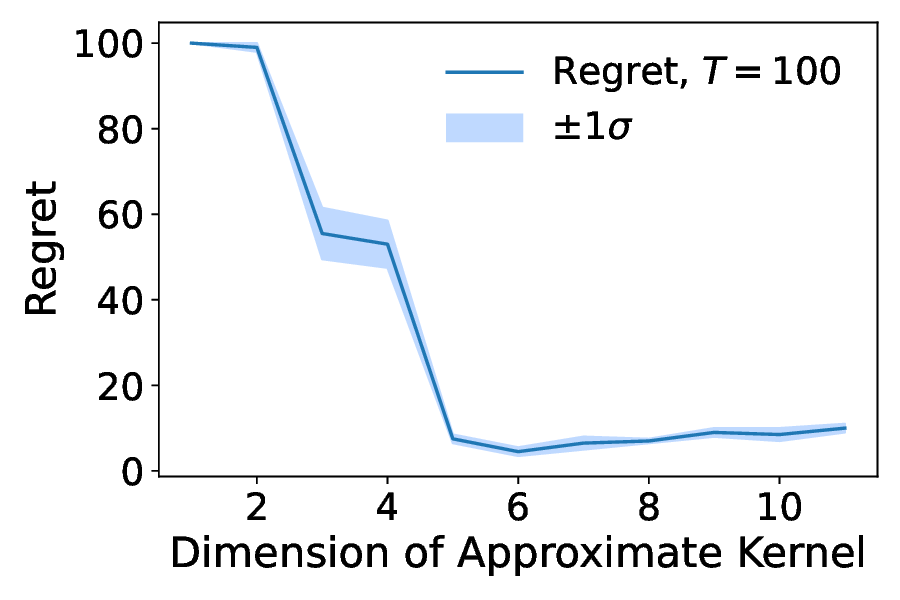}
    \caption{P-greedy}
  \end{subfigure}
  \caption{Cumulative regrets of EC-GP-UCB after $T=100$ iterations as a function of the model complexity for (a) projected  quantum kernels and (b) P-greedy approximation for   quantum phase classification.}
  \label{fig:exp2_ECGPUCB}

\end{figure}

\begin{figure}[t!]               
  \centering
  \begin{subfigure}{.235\textwidth}
    \includegraphics[width=\linewidth]{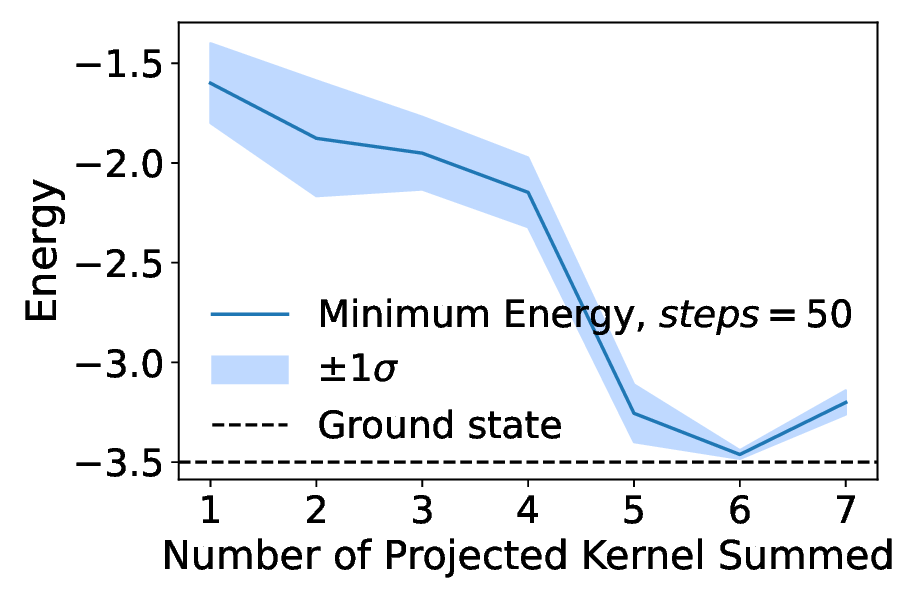}
    \caption{Projected Quantum Kernels}
  \end{subfigure}
  \begin{subfigure}{.235\textwidth}
    \includegraphics[width=\linewidth]{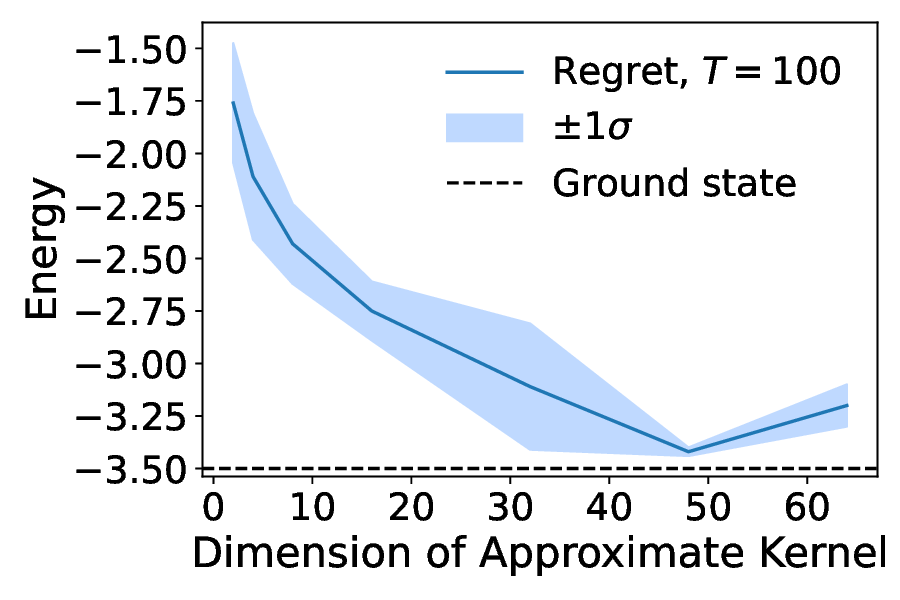}
    \caption{P-greedy}
  \end{subfigure}

  \caption{Bayesian optimization of VQE with different kernel approximation approaches, $n=3$}
  \label{fig:exp3_proj}

\end{figure}

\subsection{Phase classification} \label{exp2}

We next evaluate our approach on a \emph{quantum phase classification} task inspired by \citet{Caro2022}. 
We consider the generalized cluster Hamiltonian
$$
\Hbf_C(J_1,J_2)= \sum^n_{j=1} (\Zbf_j-J_1 \Xbf_j \Xbf_{j+1}-J_2 \Xbf_{j-1} \Zbf_j \Xbf_{j+1})
,$$
where $\Xbf_j$ and $\Zbf_j$ are Pauli-$X$ and Pauli-$Z$ operators, respectively, acting on qubit $j$. By varying the coupling coefficients $J_1$ and $J_2$, the ground state of the Hamiltonian can fall into one of four distinct phases: symmetry-protected topological (I), ferromagnetic (II), anti-ferromagnetic (III), or trivial (IV).
Our task is to identify parameters $(J_1, J_2)$ corresponding to the \emph{ferromagnetic phase} (II).

In each round, we sample a point $\xbf=(J_1, J_2)$, compute the ground state $\ket{\psi_{\xbf}}$ of $\Hbf_C$, and observe  $y=f^*(\xbf)+\eta_t$ where $\eta_t \sim \mathcal{N}(0,0.01)$.
Crucially, for this physical system, the reward function $f^*(\xbf)$ can be represented within a quantum RKHS via the normalized total magnetization squared,  $$\Hbf = \frac{1}{n^2} \sum_{j,k=1}^n \Xbf_j \Xbf_k,$$ as $f^*(\xbf) = \mathrm{Tr}(\Hbf |\psi_\xbf\rangle\langle\psi_\xbf|),$ which is approximately $1$ within the ferromagnetic phase and $0$ elsewhere.
The objective is to minimize regret for choosing points not in phase~II.

We apply EC-GP-UCB using approximate kernels (projected kernels and P-greedy) derived from a three-qubit \emph{quantum kernel} implemented via a parameterized circuit with random single-qubit rotations and entangling gates acting on the initial ground state. Figure~\ref{fig:exp2_ECGPUCB} shows the regret of EC-GP-UCB over $T=100$ iterations against model complexity. Both approximation methods  significantly outperform the full quantum kernel (rightmost points) at their optimal complexity. This suggests that lower-dimensional kernels suffice to distinguish phase II, while avoiding the higher information-gain penalty of the full kernel. 

\subsection{Variational Quantum Eigensolver} \label{exp3}
In our final experiment, we follow the setup in \citet{nicoli2023physics} to approximate the ground state of the XYZ Heisenberg Hamiltonian:
\begin{align*}
    \Hbf &= -\sum^{n-1}_{j=1}
(
J_X \Xbf_j \Xbf_{j+1} + 
J_Y \Ybf_j \Ybf_{j+1} + 
J_Z \Zbf_j \Zbf_{j+1}
)  \\ 
&\qquad -
\sum^n_{j=1}
(
h_X \Xbf_j + h_Y \Ybf_j +
h_Z \Zbf_j)
\end{align*}
We employ a 3-qubit ``Efficient SU(2)'' circuit \citep{nicoli2023physics} as the parameterized ansatz $\Ubf(\mathbf{x})$, acting on an initial state $\ket{\psi_0}$. Our goal is to find the parameters $\mathbf{x}$ that minimize the energy function
$$f^*(\mathbf{x}) = \bra{\psi_0}\Ubf(\mathbf{x})^{\dag}\Hbf \Ubf(\mathbf{x})\ket{\psi_0}.$$
We use a GP model with kernel from the SU(2) circuit, performing Bayesian optimization with the Expected Improvement (EI) acquisition function, optimized via L-BFGS. Each run consists of 50 optimization steps, averaged over 30 trials.

Figure~\ref{fig:exp3_proj} illustrates the energy landscape as a function of model complexity for the projected quantum kernel and P-greedy approximations. The full kernel appears excessively complex for this task, leading to slower convergence. In contrast, simpler approximations achieve faster convergence. While this is an optimization task (minimizing energy) rather than a regret-minimization bandit problem, the results still show that choosing an appropriate kernel complexity is crucial for efficiency. This demonstrates that the balance between expressivity and sample efficiency extends beyond standard bandit optimization contexts.



\section{Conclusion}
We developed approximate GP algorithms for quantum kernel bandits, addressing the challenges of scalability and sample efficiency in NISQ-era learning tasks. By exploiting reduced quantum subspaces and classical approximations, our methods achieve a practical trade-off between information gain and model misspecification. Theoretical regret bounds, supported by empirical results on both synthetic and quantum tasks, demonstrate the effectiveness of our approach.
Finally, the techniques discussed here are applicable to high dimensional classical kernels, offering  lower regret and reduced computational complexity (see Appendix~\ref{app:classicalrkhs}). For limitations and future work, see Appendix~\ref{app:limitations}.




\begin{acknowledgements} 
V.~Y.~F.~Tan   is supported by a Singapore Ministry of Education (MOE) AcRF Tier~1 grant (A-8002934-00-00) and an AcRF  Tier~2 grant (A-8004062-00-00). S.~T.~Jose  is supported through Royal Society International Travel Grant (Reference no: IES\textbackslash{}R2\textbackslash{}242104) and through EPSRC Quantum Technologies Career Acceleration Fellowship (UKRI1218).
\end{acknowledgements}

\newpage
\bibliography{uai2026-template}

\newpage

\onecolumn

\title{Balancing Expressivity and Learnability in Quantum Kernel Bandit Optimization\\(Supplementary Material)}
\maketitle

\appendix

\section{Details on Gaussian Processes} \label{app:GPdetails}

Suppose we have a set of noisy samples 
\[
\ybf_T \;=\; \bigl[y_1,\dots,y_T\bigr]^\top
\]
at the points 
\[
\Abf_T \;=\; \bigl[ \xbf_1,\dots,\xbf_T\bigr],
\]
where \(y_t = f(\xbf_t) + \eta_t\) with 
\(\eta_t \sim \mathcal{N}\bigl(0, \sigma^2\bigr)\) i.i.d.~Gaussian noise.
Then the posterior on \(f\) is again a GP distribution with mean 
\(\mu_T(\xbf)\) and variance \(\sigma^2_T(\xbf)\) 
given by the standard formulas:
\begin{align}
\mu_T(\xbf) \; &=\; 
   \kbf_T(\xbf)^\top \bigl(\Kbf_T + \sigma^2 \Ibf\bigr)^{-1}\,\ybf_T,\\
\sigma^2_T(\xbf) \;&=\;
   \kappa(\xbf, \xbf) 
   \;-\;
   \kbf_T(\xbf)^\top
   \bigl(\Kbf_T + \sigma^2 \Ibf\bigr)^{-1}
   \kbf_T(\xbf),
\end{align}
where \(\kbf_T(\xbf)\) is the length-\(T\)-vector \([\,\kappa(\xbf_1,\xbf),\dots,\kappa(\xbf_T,\xbf)\,]^\top\), 
and \(\Kbf_T\) is the positive semidefinite kernel matrix 
\(\bigl[\kappa(\xbf,\xbf')\bigr]_{\xbf,\xbf'\in \Abf_T}\).

\section{Information Gain of Quantum Fidelity Kernel}\label{app:Informationgain}
Consider a $n$-qubit system, where each classical input $\xbf$ produces a density matrix $\rhobf(\xbf) \in \mathbb{C}^{2^n \times 2^n}$. Let us define the quantum fidelity kernel as
$$
\kappa_Q(\xbf,\xbf') \;=\; \mathrm{tr}\bigl(\rhobf(\xbf)^\dag\rhobf(\xbf')\bigr).
$$
Note that the associated RKHS has dimension at most $4^n$ \citep{kubler2021inductive}.

It is well known that for any $D$-dimensional linear kernel, the information gain $\gamma_T$ scales as $\mathcal O(D\log T)$ \citep{srinivas2010}. A similar argument reveals that $\gamma_T$ of $\kappa_Q$ is $\mathcal O(4^n \log T)$.
Concretely, we can define the feature map 
\[
\phibf(\xbf) \;=\; \mathrm{vec}\bigl(\rhobf(\xbf)\bigr)
\;\;\in\;\;\mathbb{C}^{4^n},
\]
satisfying $\|\phibf(\xbf)\|\leq 1$ \citep{schuld2021quantumkernel},  where ${\rm vec}(\cdot)$ flattens a $2^n\times 2^n$ matrix into a vector of dimension $4^n$. Subsequently  $$
\kappa_Q(\xbf,\xbf') \;=\; \mathrm{Tr}\bigl(\rhobf(\xbf)^\dag\rhobf(\xbf')\bigr) = \phibf(\xbf)^\dag\phibf(\xbf').
$$

Corresponding to the observed $T$ data points $\{\xbf_1,\hdots,\xbf_T\}$, we then collect the feature vectors into the following matrix:
\[
\mathbf{X}_T \;=\;
\begin{bmatrix}
| & | & \cdots & | \\
\phibf(\xbf_1) & \phibf(\xbf_2) & \cdots & \phibf(\xbf_T) \\
| & | & \cdots & |
\end{bmatrix}
\;\in\;\mathbb{C}^{4^n \times T}.
\]

Recall that, for reward function $\mathbf{f}_A=[f^*(\xbf)]_{\xbf \in A}$ defined on domain $A \subset \mathcal{X}$, the \emph{information gain} on observing $T$ points $\{\xbf_1,...,\xbf_T\}$ with observations $\mathbf{y}_A$ is $\gamma_T=\max_{A\subset \mathcal{X}:|A|=T}I(\mathbf{y}_A;\mathbf{f}_A)$.

Furthermore,
\begin{align*}
I(\mathbf{y}_A;\mathbf{f}_A) &= \frac{1}{2}\log |\mathbf{I}+\sigma^{-2}\mathbf{K}_A| \\
&= \frac{1}{2}\log |\mathbf{I}+\sigma^{-2}\mathbf{X}_T^\dag\mathbf{X}_T| \\
&= \frac{1}{2}\log |\mathbf{I}+\sigma^{-2}\mathbf{X}_T\mathbf{X}_T^\dag| \\
&= \frac{1}{2} \log \prod^{4^n}_{i=1}(1+\sigma^{-2}\lambda_i) \\
&\leq \frac{1}{2} 4^n\log\bigl(1+\sigma^{-2} \cdot O(T)\bigr)
\end{align*}
where $\lambda_i$'s are the eigenvalues of $\mathbf{X}_T\mathbf{X}_T^\dag$. The last inequality follows since the largest eigenvalue of $\mathbf{X}_T\mathbf{X}_T^\dag$ is $O(T)$, since all $\|\phibf(\xbf)\|\leq 1$. Hence, $\gamma_T$ for the quantum kernel is upper bounded by $\mathcal O(4^n \log T)$.

\section{Proof and Discussion on Theorem 2} \label{app:pfthm2}

\subsection{Proof of Theorem \ref{thm:lpqk} (i)} \label{c1}

We will show that the RKHS associated with a (summed) projected kernel is a subspace of the full kernel's RKHS.

\begin{proof}
By standard arguments (e.g., the Moore--Aronszajn theorem or see \citet{kubler2021inductive}), the $n$-qubit quantum kernel $$
\kappa_Q(\xbf,\xbf') = \mathrm{Tr}[\rhobf(\xbf)\rhobf(\xbf')^{\dag}]
$$
defines the RKHS  $$
\mathcal{H}_{\kappa_Q}(\mathcal{X})=\{f:f(\cdot)=\mathrm{Tr}[\rhobf(\cdot)\Hbf],\Hbf\in \mathbb C^{2^n \times 2^n}, \Hbf=\Hbf^{\dag}\},
$$ of real-valued functions $f:\Xcal \rightarrow \mathbb R$. 

Likewise, for the projected quantum kernel onto the set of qubits $s \subseteq \{1,...,n\}$ 
$$
\kappa_Q^s(\xbf,\xbf')={\rm Tr}(\rhobf^s(\xbf)\rhobf^s(\xbf')),
$$ 
where the reduced density matrix $\rhobf^s(\mathbf x)$ is obtained from the full $\rhobf(\xbf)$ by tracing out the complementary qubits $s^c$ as follows
$$\rhobf^s(\xbf) = \mathrm{Tr}_{s^c}[\rhobf(\xbf)].$$
The associated RKHS $\mathcal{H}^s_{\kappa_Q}(\mathcal{X})$ is comprised of functions $f^s$ of the form \begin{align*}
&\mathcal{H}^s_{\kappa_Q}(\mathcal{X})=\{f^s:f^s(\cdot)=\mathrm{Tr}[\rhobf^s(\cdot)\Hbf'],\\
&\qquad  \Hbf'\in \mathbb C^{2^{|s|}\times 2^{|s|}},\Hbf'=(\Hbf')^{\dag}\}.    
\end{align*}
The definition of partial trace implies that, for any operator $\Hbf'\in \mathbb C^{2^{|s|}\times 2^{|s|}}$, $$
f^s(\xbf)=\mathrm{Tr}[\rhobf^{s}(\xbf)\Hbf'] = \mathrm{Tr}\big[\rhobf(\xbf)\,(\Hbf' \otimes \Ibf_{s^c})\big].
$$ 

Now note $\Hbf=\Hbf' \otimes \Ibf_{s^c}$ is an operator in $\mathbb C^{2^n \times 2^n}$. Hence any $f^s \ \in \mathcal{H}^{s}_{\kappa_Q}(\mathcal{X})$ must also belong to the full RKHS $\mathcal{H}_{\kappa_Q}(\mathcal{X})$. Therefore, we have $\mathcal{H}_{\kappa_Q}^{s} (\mathcal{X})\subseteq \mathcal{H}_{\kappa_Q}(\mathcal{X})$. 
Subsequently, since $\kappa_{\mathbb{S}_b} = \frac{1}{\lvert\mathbb{S}_b\rvert}\sum_{s\in \mathbb{S}_b} \kappa^s_Q $, we have $\mathcal{H}_{\kappa_{\mathbb{S}_b}}(\mathcal{X}) \subseteq\mathcal{H}_{\kappa_Q}(\mathcal{X})$.
\end{proof}

\subsection{Extension of Theorem 2 (i)} \label{c2}

We showed that the RKHS of a projected kernel is a subspace of the full kernel's RKHS. Here, we present a stronger result: this subspace relationship arises from a simple orthogonal projection in the underlying feature space:

\begin{proposition}

Suppose $\phibf_Q$ is the feature map corresponding to $\kappa_Q$ such that $\kappa_Q(\xbf,\xbf') = \langle\phibf_Q(\xbf),\phibf_Q(\xbf')\rangle$. Then for each $s \subseteq \{1,...,n\}$, there exists a unique self-adjoint idempotent operator P such that 
    \[ \kappa_Q^s(\xbf,\xbf')= \langle P\phibf_Q(\xbf), P \phibf_Q(\xbf')\rangle.\]
    This allows us to define a feature map of $\kappa_Q^s$ as $\phibf_{s}:=P\phibf_Q$.
\end{proposition}

\begin{proof}

By the representer theorem, there is a Hilbert space $\mathcal{F}_Q$ (feature space) and a map
$   \phibf_Q:\;\mathcal{X} \;\to\; \mathcal{F}_Q
$
(feature map) such that
\[
   \kappa_Q(\xbf,\xbf')
   \;=\;
   \langle \phibf_Q(\xbf),\,\phibf_Q(\xbf')\rangle_{\mathcal F_Q}.
\]
It is also standard that the RKHS $\mathcal{H}_{\kappa_Q}(\mathcal{X})$ is isometrically isomorphic to the closed linear span of $\{\phibf_Q(\xbf):\xbf\in\mathcal{X}\}$ in $\mathcal{F}_Q$. Specifically, one can define an isomorphism $
   \Phi:\;
   \overline{\mathrm{span}\{\phibf_Q(\xbf):\,\xbf\in\mathcal{X}\}}
   \;\rightarrow\;
   \mathcal{H}_{\kappa_Q}(\mathcal{X})$
by
$$
\Phi\left(\sum_i \alpha_i\, \phibf_Q(\xbf_i)\right) = \sum_i\alpha_i\,\kappa_Q(\cdot,\xbf_i).
$$
Because $\mathcal{H}_{\kappa_Q}^{s}(\mathcal{X})\subseteq\mathcal{H}_{\kappa_Q}(\mathcal{X})$ is a closed subspace, we define, with the above $\Phi$,
\[
   \mathcal{V}
   \;=\;
   \Phi^{-1}\bigl(\mathcal{H}_{\kappa_Q}^{s}(\mathcal{X})\bigr)
   \;\subseteq\;
   \overline{\mathrm{span}\{\phibf_Q(\xbf)\}}
   \;\subseteq\;
   \mathcal{F}_Q.
\]
Then $\mathcal{V}$ is indeed a closed subspace of $\mathcal{F}_Q$. By the Hilbert projection theorem \citep{Conway2007Functional}, there   exists a unique linear orthogonal projector $
  P:\;
  \mathcal{F}_Q \;\to\;\mathcal{F}_Q,$
such that $P^2 = P$, $P=P^\dagger$ and $\mathrm{Range}(P)=\mathcal{V}$.
We can then use this $P$ to define a new feature map
\[
   \tilde{\phibf}_Q^{s}(\xbf)
   \;=\;
   P\,\phibf_Q(\xbf),
\]
and the corresponding kernel,
\[
  \tilde{\kappa}_Q^{s}(\xbf,\xbf')
  \;=\;
  \langle \tilde\phibf_Q^{s}(\xbf),\,\tilde\phibf_Q^{s}(\xbf')\rangle_{\mathcal{F}_Q}
  \;=\;
  \langle P\,\phibf_Q(\xbf),\,P\,\phibf_Q(\xbf')\rangle_{\mathcal{F}_Q}.
\]

Since $P$ is an orthogonal projector, $\tilde{\kappa}_Q^{s}$ is positive semidefinite. Let $\tilde{\mathcal{H}}_{\kappa_Q}^{s}(\mathcal{X})$ be its associated RKHS.
We now show that $\tilde{\mathcal{H}}_{\kappa_Q}^{s}(\mathcal{X})=\mathcal{H}_{\kappa_Q}^{s}(\mathcal{X})$.

\textbf{(i) $\tilde{\mathcal{H}}_{\kappa_Q}^{s}(\mathcal{X})\subseteq\mathcal{H}_{\kappa_Q}^{s}(\mathcal{X})$}.

Any $f\in \tilde{\mathcal{H}}_{\kappa_Q}^{s}(\mathcal{X})$ has the form
\begin{align*}
f(\cdot)&=\sum_{i=1}^N \alpha_i\, \tilde{\kappa}_Q^{s}(\cdot, \xbf_i)\\
   &=
   \sum_{i=1}^N \alpha_i\,\bigl\langle \tilde\phibf_Q^{s}(\cdot),\,\tilde\phibf_Q^{s}(\xbf_i)\bigr\rangle_{\mathcal{F}_Q}\\
   &=
   \Bigl\langle \sum_{i=1}^N \alpha_i \,\tilde\phibf_Q^{s}(\xbf_i),\;\tilde\phibf_Q^{s}(\cdot)\Bigr\rangle_{\mathcal{F}_Q}.
\end{align*}

Define $
   \vbf \;:=\; \sum_{i=1}^N \alpha_i \,\tilde\phibf_Q^{s}(\xbf_i) \;\in\;\mathcal{V}.
$
Then 
\[
   f_\vbf(\cdot)\;:=\; \langle \vbf,\;\tilde\phibf_Q^{s}(\cdot)\rangle_{\mathcal{F}_Q} \;=\; \langle \vbf,\;P\phibf_Q(\cdot)\rangle_{\mathcal{F}_Q}.
\]

But since $\vbf\in\mathrm{Range}(P)$, $\vbf = P\,\vbf$, and since $P$ is self-adjoint, $$
f_\vbf(\cdot)=\langle \vbf,\;P\phibf_Q(\cdot)\rangle_{\mathcal{F}_Q} = \langle \vbf,\;\phibf_Q(\cdot)\rangle_{\mathcal{F}_Q}
$$

By the isomorphism property of $\Phi$, $f_\vbf$ is exactly $\Phi(\vbf) \in \mathcal H_{\kappa_Q}(\mathcal{X})$. Furthermore, since $\vbf\in\mathcal V$, we have $\Phi(\vbf) \in \mathcal H_{\kappa_Q}^{s}(\mathcal{X})$, hence $f \in \mathcal H_{\kappa_Q}^{s}(\mathcal{X})$.

\textbf{(ii) $\mathcal{H}_{\kappa_Q}^{s}(\mathcal{X})\subseteq\tilde{\mathcal{H}}_{\kappa_Q}^{s}(\mathcal{X})$}.

Similarly, each $f\in \mathcal{H}_{\kappa_Q}^{s}(\mathcal{X})$ can be written as $\Phi(\vbf)$ with some vector $\vbf\in\mathcal{V}$, i.e., 
\[
  f(\cdot) \;=\; \Phi(\vbf) \;=\; \langle \vbf,\;\phibf_Q(\cdot)\rangle_{\mathcal{F}_Q}.
\]

We have $\vbf=P\vbf$ since $\vbf\in\mathrm{Range}(P)$, thus 
\[
  f(\cdot)
  \;=\;
  \langle \vbf,\,\phibf_Q(\cdot)\rangle
  \;=\;
  \langle \vbf,\,P\,\phibf_Q(\cdot)\rangle
  \;=\;
  \langle \vbf,\,\tilde\phibf_Q^{s}(\cdot)\rangle,
\]
which, by definition, belongs to $\tilde{\mathcal{H}}_{\kappa_Q}^{s}(\mathcal{X})$.

The relationships \textbf{(i)} and \textbf{(ii)} together imply that
$\mathcal{H}_{\kappa_Q}^{s}(\mathcal{X})=\tilde{\mathcal{H}}_{\kappa_Q}^{s}(\mathcal{X})$. Finally, as RKHS is uniquely determined by its kernel, 
\[
   \kappa_Q^{s}(\xbf,\xbf') \;=\; \tilde{\kappa}_Q^{s}(\xbf,\xbf'),
   \quad
   \forall\,\xbf,\xbf'\in \mathcal{X}.
\]
This completes the proof that $\kappa_Q^{s}$ is realized exactly by the projected feature map $\tilde\phibf_Q^{s}= P\,\phibf_Q$.
\end{proof}

\subsection{Proof of Theorem \ref{thm:lpqk} (ii)}

This proof establishes a high-probability bound on the misspecification error, $\varepsilon$, when approximating a global function $f^*(\xbf) = \mathrm{Tr}[\Obf\rhobf(\xbf)]$ with one constructed from a ``low-weight'' or ``local'' observable $\Hbf$.

\textbf{From Function Approximation to Operator Distance.} 
To understand the misspecification error, it's helpful to start with a simple illustrative case. Suppose our approximate model is restricted to information from a single $b$-qubit subsystem, denoted by the index set $s$. The function $f(\xbf)$ in this model then depends only on the reduced density matrix $\rhobf^s(\xbf) = \tr_{s^c}(\rhobf(\xbf))$, i.e., $f(\xbf) = \tr(\Hbf\rhobf^s(\xbf))$, where $\Hbf$ is a $b$-qubit Hermitian observable. This can be rewritten on the full Hilbert space by letting $\Hbf' = \Hbf \otimes \Ibf_{s^c}$. Then $f(\xbf) = \tr(\Hbf'\rhobf(\xbf))$. 
The error of the best possible approximation for a given $\Obf$ is:
\begin{align*}
    \varepsilon_s(\Obf) &= \inf_f \|f-f^*\|_\infty \\
    &=\inf_{\Hbf'} \sup_{\rhobf(\xbf)} |\tr(\Obf\rhobf(\xbf)) - \tr(\Hbf'\rhobf(\xbf))| \\
    &=\inf_{\Hbf'} \sup_{\rhobf(\xbf)}|\tr[(\Obf-\Hbf')\rhobf(\xbf)]| \ \\
    &\leq \inf_{\Hbf'} \|\Obf - \Hbf'\|_{op}
\end{align*} 

\textbf{Generalization to Low-Weight Observables.} Now we generalize the previous setup, to approximate with a combination of subsystems. Let $\mathcal{H}_n = (\C^2)^{\otimes n}$ be the Hilbert space of an $n$-qubit system, with dimension $= 2^n$. Let $\mathcal{A}_n$ be the vector space of Hermitian operators on $\mathcal{H}_n$, with dimension $\dim(\mathcal{A}_n) = 4^n$. The Hilbert-Schmidt inner product on this space is $\langle \Abf, \Bbf \rangle = \tr(\Abf^\dagger \Bbf)$.

In a quantum system with $n$ qubits, any observable $\Obf \in \mathcal{A}_n$ can be uniquely decomposed in the Pauli basis: $\Obf = \sum_P c_P \Pbf$, where $\Pbf$ are $n$-qubit Pauli operators (i.e., P is a tensor product of $n$ single-qubit matrices. $\Pbf=\Pbf_1\otimes \Pbf_2\otimes\cdots\otimes \Pbf_n$, $\Pbf_i\in\{\Ibf,\Xbf,\Ybf,\Zbf\}$, each being a $2\times2$ matrix). The \textbf{Pauli weight} of a Pauli operator $\Pbf$, denoted $\mathrm{wt}(\Pbf)$, is the number of qubits on which it acts non-trivially (i.e., not as the identity) (recall if $\Hbf' = \Hbf \otimes \Ibf_{s^c}$ is from a b-qubit observable $\Hbf$, $\Hbf'$ acts as identity on qubits outside $b$).

Our goal is to approximate $f^*$ with (sum of) projected kernels. For a fully expressive ($4^n$-dimensional) full quantum kernel $\kappa_Q$, this is equivalent to approximating a global observable $\Obf$ using a combination of such ``local'' (low Pauli weight) operators. We define the subspace of low-weight operators, $\mathcal{A}_{\le b}$, as the span of all Pauli operators with weight at most $b$:
$$ \mathcal{A}_{\le b} = \mathrm{span}\{\Pbf \in \mathcal{P}_n \mid \mathrm{wt}(\Pbf) \le b \} $$
where $\mathcal{P}_n$ is the set of $n$-qubit Pauli operators.

Now the error of the best possible approximation of $\Obf$ using an operator $\Hbf' \in \mathcal{A}_{\le b}$ is given by
$$ \varepsilon_{\le b}(\Obf) = \inf_{\Hbf' \in \mathcal{A}_{\le b}} \|\Obf - \Hbf'\|_{op} $$

Further, the operator norm is upper-bounded by the Hilbert-Schmidt norm, $\|\Abf\|_{op} \le \|\Abf\|_2$. Therefore, we can find an upper bound by analyzing the Hilbert-Schmidt distance:
\begin{equation}
\varepsilon_{\le b}(\Obf) \le \inf_{\Hbf' \in \mathcal{A}_{\le b}} \|\Obf - \Hbf'\|_{2}
\end{equation}

The problem is now to find the minimum of $\|\Obf - \Hbf'\|_{2}$ over all $\Hbf' \in \mathcal{A}_{\le b}$. In a Hilbert space, the point in a subspace that is closest to an external point is its orthogonal projection.

Let $P_{\le b}: \mathcal{A}_n \to \mathcal{A}_{\le b}$ be the orthogonal projection operator. The infimum is achieved when $\Hbf' = P_{\le b}(\Obf)$.
\begin{equation}
\inf_{\Hbf' \in \mathcal{A}_{\le b}} \|\Obf - \Hbf'\|_{2} = \|\Obf - P_{\le b}(\Obf)\|_{2}
\end{equation}
By the property of orthogonal projections, we have:
\begin{equation}
\|\Obf - P_{\le b}(\Obf)\|_2^2 = \|\Obf\|_2^2 - \|P_{\le b}(\Obf)\|_2^2
\end{equation}

\textbf{Expected Error for Spherical Operators.} Here we assume $\Obf$ is drawn uniformly from the sphere of operators with $\|\Obf\|_2 = 1$ (For simplicity in this step, we'll assume the norm is one, i.e., $||\Obf||_2=1$. This analysis can be trivially extended to any arbitrary norm $||\Obf||_2=B$ by scaling the final result by $B$). Then $\|\Obf - P_{\le b}(\Obf)\|_2^2 = \|\Obf\|_2^2 - \|P_{\le b}(\Obf)\|_2^2 = 1 - \|P_{\le b}(\Obf)\|_2^2$. Taking the expectation:
$$ \E\left[ \hsnorm{\Obf - P_{\le b}(\Obf)}^2 \right] = 1 - \E\left[ \hsnorm{P_{\le b}(\Obf)}^2 \right] $$ 

We use the following result for random projections.

\begin{lemma} Let $\Obf$ be a random vector chosen uniformly from the unit sphere in a $D$-dimensional space $V$. Let $P_S$ be the orthogonal projection onto a $d$-dimensional subspace $S$. Then, $\E[\hsnorm{P_S(\Obf)}^2] = d/D$. \end{lemma} 

\begin{proof} Let $\{\ebf_1, \dots, \ebf_D\}$ be an orthonormal basis for $V$ such that $\{\ebf_1, \dots, \ebf_d\}$ is an orthonormal basis for $S$. Any vector $\Obf$ on the unit sphere can be written as $\Obf = \sum_{i=1}^D c_i \ebf_i$ with $\sum_i c_i^2 = 1$. The projection is $P_S(\Obf) = \sum_{i=1}^d c_i \ebf_i$, and its squared norm is $\|P_S(\Obf)\|_2^2 = \sum_{i=1}^d c_i^2$. By linearity of expectation, $\E\left[\|P_S(\Obf)\|_2^2\right] = \sum_{i=1}^d \E[c_i^2]$. Because the distribution of $\Obf$ is spherically symmetric, the expected value of its squared component along any basis vector is the same, i.e., $\E[c_i^2] = C$ for all $i$ for some constant $C$. We know $\E[\|\Obf\|_2^2] = \E[1] = 1$. Also, $\E[\|\Obf\|_2^2] = \E[\sum_{i=1}^D c_i^2] = \sum_{i=1}^D \E[c_i^2] = D \cdot C$. Thus, $D \cdot C = 1 \implies C = 1/D$. Finally, $\E\left[\|P_S(\Obf)\|_2^2\right] = \sum_{i=1}^d C = d \cdot C = d/D$. \end{proof}

\textbf{High-Probability Bound for Spherical Operators.}
Next we derive a high-probability bound on the error by determining the probability distribution of the squared norm of the projection, $\hsnorm{P_{\le b}(\Obf)}^2$.

\begin{lemma}
Let $\Obf$ be a random vector chosen uniformly from the unit sphere in a $D$-dimensional space $V$. Let $P_S$ be the orthogonal projection onto a $d$-dimensional subspace $S$. The random variable $X = \hsnorm{P_S(\Obf)}^2$ follows a Beta distribution with parameters $\alpha = d/2$ and $\beta = (D-d)/2$.
$$ \hsnorm{P_S(\Obf)}^2 \sim \mathrm{Beta}\left(\frac{d}{2}, \frac{D-d}{2}\right) $$
\end{lemma}

\begin{proof}[Proof of Lemma]
A random vector $\Obf$ on the unit sphere can be generated by taking a standard $D$-dimensional Gaussian vector $\Gbf = (g_1, \dots, g_D)$ where each $g_i \sim \mathcal{N}(0,1)$ are i.i.d., and normalizing it: $\Obf = \Gbf / \hsnorm{\Gbf}$.
Let $\{\ebf_1, \dots, \ebf_D\}$ be an orthonormal basis for $V$ where $\{\ebf_1, \dots, \ebf_d\}$ is a basis for $S$. The coefficients of $\Obf$ are $c_i = \langle \Obf, \ebf_i \rangle = g_i / \hsnorm{\Gbf}$.
The squared norm of the projection is:
$$ \hsnorm{P_S(\Obf)}^2 = \sum_{i=1}^{d} c_i^2 = \frac{\sum_{i=1}^{d} g_i^2}{\sum_{j=1}^{D} g_j^2} $$
Let $U = \sum_{i=1}^{d} g_i^2$ and $V = \sum_{j=d+1}^{D} g_j^2$. By definition of the chi-squared distribution, $U \sim \chi^2(d)$ and $V \sim \chi^2(D-d)$. Since the $g_i$ are independent, $U$ and $V$ are independent.
The random variable is $X = U / (U+V)$. A random variable constructed as the ratio of a chi-squared variable to the sum of itself and another independent chi-squared variable is known to follow the Beta distribution. Specifically, if $U \sim \chi^2(\nu_1)$ and $V \sim \chi^2(\nu_2)$, then $U/(U+V) \sim \mathrm{Beta}(\nu_1/2, \nu_2/2)$.
Substituting $\nu_1=d$ and $\nu_2=D-d$ proves the lemma.
\end{proof}

The Beta distribution is known to be strongly concentrated around its mean. In fact, it is $\frac{1}{4(\alpha+\beta)}$-sub-gaussian \citep{beta_sub_gauss}. Hence we can use Hoeffding's inequality to obtain a concentration inequality: For a random variable $X \sim \mathrm{Beta}(\alpha, \beta)$ with mean $\mu = \alpha/(\alpha+\beta)$, a tail bound is given by:
$$ \Prob( |X - \mu| \ge t ) \le 2\exp\left(-\frac{t^2}{2\sigma^2}\right)\le 2e^{-2(\alpha+\beta)t^2} $$
In our case, $\mu = \frac{d/2}{D/2} = d/D$ and $\alpha+\beta = D/2$. We are interested in a lower bound on $X$ to get an upper bound on the error.
$$ \Prob\left( X \le \frac{d}{D} - t \right) \le e^{-Dt^2} $$
Setting the failure probability to $\delta = e^{-Dt^2}$ gives $t = \sqrt{\ln(1/\delta)/D}$. This means with probability at least $1-\delta$, we have $X \ge d/D - \sqrt{\ln(1/\delta)/D}$.

\begin{corollary}[Spherical Error Bound]
With probability at least $1-\delta$, the squared Hilbert-Schmidt norm of the error is bounded by:
$$ \hsnorm{\Obf - P_{\le b}(\Obf)}^2 \le \left(1 - \frac{\sum_{w=0}^{b} \binom{n}{w} 3^w}{4^n}\right) + \sqrt{\frac{\ln(1/\delta)}{4^n}}. $$

\end{corollary}
\begin{proof}
From the Pythagorean identity, $\hsnorm{\Obf - P_{\le b}(\Obf)}^2 = 1 - \hsnorm{P_{\le b}(\Obf)}^2$. Let $d=\text{dim}(\mathcal{A}_{\leq b})$, using the high-probability lower bound on $\hsnorm{P_{\le b}(\Obf)}^2$ gives:
\begin{align*}
    \hsnorm{\Obf - P_{\le b}(\Obf)}^2 &\le 1 - \left( \frac{d}{D} - \sqrt{\frac{\ln(1/\delta)}{D}} \right) \\
    &= \left(1 - \frac{d}{D}\right) + \sqrt{\frac{\ln(1/\delta)}{D}}
\end{align*}

Next we show that $D=4^n$ and $d = \sum_{w=0}^{b} \binom{n}{w} 3^w$ in our case.

Recall any operator $\Obf \in \mathcal{A}_n$ can be uniquely decomposed in the Pauli basis: $\Obf = \sum_\Pbf c_P \Pbf$, where $\Pbf$ are $n$-qubit Pauli operators. And it follows directly from our assumption that $c_\Pbf\neq0$ for all $c_\Pbf$ with probability 1. The $n$-qubit Pauli operators form a basis of $\mathcal A_n$, hence $D=4^n$.

$d$ is the dimension of the subspace $\mathcal{A}_{\le b}$, the number of linearly independent Pauli operators with weight at most $b$.
Note the following:
\begin{itemize}
    \item The number of ways to choose $w$ qubits to act on from a total of $n$ is $\binom{n}{w}$.
    \item For each of the chosen $w$ qubits, there are 3 non-identity Pauli choices ($\sigma_x, \sigma_y, \sigma_z$).
    \item The total number of Pauli strings with weight exactly $w$ is $\binom{n}{w}3^w$.
\end{itemize}
The dimension $d$ is the sum over all allowed weights (from 0 to $b$):
$$ d = \dim(\mathcal{A}_{\le b}) = \sum_{w=0}^{b} \binom{n}{w} 3^w $$

Substitute the expression for $d$ and $D$ completes the proof.

\end{proof}

Now it only remains for us to show that $$
p :=1 - \frac{\sum_{w=0}^{b} \binom{n}{w} 3^w}{4^n} \leq e^{ -  \frac{8n}{3} \Big(\frac{b+1}{n}- \frac{3}{4}\Big)^2}.
$$
For this purpose we note that 
\begin{align*}
   p 
   &=\frac{1}{4^n} \sum_{w=b+1}^{n}\binom{n}{w} 3^w \\
   &= \sum_{w=b+1}^{n}\binom{n}{w} \Big(\frac{3}{4}\Big)^w \Big(\frac{1}{4}\Big)^{n-w} \\
   &=\Pr\Big( \frac{1}{n}\sum_{i=1}^n Z_i \ge\frac{b+1}{n}\Big),
\end{align*}
where $Z_i$ are independent Bernoulli random variables satisfying $\Pr(Z_i = 1) =\frac{ 3}{4}$. If $\frac{b+1}{n}\ge\frac{3}{4}$ which is the regime of interest, by the Chernoff bound,
\begin{align*}
    p \le \exp\bigg( -n D_{\mathrm{KL}}\Big( \frac{b+1}{n} \Big\|\frac{3}{4} \Big) \bigg),
\end{align*}
where $D_{\mathrm{KL}}(q\|p)$ denotes the KL divergence between Bernoulli distributions with parameters $q$ and $p$. Furthermore, it holds that $D_{\mathrm{KL}}(q +x \| q)\ge\frac{ x^2}{2q(1-q)}$. Thus, 
\begin{align*}
      D_{\mathrm{KL}}\Big( \frac{b+1}{n} \Big\|\frac{3}{4} \Big) \ge \frac{ (\frac{b+1}{n}- \frac{3}{4})^2 }{2\cdot \frac{3}{4}\cdot \frac{1}{4}}  = \frac{8}{3} \Big(\frac{b+1}{n}- \frac{3}{4}\Big)^2.
\end{align*}
We conclude that 
\begin{align*}
    p\le \exp \bigg( -  \frac{8n}{3} \Big(\frac{b+1}{n}- \frac{3}{4}\Big)^2  \bigg).
\end{align*}


    

\subsection{Justification of the Random Operator Assumption}

The proof of Theorem~\ref{thm:lpqk}(ii) relies on the assumption that the global observable $\Obf$ is drawn uniformly at random from the sphere of operators with a fixed norm. 
While this may seem abstract and restrictive, it can be motivated from an isotropic Gaussian prior on the observable $\Obf$ (which induces a GP prior on $f^*(\xbf)=\mathrm{Tr}[\Obf\rhobf(\xbf)]$). This assumption corresponds to conditioning that Gaussian prior on $\|\Obf\|_2$ (equivalently, normalizing a Gaussian draw).

Here's the connection:

\textbf{From Gaussian Processes to Operator Priors.} A function $f^*$ drawn from a GP is defined by a prior distribution over a function space. In our setting, where functions are of the form $f^*(\xbf) = \mathrm{Tr}[\Obf\rhobf(\xbf)]$, this prior on functions induces a prior on the operator $\Obf$. We can represent any observable $\Obf$ by expanding it in an orthonormal basis of Hermitian operators, $\{\Bbf_i\}$ (e.g., the normalized Pauli basis):
    $$\Obf = \sum_{i=1}^{4^n} c_i \Bbf_i$$
    A simple and common case of GP prior is one where the coefficients $c_i$ are assumed to be independent and identically distributed (i.i.d.) Gaussian random variables, i.e., $c_i \sim \mathcal{N}(0, \sigma^2)$. This is equivalent to placing a white-noise Gaussian prior on the operator $\Obf$.

\textbf{From Gaussian Vectors to Uniform Spherical Vectors.} The vector of coefficients, $\mathbf{c} = (c_1, c_2, \dots, c_{4^n})$, is therefore a vector whose components are i.i.d. Gaussians. A fundamental property of multivariate Gaussian distributions is that, taking such a vector and normalize it by its length, the resulting direction vector is uniformly distributed on the surface of the unit sphere. So, if we normalize our coefficient vector to get $\mathbf{c}' = \mathbf{c} / \|\mathbf{c}\|_2$, the vector $\mathbf{c}'$ will be uniformly random on the unit sphere in $\mathbb{R}^{4^n}$.
    
Hence, the operator constructed from these normalized coefficients, $\Obf' = \sum_i c'_i \Bbf_i$, will have a unit Hilbert-Schmidt norm ($\|\Obf'\|_2 = \|\mathbf{c}'\|_2 = 1$) and its ``direction'' in the space of operators will be uniformly random. This is precisely the assumption made in our proof.


\subsection{Structured Observables: Locality, Pauli-Weight Tails, and Regret}
\label{app:lpqk_structured_observables}

The bound in Theorem~\ref{thm:lpqk}(ii) provides a useful worst-case baseline, but it is pessimistic for many physical tasks. In many quantum optimization problems, the observable $\Obf$ has additional structure, such as locality or a rapidly decaying high-weight Pauli tail. In this subsection, we show how such structure directly improves the LPQK misspecification error and the resulting regret bound.

Recall in a quantum system with $n$ qubits, any observable $\Obf \in \mathcal{A}_n$ can be uniquely decomposed in the normalized Pauli basis: $\Obf = \sum_P c_P \Pbf$, where $\Pbf$ are $n$-qubit Pauli operators. And the \textbf{Pauli weight} of a Pauli operator $\Pbf$, denoted $\mathrm{wt}(\Pbf)$, is the number of qubits on which it acts non-trivially (i.e., not as the identity).

Define the low-weight and high-weight components
\[
    \Obf_{\le b}
    :=
    \sum_{\mathrm{wt}(\Pbf)\le b} c_{\Pbf}\Pbf,
    \qquad
    \Obf_{> b}
    :=
    \sum_{\mathrm{wt}(\Pbf)> b} c_{\Pbf}\Pbf.
\]
The summed LPQK over all subsystems $|s|\le b$ retains exactly the observable components of Pauli weight at most $b$. Therefore, for
    $f^*(\xbf)=\mathrm{Tr}[\Obf\rhobf(\xbf)]$,
the best LPQK approximation satisfies
\begin{equation} \label{eq:structured_lpqk_error}
    \varepsilon_b
    :=
    \inf_{f\in\mathcal{H}_{\kappa_{\mathbb{S}_b}}(\mathcal{X})}
    \|f-f^*\|_{\infty} 
    \le
    \sup_{\xbf\in\mathcal{X}}
    \left|
    \mathrm{Tr}[\Obf_{>b}\rhobf(\xbf)]
    \right| 
    \le
    \|\Obf_{>b}\|_{\rm op}
    \le
    \|\Obf_{>b}\|_{\rm HS}
    =
    \left(
    \sum_{\mathrm{wt}(\Pbf)>b} c_{\Pbf}^2
    \right)^{1/2}.
\end{equation}
That is, the LPQK misspecification error is controlled by the high-weight Pauli tail of the target observable.

On the other hand, the effective dimension of the summed LPQK is
    $S_b
    :=
    \sum_{w=0}^{b}\binom{n}{w}3^w$,
and hence its information gain satisfies
\[
    \gamma_T(b)
    =
    \mathcal{O}(S_b\log T).
\]
Substituting this into the EC-GP-UCB bound in Eq.~\eqref{eq:regret_ECGPUCB} gives, suppressing logarithmic factors and constants,
\begin{align}
    R_T(b)
    =
    \widetilde{\mathcal{O}}
    \left(
    (B\sqrt{S_b}+S_b)\sqrt{T}
    +
    \varepsilon_b T\sqrt{S_b}
    \right).
    \label{eq:structured_lpqk_regret}
\end{align}

\paragraph{Example 1: $k$-local observables.}
Suppose $\Obf$ is \emph{$k$-local}, meaning that its Pauli expansion contains no terms of weight larger than $k$:
\[
    c_{\Pbf}=0
    \qquad
    \text{for all } \mathrm{wt}(\Pbf)>k.
\]
Then $\Obf_{>b}=0$ for all $b\ge k$, and therefore
\[
    \varepsilon_b=0
    \qquad
    \text{for all } b\ge k.
\]
Thus, choosing $b=k$ makes the LPQK surrogate realizable while using only
\[
    S_k
    =
    \sum_{w=0}^{k}\binom{n}{w}3^w
    =
    \mathcal{O}(n^k)
\]
features for fixed $k$, instead of the full $4^n$-dimensional quantum feature space. The regret bound becomes
\[
    R_T(k)
    =
    \widetilde{\mathcal{O}}
    \left(
    (B\sqrt{S_k}+S_k)\sqrt{T}
    \right)
    =
    \widetilde{\mathcal{O}}
    \left(
    n^k\sqrt{T}
    \right)
\]
for fixed $k$ under the EC-GP-UCB bound. Therefore, for local physical observables, LPQK can replace the exponential full-kernel dimension $4^n$ by the polynomial dimension $S_k=\mathcal{O}(n^k)$.

\paragraph{Example 2: exponentially decaying high-weight Pauli tail.}
More generally, suppose the total squared Pauli mass at weight $w$ decays exponentially:
\[
    E_w
    :=
    \sum_{\mathrm{wt}(\Pbf)=w} c_{\Pbf}^2
    \le
    C_0 e^{-2\alpha w}
    \qquad
    \text{for some } C_0,\alpha>0.
\]
Then
\begin{align}
    \varepsilon_b
    \le
    \|\Obf_{>b}\|_{\rm HS}
    =
    \left(
    \sum_{w=b+1}^{n} E_w
    \right)^{1/2} 
    \le
    \left(
    C_0\sum_{w=b+1}^{\infty}e^{-2\alpha w}
    \right)^{1/2}
    =
    C_{\alpha}e^{-\alpha(b+1)},
    \label{eq:exp_tail_lpqk}
\end{align}
where $C_{\alpha}:=\sqrt{C_0/(1-e^{-2\alpha})}$. Hence the misspecification error decreases exponentially in the retained subsystem size $b$, while the information gain increases through
\[
    \gamma_T(b)
    =
    \mathcal{O}(S_b\log T),
    \qquad
    S_b=\sum_{w=0}^{b}\binom{n}{w}3^w.
\]
Combining \eqref{eq:structured_lpqk_regret} and \eqref{eq:exp_tail_lpqk} gives
\[
    R_T(b)
    =
    \widetilde{\mathcal{O}}
    \left(
    (B\sqrt{S_b}+S_b)\sqrt{T}
    +
    C_{\alpha}e^{-\alpha b}T\sqrt{S_b}
    \right).
\]
Thus, an optimal subsystem size can be selected by
\[
    b^{\star}
    \in
    \underset{0\le b\le n}{\arg\min}
    \left\{
    (B\sqrt{S_b}+S_b)\sqrt{T}
    +
    C_{\alpha}e^{-\alpha b}T\sqrt{S_b}
    \right\}.
\]
Since for small fixed $b$, $S_b=\mathcal{O}(n^b)$, balancing the two terms gives the heuristic scaling
    $b^{\star}
    =
    \widetilde{\mathcal{O}}\left(\frac{\log T}{2\alpha+\log n}\right)$, and eventually yielding
\[
    R_T
    =
    \widetilde{\mathcal O}
    \left(
    T^{\frac12+\frac{\log n}{2\alpha+\log n}}
    \right).
\]
Thus, faster Pauli-tail decay, i.e., larger $\alpha$, permits a smaller subsystem size and regret closer to the realizable $\sqrt{T}$-type rate.

Overall, LPQK is more beneficial when the target observable has most of its Pauli mass concentrated on low-weight terms. In the exact $k$-local case, the misspecification error vanishes at $b=k$ and the exponential full-kernel dimension $4^n$ is replaced by the polynomial dimension $\mathcal{O}(n^k)$. For observables with decaying high-weight tails, the optimal $b$ is determined by balancing the information-gain term $S_b$ against the tail error $\varepsilon_b$.

\paragraph{Remark on Effective Dimensionality}
It is important to note that the dimensions used in proof of Theorem \ref{thm:lpqk}(ii), $D=4^n$ and $d = \sum_{w=0}^{b} \binom{n}{w} 3^w$, are based on the assumption that the quantum feature map is sufficiently expressive to span the entire space of $n$-qubit Hermitian operators. While this is often the case for complex, generic circuits, it may not hold for all feature maps.

In many practical scenarios, a quantum circuit may have specific symmetries or a limited number of parameters, constraining the reachable states $\{\rhobf(\xbf)\}$ to a lower-dimensional subspace. In such cases, the actual dimension of the full space $D$ would be less than $4^n$, and the actual dimension of the low-weight subspace, $d$, would likewise be smaller than $ \sum_{w=0}^{b} \binom{n}{w} 3^w$.

\section{Fourier Form of Quantum Kernels and RFF} \label{app:rff}

\subsection{Fourier Form of Quantum Kernels}

\begin{theorem}[Fourier Representation of Quantum Kernels, Theorem 1 of \citet{schuld2021quantumkernel}]
\label{thm:Fourier_rep}
Let \(\mathcal{X} = \mathbb{R}^N\) and let
\[
\Sbf(\xbf) 
\;=\; 
\Wbf^{(N+1)}\,e^{-ix_N\Gbf}\,\Wbf^{(N)}\,\cdots\,\Wbf^{(2)}\,e^{-ix_1\Gbf}\,\Wbf^{(1)}
\]
be a data-encoding circuit, where each \(\Gbf\) is a diagonal
(\(d \times d\)) Hermitian operator with eigenvalues
\(\{\lambda_1,\dots,\lambda_d\}\), and \(\Wbf^{(1)}, \ldots, \Wbf^{(N+1)}\) denote fixed unitary operations.
Define the quantum state \(\ket{\phibf(\mathbf x)} = \Sbf(\mathbf x)\ket{0}\) and let
\(\kappa_Q(\mathbf x,\mathbf x')\) be the resulting quantum kernel, i.e.,
\[
\kappa_Q(\mathbf x,\mathbf x')
\;=\;
\bigl|
 \langle \phibf(\mathbf x)\vert \phibf(\mathbf x')\rangle
\bigr|^2.
\]
Then there exists a finite index set \(\Omega\subseteq\mathbb{R}^N\) (derived
from the eigenvalues \(\{\lambda_j\}\)) and coefficients
\(\{c_{\mathbf s,\mathbf t}\}_{\mathbf s,\mathbf t \in \Omega}\) such that
\begin{equation} \label{eqn:ffofqkernel}
    \kappa_Q(\mathbf x,\mathbf x')
\;=\;
\sum_{\mathbf s,\mathbf t \in \Omega}c_{\mathbf s,\mathbf t}
\exp\bigl(i\,\mathbf s \cdot \mathbf x\bigr)\,\exp\bigl(i\,\mathbf t \cdot \mathbf x'\bigr)
\end{equation}

\end{theorem}

\begin{proof}
Assume without loss of generality that the encoding generator \(\Gbf\) is diagonal
because any Hermitian operator can be diagonalized as \(\Gbf = \Vbf \Sigmabf \Vbf^\dagger\), where 
\(\Sigmabf\) is diagonal. Consequently, an encoding gate \(e^{-i x_i \Gbf}\) can be written as
\[
e^{-i x_i \Gbf} 
\;=\; 
\Vbf\,\bigl(e^{-i x_i \Sigmabf}\bigr)\,\Vbf^\dagger,
\]
and thus we may ``absorb'' \(\Vbf\) and \(\Vbf^\dagger\) into the preceding and subsequent 
arbitrary unitaries in the circuit. Hence, we have 
\[
e^{-i x_i \Sigmabf}
\;=\;
\begin{pmatrix}
e^{-i x_i \lambda_1} & \cdots & 0 \\
\vdots               & \ddots & \vdots \\
0                    & \cdots & e^{-i x_i \lambda_d}
\end{pmatrix},
\]
where \(\{\lambda_1, \dots, \lambda_d\}\) are the eigenvalues of \(\Gbf\).

\medskip

The quantum kernel can be written as
\begin{align*}
\kappa_Q(\xbf,\xbf')
&= \bigl|
 \langle \phibf(\mathbf x)\vert \phibf(\mathbf x')\rangle
\bigr|^2
\\
&=
\lvert
\bra{0}\,\bigl(\Wbf^{(1)}\bigr)^\dagger 
\bigl(e^{-\,i\,x'_1\,\Sigmabf}\bigr)^\dagger
\cdots
\bigl(e^{-\,i\,x'_N\,\Sigmabf}\bigr)^\dagger
\bigl(\Wbf^{(N+1)}\bigr)^\dagger 
\times \Wbf^{(N+1)}\,e^{-\,i\,x_N\,\Sigmabf}
\cdots
e^{-\,i\,x_1\,\Sigmabf}\,\Wbf^{(1)}\,\ket{0}
\rvert^2
\\
&=
\lvert
\bra{0}\,\bigl(\Wbf^{(1)}\bigr)^\dagger 
\bigl(e^{-\,i\,x'_1\,\Sigmabf}\bigr)^\dagger
\cdots
\bigl(e^{-\,i\,x'_N\,\Sigmabf}\bigr)^\dagger 
\times e^{-\,i\,x_N\,\Sigmabf}\cdots e^{-\,i\,x_1\,\Sigmabf}\,\Wbf^{(1)}\,\ket{0}
\rvert^2
\\
&=
\big|\sum_{j_1,\dots,j_N=1}^d
\sum_{k_1,\dots,k_N=1}^d \exp\Bigl(
  -\,i\bigl(\lambda_{j_1} x_1 - \lambda_{k_1} x'_1 + \cdots + \lambda_{j_N} x_N - \lambda_{k_N} x'_N\bigr)
\Bigr) 
\nonumber\\
&\quad\;\;\times
\bigl(\Wbf^{(1)}_{1\,k_1}\cdots \Wbf^{(N)}_{k_{N-1}\,k_N}\bigr)^{*} 
\times\bigl(\Wbf^{(N)}_{j_N\,j_{N-1}}\cdots \Wbf^{(1)}_{j_1\,1}\bigr)\big|^2
\\
&=
\left\lvert\sum_{\mathbf j}\sum_{\mathbf k}
\exp\bigl(-\,i\,(\Lambdabf_{\mathbf j}\cdot \xbf - \Lambdabf_{\mathbf k}\cdot \xbf')\bigr)
\,\bigl(w_{\mathbf{k}}\bigr)^{*}\,\bigl(w_{\mathbf{j}}\bigr)\right\rvert^2
\\
&=
\sum_{\mathbf{j}}\sum_{\mathbf k}\sum_{\mathbf h}\sum_{\mathbf l}
\exp\bigl(-\,i\,(\Lambdabf_{\mathbf j} - \Lambdabf_{\mathbf l})\cdot \xbf \bigr)\,
\exp\bigl(i\,(\Lambdabf_{\mathbf k} - \Lambdabf_{\mathbf h})\cdot \xbf '\bigr)
\times\bigl(w_{\mathbf{k}}\,w_{\mathbf{h}}\bigr)^{*}\,\bigl(w_{\mathbf{j}}\,w_{\mathbf{l}}\bigr).
\end{align*}

Here, the scalars \(W_{ab}^{(i)}\) for \(i = 1, \ldots, N\) refer to the element 
$\bra{a} \Wbf^{(i)} \ket{b}$ of the unitary operator \(\Wbf^{(i)}\). 
The bold multi-index \(\mathbf{j}\) summarizes the set \((j_1, \ldots, j_N)\) 
where \(j_i \in \{1, \ldots, d\}\), and \(\Lambdabf_{\mathbf{j}}\) is a vector 
containing the eigenvalues selected by the multi-index (and similarly for 
\(\mathbf{k}, \mathbf{h}, \mathbf{l}\)).

We can now group together all terms where 
\(\Lambdabf_{\mathbf{j}} - \Lambdabf_{\mathbf{l}} = \mathbf{s}\) 
and \(\Lambdabf_{\mathbf{k}} - \Lambdabf_{\mathbf{h}} = \mathbf{t}\); 
in other words, where the differences of eigenvalues amount to the same vectors 
\(\mathbf{s}, \mathbf{t}\). Then we have

\begin{align*}
\kappa_Q(\xbf,\xbf') 
&= \sum_{\mathbf{s},\mathbf{t} \in \Omega} e^{-i \mathbf{s}\cdot \xbf} \, e^{i \mathbf{t} \cdot\xbf'}\,
  \Bigg(\sum_{\substack{ \mathbf{j}, \mathbf{l}:\\ \Lambdabf_{\mathbf{j}} - \Lambdabf_{\mathbf{l}} = \mathbf{s}}}\,
  \sum_{\substack{\mathbf{k}, \mathbf{h}:\\ \Lambdabf_{\mathbf{k}} - \Lambdabf_{\mathbf{h}} = \mathbf{t}}}
  w_{\mathbf{j}}  w_{\mathbf{l}}  \bigl(w_{\mathbf{k}}  w_{\mathbf{h}}\bigr)^{*}\Bigg) \\
&= \sum_{\mathbf{s},\mathbf{t} \in \Omega} e^{-i \mathbf{s} \cdot \xbf} \, e^{i \mathbf{t} \cdot \xbf'} \, c_{\mathbf{s},\mathbf{t}}
\end{align*}

The frequency set \(\Omega\) contains all vectors of the form 
\(\{\Lambdabf_{\mathbf{j}} - \Lambdabf_{\mathbf{k}}\}\) with 
\(\Lambdabf_{\mathbf{j}} = (\lambda_{j_1}, \ldots, \lambda_{j_N})\), 
\(j_1, \ldots, j_N \in \{1, \ldots, d\}\).
\end{proof}

This result demonstrates that a \emph{quantum} kernel can be interpreted as a sum of complex exponentials, which are determined by the spectra of the encoding generator $\Gbf$ and the unitaries interleaved within the circuit. Although, in principle, this sum may encompass \emph{exponentially many} terms---stemming from the combinatorial complexity of multi-qubit gates---the specific structure of $\Gbf$ and the unitaries \(\Wbf\) often introduce patterns or symmetries. These patterns can be leveraged to simplify analysis. More importantly, once the circuit structure is fixed, the coefficients in this sum can be computed explicitly using a formula, as outlined earlier in the proof.

\paragraph{Remark (extension to re-uploading / repeated encodings).}
Theorem~\ref{thm:Fourier_rep} is stated for an encoding sequence in which each scalar input $x_j$ appears once.
If a circuit re-uploads the same input $\mathbf{x}\in\mathbb{R}^d$ across multiple layers (so that some coordinates are used
multiple times), the resulting kernel still admits a Fourier representation of the above form \citep{schuld2008effect}.


Next, we will examine how this perspective facilitates the analysis of the feature map corresponding to quantum kernels.

\subsection{RFF with Pauli encoding Quantum Kernel}
\label{app:rff1}

In this section, we consider the special case where the quantum kernel
$\kappa_Q(\xbf,\xbf')$ satisfies two simplifying properties:
\begin{itemize}
    \item The kernel is translation-invariant---i.e., it depends only on the difference $\xbf - \xbf'$.
    \item The frequencies $\Omega$ in the Fourier representation
(cf. Equation~\ref{eqn:ffofqkernel}) are integer-valued.
\end{itemize}
These assumptions hold in many common settings, e.g., when data encoding is implemented by Pauli rotations
\citep{landman2022classicapprox}. In this case, one may write
\[
\kappa_Q(\xbf,\xbf')
=
\sum_{\boldsymbol{\omega}\in\Omega} c_{\boldsymbol{\omega}}
\exp\bigl(i\,\boldsymbol{\omega}^{\top}(\xbf-\xbf')\bigr)
=
\sum_{\boldsymbol{\omega}\in\Omega} c_{\boldsymbol{\omega}}
\cos\bigl(\boldsymbol{\omega}^{\top}(\xbf-\xbf')\bigr),
\]
where $\Omega\subset \mathbb{Z}^d$ is a finite set of frequencies and $c_{\boldsymbol{\omega}}\ge 0$ are real Fourier coefficients.

\paragraph{Expressing $\kappa$ in a Real Feature Space}
From standard trigonometric identities, each cosine term $\cos(\boldsymbol{\omega}^{\top}\xbf - \boldsymbol{\omega}^{\top}\xbf')$ can be represented as a dot product of some 2 dimensional real feature. 
Specifically,
\begin{align*}
C\cos(\alpha - \beta)
&= C\cos\alpha\cos\beta+C\sin\alpha\sin\beta \\
&=\left\langle \begin{bmatrix}
        \sqrt{C}\cos\alpha \\
        \sqrt{C}\sin\alpha
    \end{bmatrix}, \begin{bmatrix}
        \sqrt{C}\cos\beta \\
        \sqrt{C}\sin\beta
    \end{bmatrix} \right\rangle
\end{align*}
For example, let $\Delta_i=x_i-x_i'$,
\begin{align*}
    &\cos(\Delta_1 - \Delta_2 +\Delta_3) \\
    &= \cos\left((x_1-x_1') - (x_2-x_2') + (x_3-x_3')\right) \\
    &= \cos\left((x_1-x_2+x_3) - (x_1'-x_2'+x_3')\right) \\
    &= \left\langle \begin{bmatrix}
        \cos(x_1-x_2+x_3) \\
        \sin(x_1-x_2+x_3)
    \end{bmatrix}, \begin{bmatrix}
        \cos(x_1'-x_2'+x_3')\\
        \sin(x_1'-x_2'+x_3')
    \end{bmatrix} \right\rangle.
\end{align*}

Hence, if $\kappa_Q(\xbf,\xbf') 
= \sum_{\boldsymbol{\omega}\in \Omega}
c_{\boldsymbol{\omega}}\cos\bigl(\boldsymbol{\omega}^{\top}(\xbf - \xbf')\bigr)$, then we can rewrite \(\kappa_Q\) as \(\langle \phibf(\xbf), \phibf(\xbf') \rangle\) for some (potentially
high-dimensional) $\phibf(\xbf)$ built from the
features. Consequently, for a translation-invariant quantum kernel $\kappa_Q$ with a finite
Fourier decomposition, \(\kappa_Q\) corresponds to the real feature
set
\[
\Bigl\{
  \sqrt{c_{\boldsymbol{\omega}}}\cos\bigl(\boldsymbol{\omega} \cdot \xbf\bigr),
  \sqrt{c_{\boldsymbol{\omega}}}\sin\bigl(\boldsymbol{\omega} \cdot \xbf \bigr)
\Bigr\}_{\boldsymbol{\omega} \in \Omega}.
\]
 
\paragraph{Approximation via Random Fourier Features}
Following the approach in \citep{landman2022classicapprox}, one can sample $D$ frequencies $(\boldsymbol{\omega}_1,...,\boldsymbol{\omega}_D)$ uniformly from $\Omega$, and then construct the approximate kernel $$\kappa_{\text{RFF}}(\xbf,\xbf'):= \phibf_{\rm RFF}(\xbf)^T\phibf_{\rm RFF}(\xbf'),$$
where \[
    \phibf_{\rm RFF}(\mathbf{x}) \;=\;
    \sqrt{\tfrac{1}{D}}\,
    \begin{bmatrix}
       \cos\bigl(\boldsymbol{\omega}_1 \!\cdot \!\mathbf{x}\bigr),    \ldots, 
       \sin\bigl(\boldsymbol{\omega}_D \!\cdot \!\mathbf{x}\bigr)
    \end{bmatrix}^\top.
\]

\paragraph{Uniform approximation of the target function via random bases.}
To control the misspecification error induced by replacing the full feature $\phibf_Q$ with a random low-dimensional feature map,
we use the following uniform approximation theorem of \citet{rahimi2008uniform}.

\begin{theorem}[Uniform approximation with random bases {\citep[Theorem~3.2]{rahimi2008uniform}}]
\label{thm:RR32}
Let $\mathcal{X}\subset\mathbb{R}^d$ be compact and let
$\varphi(\xbf;\thetabf)=\varphi(\thetabf^\top \xbf)$, where $\varphi:\mathbb{R}\to\mathbb{R}$ is $L$-Lipschitz, satisfies
$\varphi(0)=0$, and $|\varphi|\le 1$.
Let $p$ be a distribution over $\thetabf$ with finite second moment.
Define the mixture class
\[
\mathcal{F}(\mathcal{X},\Theta,\varphi,p)
=
\left\{
f(\xbf)=\int_{\Theta}\alpha(\thetabf)\varphi(\xbf;\thetabf)\,d\thetabf
\;:\;
\|f\|_{p}:=\sup_{\thetabf}\left|\frac{\alpha(\thetabf)}{p(\thetabf)}\right|<\infty
\right\}.
\]
Fix any $f\in\mathcal{F}$ and let $B\!_{\mathcal{X}}=\sup_{\xbf\in\mathcal{X}}\|\xbf\|_2$.
Draw $\thetabf_1,\dots,\thetabf_D\stackrel{\text{iid}}{\sim}p$.
Then for any $\delta>0$, with probability at least $1-\delta$,
there exist coefficients $c_1,\dots,c_D$ such that
\[
\hat f(\xbf)=\sum_{k=1}^D c_k\varphi(\thetabf_k^\top \xbf)
\]
satisfies the uniform bound
\[
\|\hat f-f\|_{\infty}
<
\frac{\|f\|_p}{\sqrt{D}}
\left(
\sqrt{\log\frac{1}{\delta}}
+
4LB\!_{\mathcal{X}}\sqrt{\mathbb{E}_{\thetabf\sim p}\|\thetabf\|_2^2}
\right).
\]
\end{theorem}
\noindent
This theorem yields a \emph{uniform} $O(1/\sqrt{D})$ approximation rate on the entire compact domain $\mathcal{X}$. In \citep{rahimi2008uniform}, the authors discuss that Fourier features of the form
$\varphi(\xbf;\thetabf)=\cos(\omegabf^\top \xbf+b)$ satisfy the conditions of Theorem~\ref{thm:RR32}.

\subsection{RFF with general encoding Quantum Kernel}

When one or more encoding Hamiltonians in a variational quantum circuit (VQC) are not trivially diagonalizable, directly identifying and sampling from their associated frequency set can be challenging. To address this,~\citet{landman2022classicapprox} propose a more flexible grid-based sampling approach to approximate the quantum kernel. The central idea involves sampling from a grid of frequencies evenly spaced between zero and an upper bound, \(\omegabf_{\rm max}\), along each dimension. Here, \(\omegabf_{\rm max}\) is chosen as the maximum possible frequency that can appear in the VQC---such as the one inferred via the Shannon criterion or derived from known bounds on the largest eigenvalues of the encoding Hamiltonians.


Specifically, suppose each dimension of $\xbf$ could contain frequencies in $[0,\omegabf_{\rm max}]$. We partition this
interval into regularly spaced bins of size $s>0$, giving $\lceil \omegabf_{\rm max}/s\rceil^d$ possible frequency values per
dimension. Over $d$ dimensions, this yields a regular grid of size $[0,\omegabf_{\rm max}]^d$. One then sample $D$ frequencies $(\boldsymbol{\omega}_1,\ldots,\boldsymbol{\omega}_D)$ from the grid, and similarly construct the approximate kernel $k_{\text{RFF}}(\xbf,\xbf')= \phibf_{\rm RFF}(\xbf)^\top\phibf_{\rm RFF}(\xbf')$  where, as before, \[
    \phibf_{\rm RFF}(\mathbf{x}) \;=\;
    \sqrt{\tfrac{1}{D}}\,
    \begin{bmatrix}
       \cos\bigl(\boldsymbol{\omega}_1 \!\cdot \!\mathbf{x}\bigr)  & \ldots &
       \cos\bigl(\boldsymbol{\omega}_D \!\cdot \!\mathbf{x}\bigr)
    \end{bmatrix}^\top.
\]

Then we have 
\citep{landman2022classicapprox}
\begin{theorem}[RFF for Grid Sampling]\label{thm:VQC_RFF_2}
  Let $\mathcal{X}$ be a compact subset of $\mathbb{R}^d$, and $\varepsilon > 0$.
  Consider a training set $\{(\xbf_i, y_i)\}_{i=1}^M$. Let $f^*$ be a 
  VQC model with any Hamiltonian encoding, with a maximum individual 
  frequency $\omega_{\max}$ and full freedom on the associated frequency 
  coefficients, trained with a regularization $\lambda$. 
  Let $\sigma_y^2 \;=\; \frac{1}{M}\sum_{i=1}^M y_i^2$ and $|\mathcal{X}|$ be the diameter of $\mathcal{X}$.
  Let $\tilde{f}$ be the RFF model with $D$ samples in the grid strategy, 
  trained on the same dataset and the same regularization. 
  Let $C \;=\; \|f^*\|_{\infty}|\mathcal{X}|$ and $s$ be the sampling rate as in the grid strategy.
  Then, for $0 < s < \tfrac{1}{C}$, we can guarantee 
  \(
    \sup_{\xbf \in \mathcal{X}}|f^*(\xbf) - \tilde{f}(\xbf)| \le \varepsilon
  \)
  with probability $1 - \delta$, provided the number $D$ of samples satisfies
  \begin{align*}
          D &= \Omega \biggl(
      \frac{dC_1(1 + \lambda)}{\lambda^4(\varepsilon - sC)^2}
      \Bigl[
        \log\bigl(\omegabf_{\max}\,\lvert \mathcal{X}\rvert\bigr)+ \log\Bigl(
            \frac{C_2(1 + \lambda)}{\lambda^2(\varepsilon - sC)}
          \Bigr)
        - \log\delta
      \Bigr]
    \biggr),
  \end{align*}
  where $C_1$ and $C_2$ are constants depending on $\sigma_y$ and $\Xcal$. Note that $f^*$ can be equivalently defined to be the Kernel Ridge Regression (KRR) predictor derived from the full quantum kernel $\kappa$, and this theorem can be used similarly as Theorem~\ref{thm:PauliRRFF}
\end{theorem}

\subsection{A Direct Truncated Kernel Approximation Approach}

Recall that for a translation invariant quantum kernel, we can express its feature map as \[
\Bigl\{
  \sqrt{c_{\mathbf s}}\cos\bigl(\mathbf s\cdot \xbf\bigr),
  \sqrt{c_{\mathbf s}}\sin\bigl(\mathbf s\cdot \xbf \bigr)
\Bigr\}_{\mathbf s \in \Omega}.
\]

\paragraph{Truncated Features for Kernel Approximation}
To construct a smaller feature map (or kernel) from $\kappa_Q$, a
natural idea is to pick only the largest-magnitude coefficients $\{c_{\mathbf s}\}$ and retain those frequencies $\mathbf s$ in the approximation
while discarding others. Concretely, let $S \subset \Omega$ be a
subset of frequencies corresponding to the largest values of $\{c_{\mathbf s}\}$ and define
\[
\kappa_{\text{small}}(\mathbf x,\mathbf x')
=
\sum_{\mathbf s\in S\subseteq\Omega}
c_{\mathbf s}
\Bigl[\cos\bigl(\mathbf s\cdot\mathbf x\bigr)\cos\bigl(\mathbf s\cdot\mathbf x'\bigr)
+
\sin\bigl(\mathbf s\cdot\mathbf x\bigr)\sin\bigl(\mathbf s\cdot\mathbf x'\bigr)\Bigr].
\]
This kernel $\kappa_{\text{small}}$ is essentially a truncated version of $\kappa_Q$: it retains only frequencies in $S$. We now show that a
function $f$ drawn from ${\rm GP}(0,\kappa_Q) $ can be uniformly
approximated by a function in the RKHS of $\kappa_{\text{small}}$, with an error bounded by the sum of the
left-out coefficients.

\begin{theorem}[Approximation by Truncating Fourier Series] \label{thm:approx_truncate_fourier} 
Suppose a quantum kernel admits a Fourier series
\[
\kappa_Q(\xbf,\xbf') = \sum_{\mathbf{s}\in \Omega}
c_{\mathbf{s}}\,\cos\bigl(\mathbf{s}\cdot(\xbf - \xbf')\bigr),
\]
and let \(f^*\) be a function sampled from \(\mathrm{GP}(0,\kappa_Q)\). Consider the truncated kernel
\[
\kappa_{\text{small}}(\mathbf x,\mathbf x')
=
\sum_{\mathbf s\in S\subset \Omega}
c_{\mathbf s}
\Bigl[\cos\bigl(\mathbf s\cdot\mathbf x\bigr)\cos\bigl(\mathbf s\cdot\mathbf x'\bigr)
+
\sin\bigl(\mathbf s\cdot\mathbf x\bigr)\sin\bigl(\mathbf s\cdot\mathbf x'\bigr)\Bigr],
\]
with the associated RKHS $\mathcal{H}_{\kappa_{\text{small}}}(\mathcal{X})$. Then, the expected uniform approximation error satisfies
\[
\mathbb{E}\Bigl[
  \min_{f \in\mathcal{H}_{\kappa_{\text{small}}}(\mathcal{X})}
  \|f - f^*\|_\infty
\Bigr]
\le
\sum_{\mathbf s\in \Omega \setminus S}2\sqrt{c_{\mathbf s}}.
\]
\end{theorem}

\begin{proof}
We outline the proof in the following three steps.\\

    \textbf{1. Sampling $f$ from \(\mathrm{GP}(0,\kappa_Q)\):}
    
By standard Gaussian process theory, a function $f$ drawn from \(\mathrm{GP}(0,\kappa_Q)\) admits a representation in terms of feature basis corresponding to $\kappa_Q$. Concretely, one may write $$ 
f(\xbf)
=
\sum_{\boldsymbol{\omega}\in\Omega} \alpha_{\boldsymbol{\omega}}\,\psi_{\boldsymbol{\omega}}(\mathbf x),
$$
where $\alpha_{\boldsymbol{\omega}}\sim\mathcal{N}(0,1)$, and $\{\psi_{\boldsymbol{\omega}}\}$ are orthonormal basis functions scaled by $\sqrt{c_{\boldsymbol{\omega}}}$. In our case, we can take $\psi_{\boldsymbol{\omega}}=(\sqrt{c_{\boldsymbol{\omega}}}\cos(\boldsymbol{\omega}  \cdot \mathbf x)+\sqrt{c_{\boldsymbol{\omega}}}\sin(\boldsymbol{\omega}  \cdot \mathbf \xbf))$. For notational convenience, we may redefine $\alpha_{\boldsymbol{\omega}}\leftarrow \sqrt{c_{\boldsymbol{\omega}}}\alpha_{\boldsymbol{\omega}}$ so that $\psi_{\boldsymbol{\omega}}$ is just $(\cos(\boldsymbol{\omega}  \cdot \xbf)+\sin(\boldsymbol{\omega} \cdot  \xbf))$.

\textbf{2. Restricting to a Smaller Kernel $k_{\rm small}$:}

Suppose we now try to approximate $f^*$ by a function $g$ lying in
the RKHS associated with the truncated kernel $\kappa_{\rm small}$. This space corresponds to a smaller set of frequencies $S \subset \Omega$. Consequently, any $g(\xbf)\in \mathcal H_{\kappa_{\rm small}}(\mathcal{X})$ can be written as 
\[
g(\xbf)
=
\sum_{\boldsymbol{\omega}\in S} \beta_{\boldsymbol{\omega}}\psi_{\boldsymbol{\omega}}(\xbf),
\]
where $\{\beta_{\boldsymbol{\omega}}\}$ are real coefficients
determined by the regression procedure between $f^*$ and $g$.

Since the functions $\{\psi_{\boldsymbol{\omega}}\}$ are orthogonal on $[0,2\pi]^d$ (i.e., $\int_{[0,2\pi]^d} \psi_{\boldsymbol{\omega}}(\xbf)\psi_{\boldsymbol{\omega}'}(\xbf) \, d\xbf= 0$ for all $\boldsymbol{\omega}\ne \boldsymbol{\omega}'$), the best fit to $f^*$ in the mean-square sense is
achieved by matching coefficients on those frequencies in $S$. In
other words, for each $\boldsymbol{\omega}\in S \cap\Omega$, the optimal $\beta_{\boldsymbol{\omega}}$ is $\alpha_{\boldsymbol{\omega}}$; frequencies outside $S \cap \Omega$ cannot be matched, leaving those $\beta_{\boldsymbol{\omega}}=0$.

This can be seen by expanding the objective 
\begin{align*}
&\int_{[0,2\pi]^d} \left[f(\xbf) - g(\xbf) \right]^2  \, d\xbf\\
&\; = \int_{[0,2\pi]^d}f(\xbf)^2 - 2f(\xbf)g(\xbf) +g(\xbf)^2\, d\xbf
\end{align*}
and noting that in 
\begin{align*}
&\int_{[0,2\pi]^d}f(\xbf)g(\xbf) \, d\xbf\\
&=
\int_{[0,2\pi]^N} \sum_{\boldsymbol{\omega}} \alpha_{\boldsymbol{\omega}}\beta_{\boldsymbol{\omega}} \psi_{\boldsymbol{\omega}}(\xbf) + \sum_{\mathbf s, \mathbf t, \mathbf s\neq \mathbf t}\alpha_{\mathbf s}\beta_{\mathbf t}\psi_{\mathbf s}(\xbf)\psi_{\mathbf t}(\xbf)\, d\xbf,   
\end{align*}
the second summand is zero by the orthogonality of $\{\psi_{\boldsymbol{\omega}}\}$.
Thus, the solution matching $\beta_{\boldsymbol{\omega}}=\alpha_{\boldsymbol{\omega}}$ for $\boldsymbol{\omega} \in S \cap\Omega$ and $\beta_{\boldsymbol{\omega}}=0$ otherwise minimizes the integral, and hence yields the optimal least-squares
approximation.

\textbf{3. Uniform Approximation Error:}

Finally, consider the uniform approximation error,
$
\mathbb{E}\Bigl[
  \min_{f \,\in\, \mathcal{H}_{\kappa_{\rm small}}(\mathcal{X})}
  \|f - f^*\|_\infty
\Bigr].
$
Since the optimal $f\in \mathcal H_{\kappa_{\rm small}}(\mathcal{X})$ contains coefficients that match the corresponding ones in  $f^*$, we obtain 
\begin{align*}
&\min_{f\in \mathcal H_{\kappa_{\rm small}}(\mathcal{X})} \|f - f^*\|_\infty \\
&= \left|\sum_{\mathbf s\in \Omega \setminus S} \sqrt{c_{\mathbf s}}\left(\cos(\mathbf s\cdot \xbf)+\sin(\mathbf s\cdot \xbf)\right)\right|\\
&\leq \sum_{\mathbf s\in \Omega \setminus S}2\sqrt{c_{\mathbf s}},
\end{align*}
which is the statement to be proved.
\end{proof}

\section{Proof of Theorem 3}\label{app:pfthm3}

\subsection{Introduction to SquareCB} \label{subsec:introSquareCB}

\begin{algorithm}[H]
\caption{SquareCB}
\label{alg:SquareCB}
\textbf{Parameters:} Learning rate $\gamma > 0$, exploration parameter $\mu > 0$, online regression oracle SqAlg.
\begin{algorithmic}[1]
\FOR{$t = 1, \ldots, T$}
    \STATE Receive context $\xbf_t$.
    \STATE For each action $\abf \in \mathcal{A}$, compute $\hat{y}_{t,\abf} := \hat{y}_t(\xbf_t, \abf)$
    \STATE Let $\bbf_t = \arg\min_{\abf \in \mathcal{A}} \hat{y}_{t,\abf}$
    \STATE For each $\abf\neq \bbf_t$, define $p_{t,\abf} = \frac{1}{\mu + \gamma (\hat{y}_{t,\abf} - \hat{y}_{t,\bbf_t} )}$
    and let $p_{t,\bbf_t} = 1 \;-\; \sum_{\abf \neq \bbf_t} p_{t,\abf}$
    \STATE Sample action $\abf_t \sim p_t$ and observe loss $\ell_t(\abf_t)$.
    \STATE Update SqAlg with example $\bigl((\xbf_t, \abf_t), \ell_t(\abf_t)\bigr)$.
\ENDFOR
\end{algorithmic}
\end{algorithm}

\citet{foster2020mislin} considered the following contextual bandit setting: Over \(T\) rounds:
\begin{itemize}[leftmargin=*]
  \item Nature or an adversary picks a context \(\xbf_t\in\mathcal{X}\) and a loss function \(\ell_t(\cdot)\colon \mathcal{A}\to [0,1]\).
  \item Learner observes \(\xbf_t\), selects \(\abf_t\in \mathcal{A}\).
  \item Learner incurs and observes loss \(\ell_t(\abf_t)\).
\end{itemize}
Note that we are considering the linear bandit setting in this work, which is simply a special case of their setting (with constant $\xbf_t$).

The learner has access to a function class \(\mathcal{F}\subseteq\{f:\mathcal{X}\times\mathcal{A}\to [0,1] \}\).
Under \emph{realizability}, there is an \(f\in\mathcal{F}\) so that
\(\mathbb{E}[\ell_t(\abf)\mid \xbf_t]=f(\xbf_t,\abf)\).
In a \emph{misspecified} setting, one relaxes this equality to
\[
  \bigl|
    \mathbb{E}[\ell_t(\abf)\mid \xbf_t]-f(\xbf_t,\abf)
  \bigr|
  \le\varepsilon,
  \quad\forall t,\abf
\]
for some \(\varepsilon>0\). Note this agrees with our definition for misspecified setting in the main text, namely that there exists $f \in \mathcal{F}$ such that $|f(\xbf_t,\abf)-f^*(\xbf_t,\abf)|\leq \varepsilon$. Here $f^*(\xbf_t,\abf) = \mathbb{E}[\ell_t(\abf)\mid \xbf_t]$.

 \citet{foster2020mislin}  introduced SquareCB, a contextual linear bandit algorithm. See Algorithm \ref{alg:SquareCB} for  details. 

SquareCB uses an online algorithm called SqAlg that provides predictions $\hat y_t(\xbf_t,\abf) \in [0,1]$, and receives the true label (loss) $y_t\in[0,1]$. Over $T$ rounds, SqAlg aims to do as well as the best function $f\in \mathcal F$ in  terms of the square loss, i.e., it satisfies a bound of the form \[
  \sum_{t=1}^T
    \bigl(\widehat{y}_t - y_t\bigr)^2 - \inf_{f\in\mathcal{F}}
  \sum_{t=1}^T
    \bigl(f(\xbf_t,\abf_t)-y_t\bigr)^2
  \le
  \mathrm{Reg}_{\mathrm{Sq}}(T),
\]
where $\mathrm{Reg}_{\mathrm{Sq}}(T)$ is typically called the square-loss regret of the online forecaster. If $\mathcal F$ is a simple class (e.g., linear in our case), there are known efficient algorithms that achieve small $\mathrm{Reg}_{\mathrm{Sq}}(T)$.

In the misspecified setting, \cite{foster2020mislin} showed the following:

\begin{theorem}[Pseudoregret Bound for SquareCB]
Under the misspecified setting, SquareCB with $\mu = K:=|\mathcal A|$ and $\gamma 
=  \sqrt{8KT/(\mathrm{Reg}_{\rm Sq}(T) + \varepsilon^2 T)}$ ensures that
\begin{align*}
\sup_{\pi} 
\mathbb{E} \Biggl[
  \sum_{t=1}^T \ell_t(\abf_t) 
  -
  \sum_{t=1}^T \ell_t\bigl(\pi(\xbf_t)\bigr)
\Biggr]
&\le
2 \sqrt{K T\ \mathrm{Reg}_{\rm Sq}(T)}+
  5\varepsilon \sqrt{K} T,
\end{align*}
where $\abf_t$ is the action chosen with SquareCB at time $t$, and \(\sup_{\pi}\) ranges over all policies \(\pi : \mathcal X \to \mathcal A\). Then, by instantiating
SquareCB with the Vovk-Azoury-Warmuth forecaster, which has $\mathrm{Reg}_{\rm Sq}(T) \leq D \log(1+T/D)$, one gets an efficient algorithm with regret 
\begin{align*}
R_T=\mathbb{E} \Biggl[
  \sum_{t=1}^T \ell_t(\abf_t) 
  -
  \sum_{t=1}^T \ell_t\bigl(\pi^*(\xbf_t)\bigr)
\Biggr]
&\le
2 \sqrt{K TD \log(1+T/D)} +
5\varepsilon \sqrt{K} T.
\end{align*}
with $\pi^*$ being the optimal policy.
\end{theorem}

\subsection{Choosing Optimal $D$ and The Proof of Theorem \ref{thm:regret_RFF}} \label{app_pf_regret_rff2}

We now apply the SquareCB algorithm, combined with the Distinct-Sampling RFF technique, to the linear bandit setting where, at each round $t$,  the following hold:
\begin{itemize}[leftmargin=*]
    \item The context $\xbf_t$ is now constant. 
    \item The reward (or equivalently, the loss function) is linear in the quantum feature map $\rhobf(\abf)$. Specifically, $\mathbb E[\ell_t(\abf)|\xbf_t]={\rm Tr}[\Hbf\rhobf(\abf)] $ for some observable $\Hbf$ and known quantum feature map $\rhobf(\abf)\in \mathbb C^{2^n\times 2^n}$. In other words, $f^*(\abf):=\mathbb E[\ell_t(\abf)|\xbf_t]$ belongs to the RKHS induced by the quantum kernel $\kappa_Q$.
\end{itemize}
By construction, $f^*$ belongs to the RKHS induced by the quantum kernel $\kappa_Q$.

We now instantiate Theorem~\ref{thm:RR32} for the Pauli-encoding kernel $\kappa_Q$ above.
Define the parameter set
\[
\Theta:=\Omega\times\{c,s\},
\qquad
p(\omegabf,u):=\frac{1}{2|\Omega|}\quad\text{(uniform over }\Theta\text{)},
\]
and define the real basis functions
\[
\varphi(\xbf;(\omegabf,c)):=\frac{1}{2}\cos(\omegabf^\top \xbf) - \frac{1}{2},\qquad
\varphi(\xbf;(\omegabf,s)):=\sin(\omegabf^\top \xbf).
\]
Note that any function $f\in\mathcal{H}_{\kappa_Q}$ can be written as a finite mixture of these bases:
\[
f(\xbf) 
=
c_0 +\sum_{\omegabf\in\Omega}
\Bigl(
a_{\omegabf}\cos(\omegabf^\top \xbf)+b_{\omegabf}\sin(\omegabf^\top \xbf)
\Bigr),
\]
which corresponds to a discrete mixture representation in the class $\mathcal{F}(\mathcal{X},\Theta,\varphi,p)$.
The scaling by $\frac{1}{2}$ in $\varphi(\xbf;(\omegabf,c))$ is merely to match the requirement for Theorem~\ref{thm:RR32} and does not affect the conclusion since it just shifts the result by a fixed constant.

We obtain the following corollary, which we will use as the misspecification guarantee
for SquareCB.

\begin{theorem}[Uniform misspecification bound of RFF approximation of Pauli-induced quantum $f^*$]
\label{thm:PauliRRFF}
Let $\kappa_Q$ be a shift-invariant Pauli-encoding quantum kernel on a compact domain $\mathcal{X}\subset\mathbb{R}^d$
with finite spectrum $\Omega\subset\mathbb{Z}^d$, and let $f^*\in\mathcal{H}_{\kappa_Q}$.
Fix $\delta\in(0,1)$ and draw $(\omegabf_1,u_1),\dots,(\omegabf_D,u_D)\stackrel{\text{iid}}{\sim}p$ uniformly from $\Theta=\Omega\times\{c,s\}$.
Define the random feature map
\[
\phibf_{\rm RFF}(\xbf)
=
\frac{1}{\sqrt{D}}
\big[1,
\varphi(\xbf;(\omegabf_1,u_1)),
\dots,
\varphi(\xbf;(\omegabf_D,u_D))
\big]^\top
\in\mathbb{R}^{D+1}.
\]
Then with probability at least $1-\delta$, there exists a linear predictor $\tilde f(\xbf)=\langle \wbf,\phibf_{\rm RFF}(\xbf)\rangle$
such that
\begin{equation}
\varepsilon_D
:=\sup_{\xbf\in\mathcal{X}}|f^*(\xbf)-\tilde f(\xbf)|
<
\frac{\|f^*\|_{p}}{\sqrt{D}}
\left(
\sqrt{\log\frac{1}{\delta}}
+
4B_\Xcal\,\omega_{\max}
\right), \label{eqn:rfferror}
\end{equation}
where $B_\Xcal=\sup_{\xbf\in\mathcal{X}}\|\xbf\|_2$ and $\omega_{\max}:=\max_{\omegabf\in\Omega}\|\omegabf\|_2$.

Note that since $\Omega$ is finite in any fixed-$n$ qubit system, we have assumed $\Vert f^* \Vert_{\mathcal{H}_{\kappa_Q}}$ is bounded by $B$ and $p$ is uniform, we have that $\omega_{\max}$ and $\|f^*\|_p$ are all finite constants.
\end{theorem}





By \cite{foster2020mislin}, under the misspecified setting, SquareCB with $\mu=K:=|\mathcal A|$ and
$\gamma=\sqrt{8KT/(\mathrm{Reg}_{\rm Sq}(T)+\varepsilon^2T)}$ guarantees the expected pseudoregret bound
\[
\sup_{\pi}\mathbb E\left[\sum_{t=1}^T \ell_t(\abf_t)-\sum_{t=1}^T \ell_t(\pi(\xbf_t))\right]
\le 2\sqrt{KT\, \mathrm{Reg}_{\rm Sq}(T)} + 5\varepsilon\sqrt{K}T,
\]
where $\mathrm{Reg}_{\rm Sq}(T)$ is the square-loss regret of the regression oracle.

We instantiate the regression oracle by the Vovk--Azoury--Warmuth forecaster. For a $d$-dimensional
linear class, this gives $\mathrm{Reg}_{\rm Sq}(T) \leq d\log(1+T/d)$, hence
\begin{equation}
\mathbb E[R_T]\le2\sqrt{KT\cdot d\log(1+T/d)} + 5\varepsilon\sqrt{K}T \label{squarecb_full_regret}
\end{equation}

\paragraph{Case 1: Exact quantum model.}
Using the exact quantum feature map with ambient dimension $d=4^n$ yields a realizable model (so $\varepsilon=0$),
and therefore
\[
\mathbb E[R_T]\le2\sqrt{KT\cdot 4^n\log\!\Big(1+\tfrac{T}{4^n}\Big)}.
\]

\paragraph{Case 2: RFF approximate model.}
Using a $D$-dimensional RFF surrogate gives a misspecification level $\varepsilon=\varepsilon_D$.
Set $\delta=1/T$ in \ref{eqn:rfferror}, then with probability at least $1-1/T$ (over the sampled Fourier features),
\[
\varepsilon_D
<
\frac{\|f^*\|_{p}}{\sqrt{D}}
\left(
\sqrt{\log T}
+
4B_\Xcal\,\omega_{\max}
\right).
\]
On the (probability $\ge 1-1/T$) event above, we substitute the explicit bound on $\varepsilon_D$ as in \ref{eqn:rfferror}; on the
complement event we use the trivial bound $R_T\le T$. Taking total expectation gives
\[
\mathbb E[R_T]
\le
2\sqrt{KT\cdot D\log\!\Big(1+\tfrac{T}{D}\Big)}
+
5\sqrt{K}T\, 
\frac{\|f^*\|_{p}}{\sqrt{D}}
\left(
\sqrt{\log T}
+
4B_\Xcal\,\omega_{\max}
\right)
+1.
\]

Choosing $D=\lceil \sqrt{T}\rceil$ yields
\[
\mathbb E[R_T]
\le
2\sqrt{K}T^{3/4}\sqrt{\log T}
+
5\sqrt{K}T^{3/4}\|f^*\|_{p}\Big(\sqrt{\log T}+4B_\Xcal\,\omega_{\max}\Big)
+1
=\tilde{\mathcal O}\big(\sqrt{K}T^{3/4}\big)
\]

Finally, taking the minimum of the exact-model bound and the RFF-surrogate bound yields the result in
Theorem~\ref{thm:regret_RFF}.

\subsection{Bound on $\|f^*\|_{p}$}

Assume the quantum kernel is shift-invariant with a finite Fourier expansion
\[
\kappa_Q(\xbf,\xbf') \;=\; \sum_{\omegabf\in \Omega} c_{\omegabf} \cos\!\big(\omegabf^\top(\xbf-\xbf')\big).
\]
Let $f^*\in\mathcal H_{\kappa_Q}$ with $\|f^*\|_{\mathcal H_{\kappa_Q}}\le B$.
Then we show that
\[
\|f^*\|_{p}\;\le\;4|\Omega|B.
\]
This, together with the derivation in Appendix~\ref{app_pf_regret_rff2}, yields the result in Theorem~\ref{thm:regret_RFF}.

\begin{proof}
Recall that if we defined the feature map
\[
\phibf(\xbf)
:=\big(\sqrt{c_{\omegabf}}\cos(\omegabf^\top \xbf),\ \sqrt{c_{\omegabf}}\sin(\omegabf^\top \xbf)\big)_{\omegabf\in\Omega}
\in\mathbb R^{2|\Omega|},
\]
then $\kappa_Q(\xbf,\xbf')=\langle \phibf(\xbf),\phibf(\xbf')\rangle$, and every $f^*\in\mathcal H_{\kappa_Q}$ can be written as
$f^*(\xbf)=\langle \wbf,\phibf(\xbf)\rangle$ for some $\wbf\in\mathbb R^{2|\Omega|}$ with
$\|f^*\|_{\mathcal H_{\kappa_Q}}=\|\wbf\|_2$.
Write $\wbf=(u_{\omegabf},v_{\omegabf})_{\omegabf \in\Omega}$ such that
\[
f^*(\xbf)=\sum_{\omegabf\in\Omega}\Big(u_{\omegabf}\sqrt{c_{\omegabf}}\cos(\omegabf^\top \xbf)
+v_{\omegabf}\sqrt{c_{\omegabf}}\sin(\omegabf^\top \xbf)\Big)
=\sum_{\omegabf\in\Omega}\Big(a_{\omegabf}\cos(\omegabf^\top \xbf)+b_{\omegabf}\sin(\omegabf^\top \xbf)\Big),
\]
where $a_{\omegabf}:=u_{\omegabf}\sqrt{c_{\omegabf}}$ and $b_{\omegabf}:=v_{\omegabf}\sqrt{c_{\omegabf}}$.
Since $\kappa_Q(\xbf,\xbf)=\sum_{\omegabf\in\Omega}c_{\omegabf}\le 1$, we have $c_{\omegabf}\le 1$ for all $\omegabf$, hence
\[
|a_{\omegabf}|\le |u_{\omegabf}|\sqrt{c_{\omegabf}}\le \|\wbf\|_2\sqrt{c_{\omegabf}}\le \|\wbf\|_2\le B,
\quad \text{similarly, } \quad
|b_{\omegabf}|\le B.
\]

Now note that $\varphi(0;(\omegabf,c))=0$, so the centered function
$g(\xbf):=f^*(\xbf)-f^*(0)$ admits the exact mixture representation
\[
g(\xbf)
=\sum_{\omegabf\in\Omega} a_{\omegabf}\big(\cos(\omegabf^\top \xbf)-1\big)+\sum_{\omegabf\in\Omega} b_{\omegabf}\sin(\omegabf^\top \xbf)
=\sum_{\omegabf\in\Omega}\big(2a_{\omegabf}\big)\,\varphi(\xbf;(\omegabf,c))
+\sum_{\omegabf\in\Omega} b_{\omegabf}\,\varphi(\xbf;(\omegabf,s)).
\]
Thus we may take $\alpha(\omegabf,c)=2a_{\omegabf}$ and $\alpha(\omegabf,s)=b_{\omegabf}$.
Also recall that we sampled the Random Features uniformly with $p(\omegabf,u)=\frac{1}{2|\Omega|}$, hence
\[
\|g\|_p
=\sup_{(\omegabf,u)\in\Theta}\left|\frac{\alpha(\omegabf,u)}{p(\omegabf,u)}\right|
=2|\Omega|\cdot \max\Big\{\max_{\omegabf}|2a_{\omegabf}|, \max_{\omegabf}|b_{\omegabf}|\Big\}
\le 2|\Omega|\cdot \max\{2B,B\}
=4|\Omega|B.
\]
\end{proof}

\subsection{Bound on  $\omega_{\max}$}
Recall $\omega_{\max}:=\max_{\omegabf\in\Omega}\|\omegabf\|_2$, where $\Omega$ is the frequency set in the Fourier/Bochner
representation of a shift-invariant quantum kernel $\kappa_Q(\mathbf{x},\mathbf{x}')=g(\mathbf{x}-\mathbf{x}')$, appeared in Theorem~\ref{thm:regret_RFF}. We hereby give a short discussion on $\omega_{\max}$ to show that it is moderate and usually negligible.

\paragraph{A simple upper bound.}
Consider any data-encoding circuit in which each input coordinate $x_j$ appears $m_j$ times via gates
$\exp(-i x_j \Gbf_{j,\ell})$, $\ell=1,\dots,m_j$.
Let $\Delta_{j,\ell}:=\lambda_{\max}(\Gbf_{j,\ell})-\lambda_{\min}(\Gbf_{j,\ell})$ be the spectral diameter of the generator.
Then every frequency vector $\omegabf\in\Omega$ satisfies
\begin{align}
|\omega_j|\ \le\ \sum_{\ell=1}^{m_j}\Delta_{j,\ell}\qquad (j=1,\dots,d),
\label{eq:omegamax_coord_bound}
\end{align}
and hence
\begin{align}
\omega_{\max}\ \le\ \left\|\Big(\sum_{\ell=1}^{m_1}\Delta_{1,\ell},\dots,\sum_{\ell=1}^{m_d}\Delta_{d,\ell}\Big)\right\|_2 < \infty.
\label{eq:omegamax_bound}
\end{align}

We next give a concrete bound on $\omega_{\max}$ for some example cases.

\paragraph{Example 1.}
Suppose $d=n$ and each $x_j$ is encoded once on one qubit using $\Gbf_j\in\{\Xbf,\Ybf,\Zbf\}$, so $m_j=1$.
Since $\Xbf,\Ybf,\Zbf$ have eigenvalues $\pm 1$, we have $\Delta_{j,1}=2$ and thus $|\omega_j|\le 2$ for all $j$.
Therefore
\[
\omega_{\max}\ \le\ \sqrt{\sum_{j=1}^n 2^2}\ =\ 2\sqrt{n}\ =\ \mathcal{O}\left(\sqrt{n}\right).
\]

\paragraph{Example 2.}
Let $d=1$ and encode the same scalar $x$ on each of $n$ qubits once, e.g.,
$|\psi(x)\rangle = (e^{-i x \Xbf/2}|0\rangle)^{\otimes n}$.
Then the fidelity kernel is
\[
\kappa(x,x')=\big|\langle\psi(x)|\psi(x')\rangle\big|^2=\cos^{2n}\!\Big(\frac{x-x'}{2}\Big).
\]
Using $\cos(\theta)=\tfrac12(e^{i\theta}+e^{-i\theta})$,
\[
\cos^{2n}\!\Big(\frac{\Delta}{2}\Big)=2^{-2n}\sum_{j=0}^{2n}\binom{2n}{j}e^{i(n-j)\Delta}
=\sum_{\ell=-n}^{n} c_\ell e^{i\ell\Delta},\qquad \Delta:=x-x'.
\]
Hence $\Omega=\{-n,-n+1,\dots,n\}$ and
\[
\omega_{\max}=n.
\]

\paragraph{Example 3.}
Suppose $d=n$ and each coordinate $x_j$ is re-uploaded $L$ times with local Pauli generators
(so $m_j=L$ and each $\Delta_{j,\ell}=2$). Then by \eqref{eq:omegamax_bound},
\[
\omega_{\max}\ \le\ \sqrt{\sum_{j=1}^n (2L)^2}\ =\ 2L\sqrt{n}\ =\ \mathcal O\left(L\sqrt{n}\right)
\]

\section{Discussion on Running Time}\label{app:runningtime}

\begin{algorithm}[H]
\caption{GP-UCB}
\label{alg:GPUCB}
\textbf{Input:} Input (Action) space $\mathcal A$; GP Prior $\mu_0=0,\sigma_0,\kappa$.
\begin{algorithmic}[1]
\FOR{$t = 1, \ldots, T$}
    \STATE Choose $\xbf_t = \underset{\xbf \in \mathcal{A}}{\arg\max}\ \left[ \mu_{t-1}(\xbf)+\sqrt{\beta_t}\sigma_{t-1}(\xbf)\right]$, \text{where} $\beta_t=2\log(|\mathcal A|t^2\pi^2/6\delta)$
    \STATE Sample $y_t:=f^*(\xbf_t)+\eta_t$, \text{where} $\eta_t$ is the noise
    \STATE Update posterior mean and variance  $\mu_t$ and $\sigma_t$
\ENDFOR
\end{algorithmic}
\end{algorithm}

\begin{algorithm}[ht]
\caption{EC-GP-UCB (Enlarged Confidence GP-UCB)} \label{alg:ecgpucb}
\textbf{Input:} Kernel function $\kappa(\cdot,\cdot)$, domain $\mathcal{X}$, misspecification $\varepsilon$, and parameters $B, \lambda, \sigma$.
\begin{algorithmic}[1]
\STATE Set $\mu_0(\xbf) = 0$ and $\sigma_0(\xbf) = \kappa(\xbf,\xbf)$ for all $\xbf \in \mathcal{X}$. 
\FOR{$t = 1 \ldots T$}
    \STATE Choose 
    \[
       \xbf_t \in \underset{\xbf \in \mathcal{X}}{\arg\max}\; \left[\,
         \mu_{t-1}(\xbf) + \left(\beta_t + \frac{\varepsilon \sqrt{t}}{\sqrt{\lambda}}\right)\sigma_{t-1}(\xbf)
       \right].
    \]
    \STATE Observe reward $y_t^* \;=\; f^*(\xbf_t) \;+\; \eta_t.$
    \STATE Update posterior mean and variance  $\mu_t(\cdot)$ and $\sigma_t(\cdot)$
\ENDFOR
\end{algorithmic}
\end{algorithm}

In this work, we focus on the \textit{quantum kernel bandit} problem, where the unknown mean-reward function $f^*$ resides in the RKHS induced by a quantum kernel. Our objective is to minimize the   cumulative regret $R_T=\sum_{t=1}^T (f^*(\xbf^*) - f^*(\xbf_t))$ where $\xbf^* =\arg \max_{\xbf \in \mathcal{X}} f^*(\xbf)$ denotes the optimal action in hindsight.

A common approach, assuming the true (quantum) kernel is given, is to apply Gaussian process (GP) based methods such as GP-UCB \citep{srinivas2010}; see Algorithm \ref{alg:GPUCB}. However, the computational burden of GP inference is typically high. Specifically, at round $t$, updating the GP posterior naively (as in Appendix \ref{app:GPdetails} for line 4 in the algorithm) requires $\mathcal{O}(t^3)$ for an exact matrix inversion; or $\mathcal{O}(t^2)$ per round if one employs Sherman-Morrison rank-1 updates. Moreover, computing the next action $\xbf_t$ within the action set $\mathcal A = \mathcal X$ requires $\mathcal{O}(t^2|\mathcal A|)$ for posterior mean/variance queries. Summing over $T$ rounds yields a total time of order $$\mathcal{O}(T\times(T^2+T^2|\mathcal A|))=\mathcal{O}(T^3+T^3|\mathcal A|).$$  Such a cubic dependence on $T$ becomes intractable for large-scale bandit problems.

In contrast, our method considers a $D$-dimensional approximate feature map $\phibf_{\rm RFF}$ for the quantum kernel. When using SquareCB with the Vovk-Azoury-Warmuth forecaster in this $D$-dimensional feature space, the per-round update cost is only $\mathcal{O}(D^2)$. And inference (deciding $\xbf_t$) over the set $\mathcal A$ takes $\mathcal{O}(D\times|\mathcal A|)$. Consequently, the total cost across $T$ rounds is $$\mathcal{O}(TD^2+TD|\mathcal A|),$$
which grows linearly in $T$.

Alternatively, one may run EC-GP-UCB \citep{bogunovic2021misspecified} (see Algorithm \ref{alg:ecgpucb}) using the same $\phibf_{\rm RFF}$. Rather than maintaining a full GP posterior, one performs standard Bayesian linear regression in the $D$-dimensional space. As above, each round's update can be done in $\mathcal{O}(D^2)$ using rank-1 matrix updates (Sherman-Morrison), and the inference cost is $\mathcal{O}(D\times|\mathcal A|)$. Hence, this approach also achieves a total cost of $$\mathcal{O}(TD^2+TD|\mathcal A|).$$

Thus, by replacing the exact (and computationally heavy) GP posterior inference with a $D$-dimensional approximation via random Fourier features (or other finite-dimensional embeddings), one obtains linear dependence on $T$, which is far more tractable in large-scale bandit settings.

\section{About P-greedy}\label{app:pgreedy}

\begin{algorithm}[ht]
\caption{Construction of Newton basis with P-greedy algorithm}
\label{alg:pGreedy}
\textbf{Input:} Kernel $K$, admissible error $e > 0$, a subset of points $\hat{\Xcal} \subseteq \Xcal$ \\
\textbf{Output:} A subset of points $X_m \subseteq \hat{\Xcal}$ and Newton basis $N_1, \dots, N_m$ of $V(X_m)$

\begin{algorithmic}[1]

\STATE $\xibf_1 := \arg\max_{\xbf \in \hat{\Xcal}} K(\xbf,\xbf)$
\STATE $N_1(\xbf) := \frac{K(\xbf,\xibf_1)}{\sqrt{K(\xibf_1,\xibf_1)}}$

\FOR{$m = 1, 2, 3, \dots$}
    \STATE $P_m^2(\xbf) := K(\xbf,\xbf) \;-\; \sum_{k=1}^{m} \bigl(N_k(\xbf)\bigr)^2$
    \IF{$\max_{\xbf \in \hat{\Xcal}} P_m^2(\xbf) < e^2$}
        \STATE \textbf{return} $\{\xibf_1, \dots, \xibf_m\}$ and $\{N_1, \dots, N_m\}$
    \ENDIF
    \STATE $\xibf_{m+1} := \arg\max_{\xbf \in \hat{\Xcal}} P_m^2(\xbf)$
    \STATE $u(\xbf) := K(\xbf,\xibf_{m+1}) \;-\;\sum_{k=1}^{m} N_k(\xibf_{m+1})\,N_k(\xbf)$
    \STATE $N_{m+1}(\xbf) := \frac{u(\xbf)}{\sqrt{P_m^2(\xibf_{m+1})}}$
\ENDFOR

\end{algorithmic}
\end{algorithm}

We provide a more detailed introduction to P-greedy methods along with some established results to facilitate discussion in Appendix \ref{app:classicalrkhs}.

Let \(\Xcal\) be a non-empty subset of a Euclidean space \(\mathbb{R}^d\), and let \(K: \Xcal \times \Xcal \to \mathbb{R}\) be a symmetric, positive definite kernel defined on \(\Xcal\). For a given discretization \(\hat{\Xcal} \subset \Xcal\) and a function \(f \in \mathcal{H}_K(\Xcal)\), the goal is to approximate \(f\) using the P-greedy procedure. This method iteratively selects a small set \(X_n = \{\xibf_1, \ldots, \xibf_n\} \subset \hat{\Xcal}\) of \(n\) representative points from the candidate set \(\hat{\Xcal}\), which often serves as a discrete approximation of \(\Xcal\).
These points define a finite-dimensional subspace \(V(X_n)\) that approximates the entire RKHS \(\mathcal{H}_K(\Xcal)\). 

At each iteration, P-greedy evaluates an ``error surface'' via the power function
$$
P_{V(X)}(\xbf):=\sup_{f\in\mathcal H_{K}\setminus\{0\}} \frac{|f(\xbf)-(\Pi_{V(X)}f)(\xbf)|}{\|f\|_{\mathcal H_{K}(\Xcal)}},
$$
which quantifies the current approximation's looseness at each point \(\xbf\). The method identifies the point where this error is the largest and adds it to the set \(X\) to reduce the worst-case error in subsequent iterations. Specifically, the next point is chosen as
$$ 
\xibf_m = \arg\max_{\xbf \in \hat{\Xcal}} P_{V(X_{m-1})}(\xbf).
$$
Repeating this process for $m$ iterations results in a set of points \(X_m = \{\xibf_1, \ldots, \xibf_m\}\). These points generate the finite-dimensional subspace
$$
V(X_m) = \operatorname{span}\{K(\cdot, \xibf_1), \ldots, K(\cdot, \xibf_m)\},
$$
which approximates \(\mathcal{H}_K(\Xcal)\). The resulting approximation effectively yields an approximate kernel whose feature map corresponds to the Newton basis \( \{ N_1, \ldots, N_m \}\), obtained via Gram--Schmidt orthonormalization of the basis $\{K(\cdot, \xibf_1), \ldots, K(\cdot, \xibf_m)\}$.

We provide the pseudo-code for the P-greedy algorithm, adapted from \citet{takemori2021approximation}, in Algorithm \ref{alg:pGreedy}.

Furthermore, the following is known from \cite{takemori2021approximation}. \begin{theorem}
\label{thm:pgreedyerrorbound}
Let $\alpha, q > 0$ be parameters, $d$ be the dimension of input data, and denote by $D = D_{q,\alpha}(T)$ the number of points returned by the P-greedy algorithm with error $e = \alpha/T^{q}$. Then:
\begin{itemize}
\item[(i)] Suppose $K$ has \emph{finite smoothness} (e.g., Mat\'ern kernels) with smoothness parameter $\nu > 0$. Then
$
D_{q,\alpha}(T) = O\bigl(\alpha^{-\frac{d}{\nu}}  T^{\frac{d q}{\nu}}\bigr).
$
\item[(ii)] Suppose $K$ has \emph{infinite smoothness} (e.g., Gaussian kernel). Then
$
D_{q,\alpha}(T) = O\bigl( (q \log T -\log(\alpha))^{d}\bigr).
$
\end{itemize}
\end{theorem}

\section{P-greedy misspecification error}
\label{app:pgreedy_misspec}

In this section, we provide some ways to understand and bound the misspecification error from using the P-greedy algorithm.

\subsection{Setup}
Let $\mathcal X$ be compact and let $\kappa:\mathcal X\times\mathcal X\to\mathbb R$ be continuous,
symmetric, and positive definite with RKHS $\Hcal_\kappa$ and norm $\|\cdot\|_{\Hcal_\kappa}$.
For each $\xbf\in\mathcal X$ define the kernel section $\kappa_\xbf:=\kappa(\cdot,\xbf)\in\Hcal_\kappa$ and the compact set
\[
\mathcal F \;:=\; \{\kappa_\xbf:\;\xbf\in\mathcal X\}\subset \Hcal_\kappa .
\]
Given points $\xibf_1,\dots,\xibf_D\in\mathcal X$ selected by P-greedy, define the associated approximation space
\begin{equation}
\label{eq:VD-def}
V_D \;:=\; \mathrm{span}\{\kappa_{\xibf_1},\dots,\kappa_{\xibf_D}\}\subset\Hcal_\kappa,
\end{equation}
and let $\Pi_{V_D}:\Hcal_\kappa\to V_D$ denote the $\Hcal_\kappa$-orthogonal projector.

We are interested in the misspecification between $f^*\in\Hcal_\kappa$, and best possible approximation $f_D \in V_D$: 
\[
\varepsilon_D:=\varepsilon_D(f^*)
:=
\inf_{f_D\in V_D}\|f^* - f_D\|_\infty,
\qquad
\|g\|_\infty:=\sup_{\xbf\in\mathcal X}|g(\xbf)|.
\]

\subsection{Bounding misspecification error by the Kolmogorov width}

\paragraph{ $\varepsilon_D$ is controlled by the greedy error $\sigma_D(\mathcal F)$.}
We first recall \cite{greedyalgo2012dev}'s definition of greedy error for the compact set $\mathcal F\subset \Hcal_\kappa$:
\begin{equation}
\label{eq:sigma-def}
\sigma_D(\mathcal F)
\;:=\;
\sup_{k^*\in\mathcal F}\ \inf_{k_D\in V_D}\|k^*-k_D\|_{\Hcal_\kappa}.
\end{equation}

\begin{lemma}[Greedy error equals worst-case power function]
\label{lem:sigma-power}
For $V_D$ defined in~\eqref{eq:VD-def}, one has
\[
\sigma_D(\mathcal F)
=
\sup_{\xbf\in\mathcal X}\|\kappa_\xbf-\Pi_{V_D}\kappa_\xbf\|_{\Hcal_\kappa}.
\]
\end{lemma}

\begin{proof}
Fix $\xbf\in\mathcal X$. Since $\Pi_{V_D}\kappa_\xbf\in V_D$ and $\Pi_{V_D}$ is the best approximation in Hilbert norm,
\[
\inf_{f_D\in V_D}\|\kappa_\xbf - f_D\|_{\Hcal_\kappa}
=
\|\kappa_\xbf-\Pi_{V_D}\kappa_\xbf\|_{\Hcal_\kappa}.
\]
Taking $\sup_{f\in \mathcal{F}}$, which is equivalent to  $\sup_{\xbf\in\mathcal X}$, yields the claim.
\end{proof}

\begin{lemma}[Misspecification bounded by greedy error]
\label{lem:misspec-by-sigma}
For any $f^*\in\Hcal_\kappa$,
\[
\varepsilon_D(f^*)
\;\le\;
\|f^*\|_{\Hcal_\kappa}\,\sigma_D(\mathcal F).
\]
In particular, since we assumed  $\|f^*\|_{\Hcal_\kappa}\le B$, we have $\varepsilon_D(f^*)\le B\,\sigma_D(\mathcal F)$.
\end{lemma}

\begin{proof}
Let $f_D := \Pi_{V_D}f^*\in V_D$. Then $\varepsilon_D(f^*)\le\|f^*-f_D\|_\infty$.
Fix $\xbf\in\mathcal X$. By the reproducing property,
\[
(f^*-f_D)(\xbf)=\langle f^*-f_D,\ \kappa_\xbf\rangle_{\Hcal_\kappa}.
\]
Since $\Pi_{V_D}$ is an orthogonal projector, $\langle f^*-f_D,\ \Pi_{V_D}\kappa_\xbf\rangle_{\Hcal_\kappa}=0$ and therefore
\[
(f^*-f_D)(\xbf)=\langle f^*-f_D,\ \kappa_\xbf-\Pi_{V_D}\kappa_\xbf\rangle_{\Hcal_\kappa}.
\]
By Cauchy--Schwarz and properties of orthogonal projection,
\[
|(f^*-f_D)(\xbf)|
\le
\|f^*-f_D\|_{\Hcal_\kappa}\ \|\kappa_\xbf-\Pi_{V_D}\kappa_\xbf\|_{\Hcal_\kappa}
\le
\|f^*\|_{\Hcal_\kappa}\ \|\kappa_\xbf-\Pi_{V_D}\kappa_\xbf\|_{\Hcal_\kappa}.
\]
Taking the supremum over $\xbf\in\mathcal X$ and invoking Lemma~\ref{lem:sigma-power} yields
\[
\|f^*-f_D\|_\infty
\le
\|f^*\|_{\Hcal_\kappa}\,\sigma_D(\mathcal F).
\]
\end{proof}

\paragraph{$\sigma_{2D}(\mathcal F)$ is controlled by Kolmogorov width $d_D(\mathcal F)$.}
Define the Kolmogorov $D$-width of $\mathcal F$ in $\Hcal_\kappa$:
\begin{equation}
\label{eq:width-def}
d_D(\mathcal F)
\;:=\;
\inf_{\substack{Y_D\subset\Hcal_\kappa\\ \dim(Y_D)=D}}\ \sup_{k^*\in\mathcal F}\ \inf_{y_D\in Y_D}\|k^*-y_D\|_{\Hcal_\kappa}.
\end{equation}

It measures the best possible $D$-dimensional approximates to the full space.

We use the following result from \cite{greedyalgo2012dev}, which asserts that P-greedy typically achieves near-optimal performance: the error $\sigma_{2D}$ using $2D$ points satisfies $\sigma_{2D} \leq \gamma^{-1} \sqrt{2d_{D}}$ for some constant $0<\gamma<1$:

\begin{lemma}[greedy error bounded by Kolmogorov width] (Corollary 3.3 of \cite{greedyalgo2012dev})
\label{sigma_by_d}

\begin{equation}
\label{eq:devore-sigma-d}
\sigma_{2D}(\mathcal F)
\;\le\;
\frac{\sqrt{2}}{\gamma}\, \sqrt{d_D(\mathcal F)},
\end{equation}
for some constant $\gamma$.
\end{lemma}

Combining Lemma~\ref{lem:misspec-by-sigma} with~\eqref{eq:devore-sigma-d} gives the key implication:

\begin{equation}
\label{eq:misspec-main}
\varepsilon_{2D}(f^*)
\;\le\;
B\,\sigma_{2D}(\mathcal F)
\;\le\;
\frac{\sqrt{2}B}{\gamma}\, \sqrt{d_D(\mathcal F)}.
\end{equation}

\subsection{Bounding $d_D(\mathcal F)$ for finite Fourier-type kernels}

Further assume the kernel is translation invariant and admits the finite Fourier-type expansion
\begin{equation}
\label{eq:fourier-expansion}
\kappa(\xbf,\xbf')
=
\sum_{\omegabf\in\Omega} c_{\omegabf}\, e^{i\langle \omegabf,\,\xbf-\xbf'\rangle},
\qquad c_{\omegabf}\ge 0,
\end{equation}
where $\Omega$ is finite.
Let $\Omega_D\subset\Omega$ index the $D$ largest coefficients $\{c_{\omegabf}\}$.

For $\xbf\in\mathcal X$, the section function satisfies
\begin{equation}
\label{eq:section-expansion}
\kappa(\cdot,\xbf)
=
\sum_{\omegabf\in\Omega} c_{\omegabf}\, e^{-i\langle\omegabf,\xbf\rangle}\, e^{i\langle\omegabf,\cdot\rangle}.
\end{equation}

\paragraph{(A) A general upper bound}
Consider the $D$-dimensional subspace
\[
Y_D \;:=\; \mathrm{span}\{e^{i\langle \omegabf,\cdot\rangle}:\omegabf\in\Omega_D\}\subset \Hcal_\kappa.
\]
For each $\xbf\in\mathcal X$, define the truncation $g_D(\cdot;\xbf)\in Y_D$ by
\[
g_D(\cdot;\xbf)
:=
\sum_{\omegabf\in\Omega_D} c_{\omegabf}\, e^{-i\langle\omegabf,\xbf\rangle}\, e^{i\langle\omegabf,\cdot\rangle}.
\]
Then by~\eqref{eq:section-expansion},
\[
\kappa(\cdot,\xbf)-g_D(\cdot;\xbf)
=
\sum_{\omegabf\notin\Omega_D} c_{\omegabf}\, e^{-i\langle\omegabf,\xbf\rangle}\, e^{i\langle\omegabf,\cdot\rangle}.
\]
Using the triangle inequality in $\Hcal_\kappa$ and $|e^{-i\langle\omegabf,\xbf\rangle}|=1$,
\[
\|\kappa(\cdot,\xbf)-g_D(\cdot;\xbf)\|_{\Hcal_\kappa}
\le
\sum_{\omegabf\notin\Omega_D} c_{\omegabf}\, \|e^{i\langle\omegabf,\cdot\rangle}\|_{\Hcal_\kappa}.
\]
In the finite-rank RKHS induced by~\eqref{eq:fourier-expansion}, one has
$\|e^{i\langle\omegabf,\cdot\rangle}\|_{\Hcal_\kappa} = c_{\omegabf}^{-1/2}$ (whenever $c_{\omegabf}>0$),
hence
\begin{equation}
\label{eq:width-loose}
\sup_{\xbf\in\mathcal X}\mathrm{dist}(\kappa(\cdot,\xbf),Y_D)_{\Hcal_\kappa}
\le
\sum_{\omegabf\notin\Omega_D}\sqrt{c_{\omegabf}}.
\end{equation}
Since $d_D(\mathcal F)$ is the infimum over all $D$-dimensional subspaces, we conclude
\begin{equation}
\label{eq:width-loose-final}
d_D(\mathcal F)
\le
\sum_{\omegabf\notin\Omega_D}\sqrt{c_{\omegabf}}.
\end{equation}

\paragraph{(B) An improved upper bound for integer-valued frequencies.}
Assume additionally that $\mathcal X=\mathbb T^d=[0,2\pi]^d$ and that $\Omega\subset\mathbb Z^d$.
Then the Fourier characters $\{e^{i\langle\omegabf,\cdot\rangle}\}_{\omegabf\in\Omega}$ yield a Mercer decomposition,
and the functions $\{\sqrt{c_{\omegabf}}\,e^{i\langle\omegabf,\cdot\rangle}\}_{\omegabf\in\Omega}$ form an orthonormal
family in $\Hcal_\kappa$. Consequently, the residual above is an orthogonal sum in $\Hcal_\kappa$ and
Pythagoras' identity gives, for all $\xbf\in\mathcal X$,
\[
\|\kappa(\cdot,\xbf)-g_D(\cdot;\xbf)\|_{\Hcal_\kappa}^2
=
\sum_{\omegabf\notin\Omega_D} c_{\omegabf}.
\]
Therefore,
\begin{equation}
\label{eq:width-sharp-final}
d_D(\mathcal F)
\le
\sup_{\xbf\in\mathcal X}\|\kappa(\cdot,\xbf)-g_D(\cdot;\xbf)\|_{\Hcal_\kappa}
=
\Big(\sum_{\omegabf\notin\Omega_D} c_{\omegabf}\Big)^{1/2}.
\end{equation}

\paragraph{Final misspecification bounds.}
Combining~\eqref{eq:misspec-main} with~\eqref{eq:width-loose-final} yields the general bound
\[
\varepsilon_{2D}(f^*)
\le
\frac{\sqrt{2}B}{\gamma} \Big(\sum_{\omegabf\notin\Omega_D}\sqrt{c_{\omegabf}}\Big)^{1/2}.
\]
Under the integer-frequency assumption, combining~\eqref{eq:misspec-main} with~\eqref{eq:width-sharp-final} yields the sharper bound
\[
\varepsilon_{2D}(f^*)
\le
\frac{\sqrt{2}B}{\gamma}\Big(\sum_{\omegabf\notin\Omega_D} c_{\omegabf}\Big)^{1/4}.
\]

\subsection{Example Circuits and their spectral decay} \label{app:pgreedy:examples}

\paragraph{Example: Product single-qubit rotation encoding (exponential tail).}
We consider the one-dimensional periodic domain $\mathcal X=\mathbb T=[0,2\pi]$ and the $n$-qubit feature map
\begin{equation}
\label{eq:prod-encoding-state}
\ket{\psi(x)}
\;:=\;
\bigl(R_x(x)\ket{0}\bigr)^{\otimes n},
\qquad x\in[0,2\pi],
\end{equation}
where $R_x(\theta)=\exp(-i\theta \Xbf/2)$ is a single-qubit rotation about the Pauli-$X$ axis.
We define the corresponding quantum kernel as the squared state overlap
\begin{equation}
\label{eq:prod-encoding-kernel-def}
\kappa(x,x')
\;:=\;
\bigl|\langle \psi(x)|\psi(x') \rangle\bigr|^2,
\qquad x,x'\in[0,2\pi].
\end{equation}

For one qubit,
\[
\bra{0}R_x(x)^\dagger R_x(x')\ket{0}
=
\cos\Bigl(\frac{x-x'}{2}\Bigr)
=
\cos\Bigl(\frac{\Delta}{2}\Bigr).
\]
Taking the $n$-fold tensor product gives
$\langle \psi(x)|\psi(x') \rangle=\cos(\Delta/2)^n$.
Squaring the modulus yields, with $\Delta:=x-x'$,
\begin{equation}
\label{eq:prod-encoding-closed-form}
\kappa(x,x')
\;=\;
\cos^{2n}\Bigl(\frac{\Delta}{2}\Bigr).
\end{equation}

Since $\kappa(x,x')$ depends only on $\Delta=x-x'$, it is shift-invariant on $\mathbb T$ and admits a Fourier series $\kappa(\Delta)=\sum_{\ell\in\mathbb Z} c_\ell e^{i\ell\Delta}$.

Moreover, the kernel~\eqref{eq:prod-encoding-closed-form} is a trigonometric polynomial of degree $n$ and therefore has
finite spectrum $\ell\in\{-n,\dots,n\}$.
Using $\cos(\theta)=\tfrac12(e^{i\theta}+e^{-i\theta})$, we obtain
\begin{align*}
\cos^{2n}\!\Bigl(\frac{\Delta}{2}\Bigr)
&=
2^{-2n}\bigl(e^{i\Delta/2}+e^{-i\Delta/2}\bigr)^{2n}
\nonumber\\
&=
2^{-2n}\sum_{j=0}^{2n}\binom{2n}{j}e^{i(n-j)\Delta} \\
&=
\sum_{\ell=-n}^{n} c_\ell e^{i\ell\Delta},
\end{align*}
where the Fourier coefficients are explicitly
\begin{equation}
\label{eq:prod-encoding-coefs}
c_\ell
=
2^{-2n}\binom{2n}{n-\ell},
\qquad \ell=-n,\dots,n.
\end{equation}
In particular, $c_\ell\ge 0$, $c_{-\ell}=c_\ell$, and the sequence is unimodal with
$c_0\ge c_1\ge \cdots \ge c_n$.

Now, by our previous results, the (squared) Kolmogorov $D$-width of $\mathcal F$
is upper bounded by the spectral tail:

$$d_D(\mathcal F)^2
\;\le\;
\sum_{j>D} c_{(j)}
\;=\;
\sum_{\ell\notin\Omega_D} c_\ell,$$

where $c_{(1)}\ge c_{(2)}\ge\cdots$ denotes the non-increasing rearrangement and $\Omega_D$ indexes the top-$D$
coefficients.
Note the top coefficients correspond to the lowest frequencies. In particular, for any integer
$k\in\{0,1,\dots,n\}$, choosing
\[
\Omega_{2k+1}:=\{-k,-k+1,\dots,0,\dots,k\}
\]
gives
\begin{equation}
\label{eq:tail-as-binomial}
d_{2k+1}(\mathcal F)^2
\;\le\;
\sum_{|\ell|\ge k+1} c_\ell
=
2\sum_{\ell=k+1}^{n} c_\ell.
\end{equation}

Now to bound \eqref{eq:tail-as-binomial}, let $S\sim\mathrm{Bin}(2n,\tfrac12)$. Then, from~\eqref{eq:prod-encoding-coefs},
\[
c_\ell = \mathbb P(S=n-\ell).
\]
Hence
\[
2\sum_{\ell=k+1}^n c_\ell
=
2\,\mathbb P(S\le n-k-1)
=
2\,\mathbb P(S-n\le -(k+1))
\;\le\;
2\,\mathbb P(|S-n|\ge k+1).
\]
Applying Hoeffding's inequality to the binomial variable $S$ yields
\[
\mathbb P(|S-n|\ge k+1)
\le
2\exp\Bigl(-\frac{(k+1)^2}{n}\Bigr),
\]
and substituting into~\eqref{eq:tail-as-binomial} gives the explicit tail bound
\begin{equation*}
d_{2k+1}(\mathcal F)^2
\;\le\;
4\exp\Bigl(-\frac{(k+1)^2}{n}\Bigr),
\qquad k=0,1,\dots,n-1.
\end{equation*}
Equivalently, for any $D\in\{1,2,\dots,2n+1\}$, letting $k=\lfloor(D-1)/2\rfloor$,
\begin{equation}
d_D(\mathcal F)^2
\;\le\;
4\exp\Bigl(-\frac{(\lfloor (D-1)/2\rfloor+1)^2}{n}\Bigr)
\;\le\;
4\exp\Bigl(-\frac{D^2}{16n}\Bigr).
\end{equation}

\paragraph{Gaussian frequency profiles in quantum re-uploading models and sub-Gaussian spectral tails.}
\cite{Barthe2024gradientsfrequency} study quantum re-uploading models, i.e.,
parametrized quantum circuits in which data-encoding unitaries are repeatedly interleaved with trainable gates.
They prove that such models output functions with vanishing high-frequency components and, in particular,
derive uniform upper bounds on derivatives with respect to the input data (hence limiting sensitivity to fine-scale variations).
Moreover, their numerical experiments indicate that when the data-uploading gates are interleaved with
random trainable gates, the Fourier transform of quantum re-uploading models exhibits Gaussian profiles.

For a kernel with Fourier expansion
$\kappa(\Delta)=\sum_{\ell\in\mathbb Z} c_\ell e^{i\ell \Delta}$, if its Fourier weights satisfy a Gaussian (sub-Gaussian) decay
\begin{equation*}
c_\ell \;\le\; C_0 \exp(-a\ell^2),\qquad \forall\,\ell\in\mathbb Z,
\end{equation*}
for some constants $C_0,a>0$.
Since the largest coefficients correspond to the lowest frequencies, choosing $\Omega_{2k+1}=\{-k,\dots,k\}$ gives
\[
d_{2k+1}(\mathcal F)^2
\;\le\;
\sum_{|\ell|\ge k+1} c_\ell
\;\le\;
2C_0\sum_{\ell=k+1}^\infty e^{-a\ell^2}
\;\le\;
2C_0\int_{k}^{\infty} e^{-at^2}\,dt
\;\le\;
\frac{C_0}{a k}\,e^{-a k^2},
\]
where the last inequality uses the standard Gaussian tail bound.
Equivalently, for $D=2k+1$ and $k=\lfloor(D-1)/2\rfloor$,
\begin{equation}
\label{eq:kolmogorov-subgaussian-D}
d_D(\mathcal F)^2
\;\le\;
\frac{C_1}{D}\exp\!\Big(-\frac{a}{4}(D-1)^2\Big),
\end{equation}
for some constant $C_1>0$ depending only on $(C_0,a)$.
Thus, a Gaussian-shaped frequency profile implies an exponential decay of the spectral tail, and hence of the Kolmogorov width.

\section{Extension to the Case where the Reward Function lies in an RKHS of Classical Kernels}\label{app:classicalrkhs}

Although we initially presented our approach in the quantum context, it applies equally well to classical kernels, whether they induce finite- or infinite-dimensional RKHS. In particular, the same analysis used for Random Fourier Features carries over to the classical setting. Moreover, for certain infinite-dimensional kernels, the P-greedy algorithm admits a rapidly decaying error bound---often polynomial or even exponential in the number of selected points---according to Theorem \ref{thm:pgreedyerrorbound}.

Consequently, in a bandit optimization scenario over a finite discretized domain, one can directly integrate P-greedy approximations into methods such as EC-GP-UCB \citep{bogunovic2021misspecified} or standard misspecified linear bandit algorithms. By substituting the resulting $\varepsilon$-term (representing the approximation error), it is possible to derive improved regret guarantees alongside reduced computational overhead.

\section{Experimental Details}\label{app:expdetails} 

We provide more details about the experiments in Section 4.

\subsection{Experiment 1}

\begin{figure*}[t!] 
  \centering 
    \includegraphics[width=0.9\linewidth]{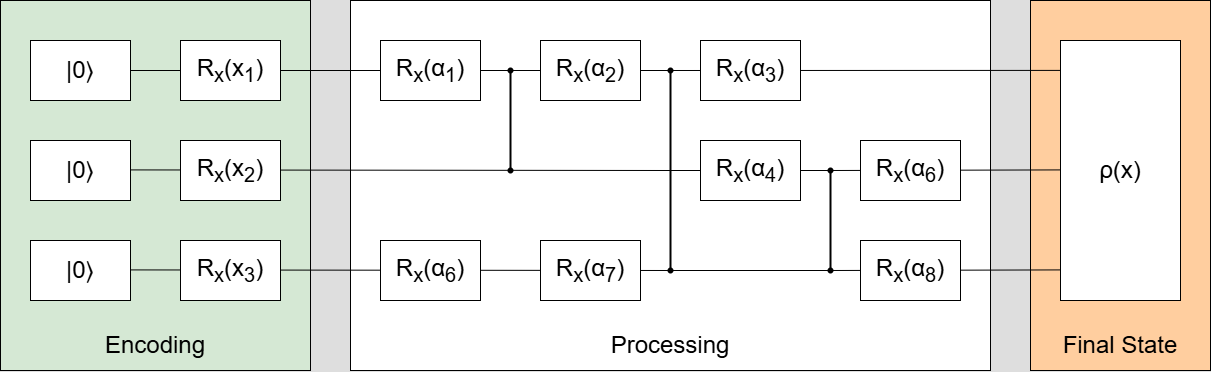}
    \caption{Example Circuit Structure for generating $\kappa_Q$, where $x_i$ are the input data components and $\alpha_i$ are randomly sampled parameters for the variational layers}
  \label{fig:circuit}

\end{figure*}

\begin{itemize}[leftmargin=*] 

\item \textbf{Quantum Circuit Construction.}
We generate the synthetic reward function \(f^*\) using a three-qubit parameterized quantum circuit composed of random layers of rotation gates and entangling operations, implemented via PennyLane. This results in a final circuit whose output density matrix defines our \emph{global fidelity} quantum kernel \(\kappa_Q(\cdot,\cdot)\).

The structure of the parameterized quantum circuit that defines the kernel $\kappa_Q$ (for this experiment) is depicted in Figure \ref{fig:circuit}. The architecture begins with a data-encoding layer, by letting each component of the input vector $\xbf$ parameterizes a Pauli-X rotation gate acting on a corresponding qubit. In the notation of Eq.~\eqref{eq:quantumcircuit}, this corresponds to $\Gbf_j=\Xbf_j$ for the data-encoding layer. Subsequently, the state is evolved through randomly structured layers of rotation and entangling gates, constructed using PennyLane's RandomLayers template. For a given input $\xbf$, the resulting $n$-qubit state is described by the density matrix $\rhobf(\xbf)$. The kernel is then defined as the Hilbert-Schmidt inner product between two states, $\kappa_Q(\xbf, \xbf') = \mathrm{Tr}[\rhobf(\xbf)\rhobf(\xbf')]$, which quantifies their fidelity.

\item \textbf{Action Space and Noise.}
We discretize the input domain \([0,2\pi]^3\) on a \(10\times 10\times 10\) uniform grid, yielding \(1000\) possible arms (actions). At each round \(t\), upon selecting an arm \(\mathbf{x}_t\), we observe a noisy reward
\[
    y_t \;=\; f^*(\mathbf{x}_t) \;+\; \eta_t,
    \quad \eta_t \sim \mathcal{N}(0,\,0.1^2).
\]
The observed values are then normalized to \([0,1]\) across the entire grid.

\item \textbf{Algorithms and Hyperparameters.}
We compare three kernel-approximation strategies---\emph{projected quantum kernels}, \emph{random Fourier features} (RFF), and \emph{P-greedy}---in combination with two bandit algorithms, \emph{SquareCB} \citep{foster2020mislin} and \emph{EC-GP-UCB} \citep{bogunovic2021misspecified}. We use their standard parameter settings from the respective papers, initializing the algorithms at round \(1\) (no special warm-up phase).

\item \textbf{Projected Kernel Details.}
Given our three-qubit circuit, we construct projected kernels by tracing out subsets of qubits. Specifically, we consider individual-qubit projections onto qubits \(\{0\}, \{1\}, \{2\}\) as well as pairs \(\{0,1\}, \{1,2\}, \{0,2\}\). To create a \emph{summed projected kernel} of size \(b\), we sum the local projected kernels over \(b\) disjoint sub-systems. For instance, having ``Number of Projected Kernels Summed'' $=3$ in the plot means summing the projected kernels on qubits \(0\), \(1\), and~\(2\).

\item \textbf{Random Fourier Features.}
For RFF, we sample frequencies uniformly at random from the kernel's Fourier frequencies. We vary the RFF dimension \(D\), as reported on the horizontal axis of Figures~\ref{fig:exp1_SquareCB}(b) and \ref{fig:exp1_ECGPUCB}(b).

\item \textbf{P-greedy.}
For the P-greedy approach, we greedily select basis points from the same \(10\times 10\times 10\) grid according to the power function criterion \citep{DeMarchi2005pgreedy}, incrementally building a subspace that approximates \(\kappa_Q\).

\item \textbf{Trials and Metrics.}
Each algorithm is run for \(T=100\) bandit rounds. We repeat the entire procedure \(30\) times with different random seeds. The plotted curves show the mean cumulative regret, with shaded regions or error bars indicating \(\pm 1\) standard deviation across trials. 
``Model complexity'' on the horizontal axis is the number of summed local kernels for projected quantum kernels, the RFF dimension \(D\) for random Fourier features, or the number of basis points (kernel dimension) in P-greedy. The rightmost point in each plot always corresponds to the
``Full kernel'', using the unprojected \(3\)-qubit fidelity kernel.

\end{itemize}

\subsection{Experiment 2}

\begin{itemize}[leftmargin=*]

\item \textbf{Parameter Range and Actions.}
We consider two coupling parameters \((J_1,J_2)\), each discretized into \(20\) equally spaced points over \([-4,4]\). This yields \(20\times20=400\) total arms in the bandit problem. At each round \(t\), we select an arm \(\bigl(J_1,J_2\bigr)\in[-4,4]^2\) and receive a reward corresponding to whether the system's ground state lies in phase~II.

\item \textbf{Ground-State Computation.}
The generalized cluster Hamiltonian for \(n\) qubits is given by
\[
  \Hbf_C \;=\; \sum_{j=1}^{n}
    \Bigl(\Zbf_j \;-\; J_1\,\Xbf_j \Xbf_{j+1} \;-\; J_2\,\Xbf_{j-1} \Zbf_j \Xbf_{j+1}\Bigr),
\]
where \(\Xbf_j\) and \(\Zbf_j\) denote Pauli operators on qubit \(j\). We exactly compute its ground state (i.e., the eigenvector corresponding to the lowest eigenvalue) via sparse matrix diagonalization.

\item \textbf{Phase Labeling and Noise.}
To label each \((J_1,J_2)\) as ``phase~II'' or ``not phase~II'', we use the known boundary conditions:
\[
  \text{phase~II if }
  \bigl(J_2 > -1 - J_1\bigr) \;\;\text{and}\;\;\bigl(J_2 < J_1 - 1\bigr).
\]
Hence, if \((J_1,J_2)\) is inside that region, the ``ideal'' label is \(1\); otherwise, it is \(0\). We then add \emph{i.i.d.}\ Gaussian noise of variance \(\sigma^2=0.01\), so the observed reward is 
\[
  y \;=\; \mathbb{I}\{\text{phase~II}\} \;+\; \eta,\quad \eta\sim\mathcal{N}(0,0.01).
\]

\item \textbf{Regret Definition.}
For each chosen \((J_1,J_2)\), the \emph{instantaneous regret} is \(1\) if that point is \emph{outside} phase~II (i.e., the ideal label is \(0\)), and \(0\) if it is inside phase~II. Thus, the \emph{cumulative regret} over \(T\) rounds equals the total number of times we pick a point out of phase~II in hindsight.

\item \textbf{Kernel Approximations and Algorithm Setup.}
We implement a three-qubit parameterized circuit (different from that of Experiment 1) with multiple layers of single-qubit rotations and entangling gates on the ground state of $\Hbf_C$ as the initial state, yielding a fidelity kernel \(\kappa_Q\). We then apply projected kernel and P-greedy with \emph{EC-GP-UCB}~\citep{bogunovic2021misspecified}. We run the algorithm for \(T=100\) rounds, record the cumulative regret, and then average over 30 runs. The standard deviation across these runs is shown as error bars in our plots.

\end{itemize}

\subsection{Experiment 3}

\begin{itemize}[leftmargin=*]

\item \textbf{Hamiltonian Setup.}
Specifically, we aim to approximate the ground state of the XYZ Hamiltonian, i.e., \(J_X=-1\), \(J_Y=J_Z=0\), and \(h_X=h_Y=0\), \(h_Z=-1\) in $$
\Hbf = -\sum^{n-1}_{j=1}
(
J_X \Xbf_j \Xbf_{j+1} + 
J_Y \Ybf_j \Ybf_{j+1} + 
J_Z \Zbf_j \Zbf_{j+1}
) -
\sum^n_{j=1}
(
h_X \Xbf_j + h_Y \Ybf_j +
h_Z \Zbf_j)
$$
where ${\Xbf_j,\Ybf_j,\Zbf_j}$ are Pauli matrices acting on the $j$th qubit. We fix \(n=3\) qubits for all experiments.

\item \textbf{Circuit Ansatz and Initial State.}
Following \citet{nicoli2023physics}, we employ an ``Efficient~SU(2)'' circuit with PennyLane as a parameterized ansatz \(\Ubf(\mathbf{x})\). The initial state \(\ket{\psi_0}\) is set to \(\ket{0}^{\otimes 3}\). The circuit outputs 
\[
  f^*(\mathbf{x})
  \;=\;
  \bra{\psi_0}\,\Ubf(\mathbf{x})^\dagger\,\Hbf\,\Ubf(\mathbf{x})\,\ket{\psi_0}
\]
exactly without measurement noise.

\item \textbf{Bayesian Optimization Details.}
We define a Gaussian Process (GP) model with the \emph{fidelity kernel} derived from the same 3-qubit circuit architecture, then apply Bayesian optimization using the Expected Improvement (EI) acquisition function. The EI is optimized by L-BFGS at each iteration to propose new circuit parameters. We run 50 optimization steps, repeating the procedure for 30 independent random seeds that affect parameter initialization, random draws in the approximate kernels, etc. 

\item \textbf{Metrics and Plots.}
The optimization runs for 50 steps, we record and plot the \emph{best} (lowest) energy found so far. The error bars (or shaded regions) depict $\pm1$ standard deviation across 30 trials.

\end{itemize}

\section{Bigger qubit system for Experiment 1} \label{app:more_exp1}

\begin{figure*}[tbp]                
  \centering

  \begin{subfigure}{.3\textwidth}
    \includegraphics[width=\linewidth]{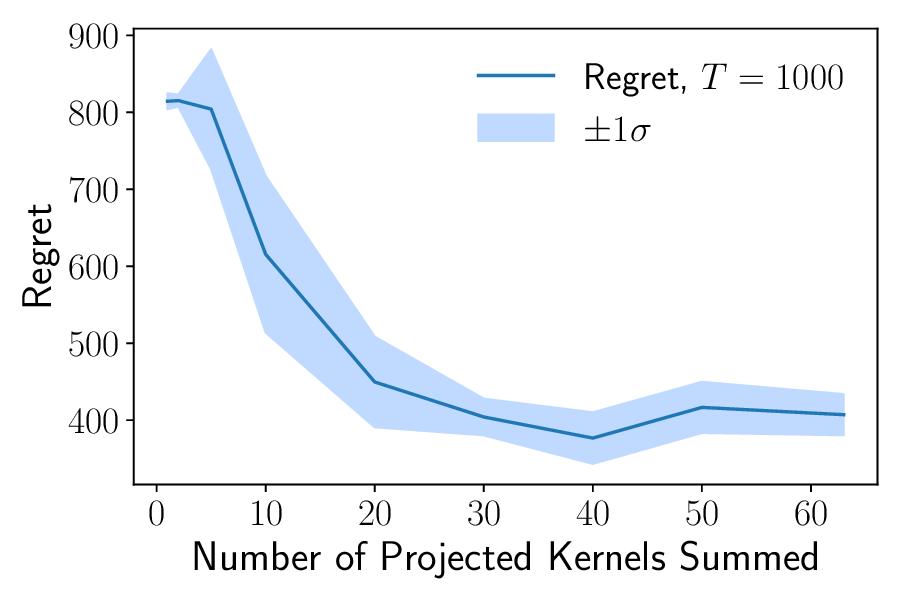}
    \caption{Projected Quantum Kernels}
  \end{subfigure}\hfill
  \begin{subfigure}{.3\textwidth}
    \includegraphics[width=\linewidth]{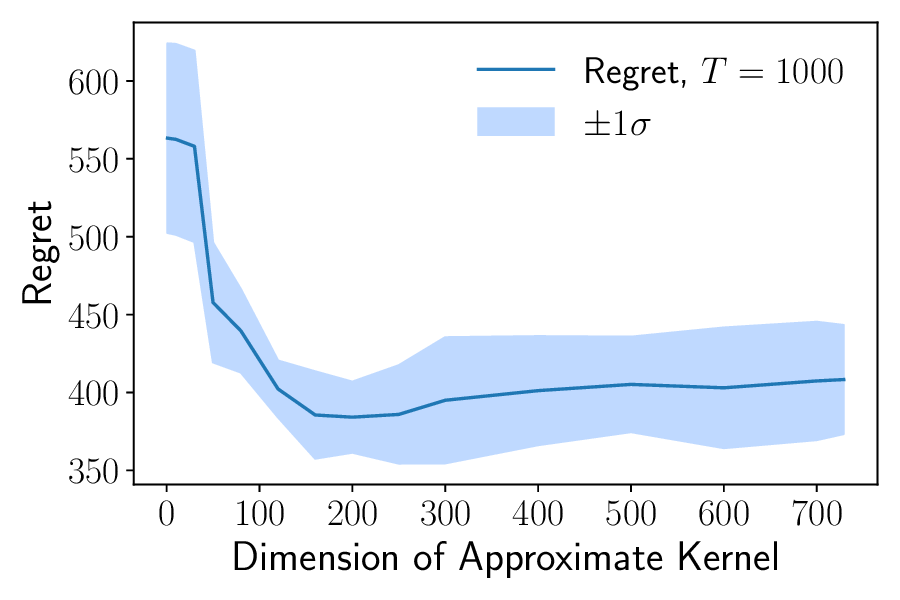}
    \caption{Random Fourier Features}
  \end{subfigure}\hfill%
  \begin{subfigure}{.3\textwidth}
    \includegraphics[width=\linewidth]{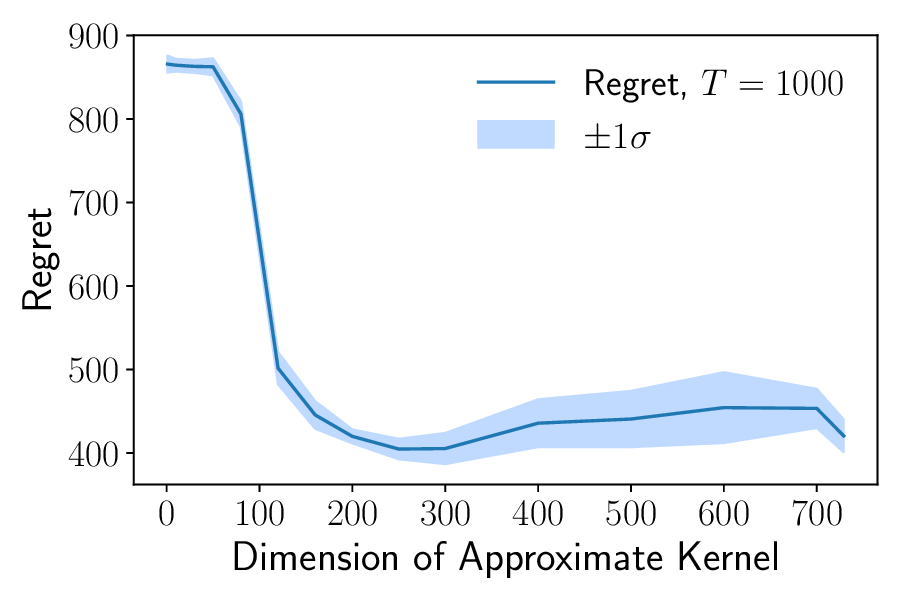}
    \caption{P-greedy}
  \end{subfigure}

  \caption{Cumulative regret of SquareCB algorithm as a function of  approximation dimension $D$ for $(a)$ projected quantum kernels, $(b)$ RFF and $(c)$ P-greedy approximation.}
  \label{fig:exp1_SquareCB_big_n}

\end{figure*}

\begin{figure*}[tbp]                
  \centering

  \begin{subfigure}{.3\textwidth}
    \includegraphics[width=\linewidth]{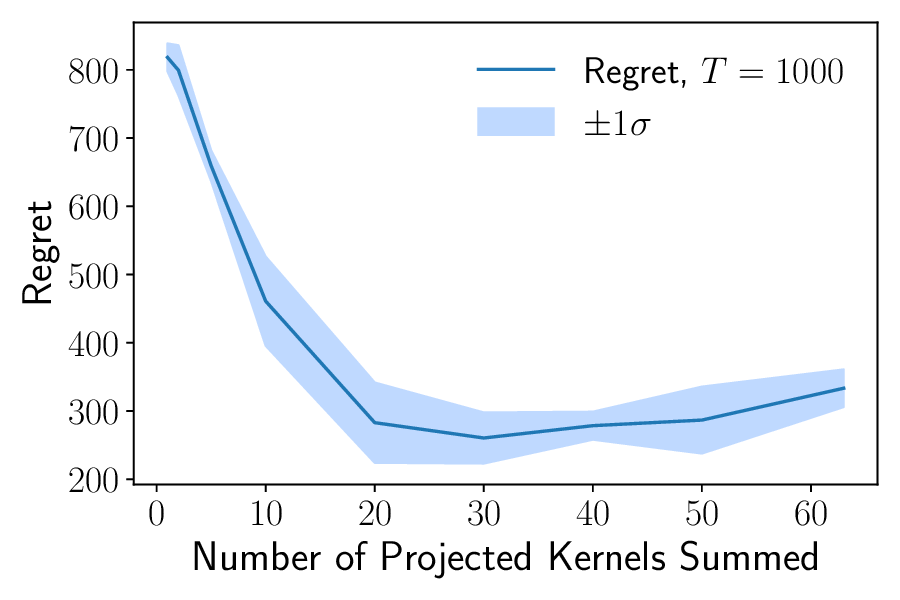}
    \caption{Projected Quantum Kernels}
  \end{subfigure}\hfill
  \begin{subfigure}{.3\textwidth}
    \includegraphics[width=\linewidth]{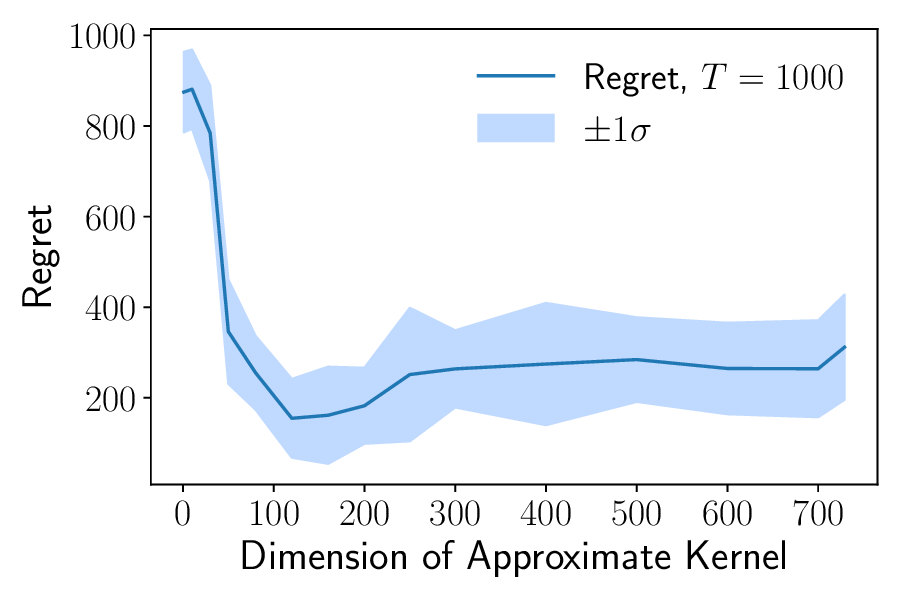}
    \caption{Random Fourier Features}
  \end{subfigure}\hfill%
  \begin{subfigure}{.3\textwidth}
    \includegraphics[width=\linewidth]{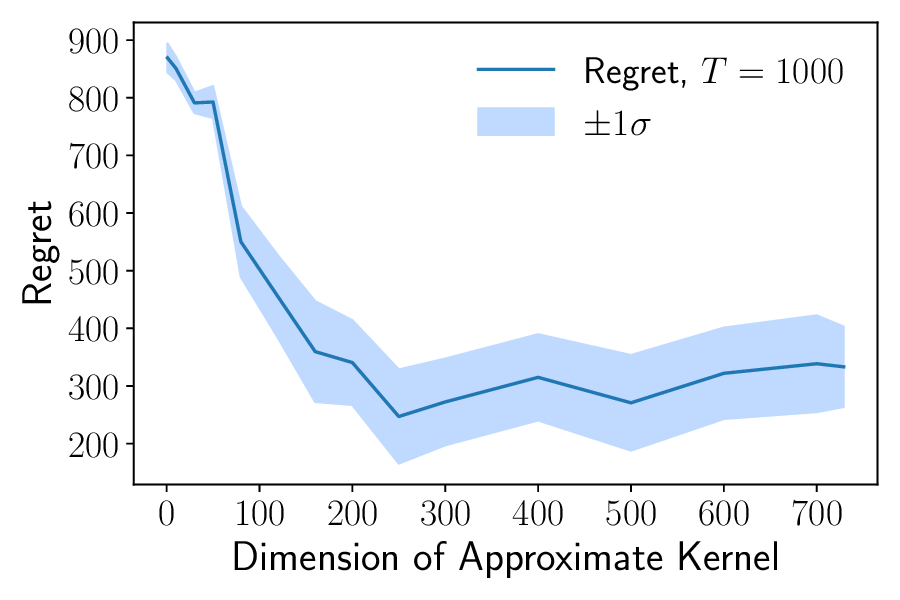}
    \caption{P-greedy}
  \end{subfigure}

  \caption{EC-GP-UCB with different kernel approximation approaches for reward functions drawn from GPs with the full quantum kernel}
  \label{fig:exp1_ECGPUCB_big_n}

\end{figure*}

This section provides additional results to illustrate scalability beyond the 3-qubit setting in Section~\ref{exp1}. Following exactly the same protocol as Experiment~1 (same reward-generation mechanism, noise model, and evaluation procedure, though we increase $T$ to $1000$ here), we increase the number of qubits from $3$ to $6$, specifically, we consider a synthetic reward function $f^*$ drawn from a GP prior induced by the full 6-qubit product kernel on the domain $[0,2\pi]^6$. We then run the same bandit pipelines (SquareCB and EC-GP-UCB) with the same three kernel-approximation strategies: projected quantum kernels, random Fourier features (RFF), and P-greedy.

Figures~\ref{fig:exp1_SquareCB_big_n} and \ref{fig:exp1_ECGPUCB_big_n} mirror Figures~\ref{fig:exp1_SquareCB} and~\ref{fig:exp1_ECGPUCB} in the main text, each curve reports cumulative regret as a function of the approximation dimension (number of summed projected kernels for LPQK; feature dimension $D$ for RFF; or the number of selected basis elements for P-greedy). As in the 3-qubit experiments, we observe the characteristic ``U-shaped'' behavior: small approximation dimension underfits the true function (high misspecification error), while overly large dimension increases the information-gain penalty and can degrade regret.

Overall, these results confirm that the qualitative trade-off identified in Section~\ref{exp1} persists at higher qubit counts, supporting the scalability of our approximation framework to larger quantum systems.



\section{Limitations and Future Directions}\label{app:limitations}

\begin{itemize}[leftmargin=*]
\item \textbf{Conservativeness of Regret Bounds.} Our theoretical regret bounds provide upper estimates that may be overly conservative  in practice. Consequently, using these bounds to select an optimal kernel dimension $D$ might lead to suboptimal or unnecessarily large choices.

\item \textbf{Approximation Error Uncertainty.} Although Random Fourier Features (RFF) typically yield an error rate of $\varepsilon = 1/\sqrt{D}$, this estimate can be pessimistic in many real-world scenarios. Faster decay rates---such as polynomial or exponential---are achievable but generally require problem-specific analysis. Additionally, for projected kernels and P-greedy methods, there is no universal closed-form bound for \(\varepsilon\); each case necessitates empirical evaluation or tailored theoretical analysis.

\item \textbf{Limitations of Fourier Series Representations.} Our RFF analysis relies on the assumption that the quantum kernel admits a discrete or tractable Fourier expansion. However, some kernels may lack such a structure, making our current approach inapplicable without significant modifications.

\item \textbf{Noise Model and Kernel Estimation.}
We model the observed reward noise as i.i.d. Gaussian, which is a standard approximation to finite-shot measurement noise in the large-shot regime.
This abstraction does not capture structured NISQ noise, such as readout errors, depolarizing noise, relaxation, drift, or cross-talk. 
Moreover, our approximation-error analysis assumes exact kernel values; finite-shot kernel estimation and device noise would introduce additional errors that may interact with the kernel approximation error.

\item \textbf{Quantum Measurement Overhead.}
The computational speedups discussed in Appendix~\ref{app:runningtime} concern classical posterior updates and linear-algebra costs. If kernel values are estimated on quantum hardware, the total measurement cost can still be significant. Full GP-UCB requires at least $\mathcal{O}(T^2)$ pairwise kernel evaluations over $T$ selected points, while a fixed $D$-dimensional approximation requires $\mathcal{O}(TD)$ feature or basis evaluations after the surrogate is fixed. These counts must be multiplied by the number of shots required per kernel or observable estimate. For LPQK, a direct tomography-based implementation over all subsystems $|s|\le b$ may additionally require local Pauli measurements over $\sum_{w=0}^{b}\binom{n}{w}3^w$ settings. A full measurement-aware regret and runtime analysis is left for future work.

\item \textbf{Generality Beyond Quantum Kernels.}
The misspecified-bandit framework applies to general classical kernels as well as quantum kernels. 

\end{itemize}

\end{document}